\documentclass{article}

\usepackage[dvipsnames]{xcolor}  %

\usepackage{microtype}
\usepackage{graphicx}
\usepackage{subfigure}
\usepackage{booktabs} %

\usepackage{hyperref}

\usepackage[accepted]{icml2022}

\icmltitlerunning{Interactively Learning Preference Constraints in Linear Bandits}

\usepackage{amsmath}
\usepackage{amssymb}
\usepackage{amsthm}
\usepackage{mathtools}
\usepackage{cleveref}
\usepackage{xspace}
\usepackage{thmtools}
\usepackage{thm-restate}

\usepackage[hang,flushmargin]{footmisc}

\newcommand{\specialcell}[2][c]{%
\begin{tabular}[#1]{@{}c@{}}#2\end{tabular}}
\usepackage{array}
\newcolumntype{M}[1]{>{\centering\arraybackslash}m{#1}}

\newcommand{\R}{\mathbb{R}}

\newcommand{\cA}{\mathcal{A}}

\newcommand{\vx}{\mathbf{x}}

\newcommand{\vf}{\mathbf{f}}
\newcommand{\vr}{\mathbf{r}}
\newcommand{\vc}{\mathbf{c}}

\newcommand{\diag}{\mathrm{diag}}

\newcommand{\Gaussian}{\mathcal{N}}

\newcommand{\argmax}{\mathop\mathrm{argmax}}
\newcommand{\argmin}{\mathop\mathrm{argmin}}
\newcommand{\Expectation}{\mathbb{E}}
\newcommand{\NormalDist}{\mathcal{N}}
\newcommand{\KL}{\mathrm{KL}}
\newcommand{\defeq}{\coloneqq}

\newcommand{\Problem}{\nu}
\newcommand{\RewParam}{\theta}
\newcommand{\ConstParam}{\phi}
\newcommand{\Act}{x}
\newcommand{\ActSet}{\mathcal{X}}

\newcommand{\noise}{\eta}
\newcommand{\BetterActSet}{\ActSet^{\geq}_\RewParam}

\newcommand{\Design}{\lambda}
\newcommand{\DesMat}{A_{\Design}}
\newcommand{\Complexity}{H_{\mathrm{CLB}}}
\newcommand{\WeakComplexity}{\bar{H}_{\mathrm{CLB}}}

\newcommand{\MinConst}{{C^+_{\mathrm{min}}}}
\newcommand{\StoppingTime}{\tau}

\newcommand{\BestAct}{\Act^*_\Problem}
\newcommand{\BestActP}{\Act^*_{\Problem'}}
\newcommand{\TrueConst}{{\ConstParam}}

\newcommand{\UpperConst}{u_\ConstParam}
\newcommand{\LowerConst}{l_\ConstParam}

\newcommand{\pd}[2]{\frac{\partial #1}{\partial #2}}
\renewcommand{\equiv}{\Leftrightarrow}

\newcommand{\UncertainSet}{\mathcal{U}}
\newcommand{\FeasibleSet}{\mathcal{F}}
\newcommand{\Round}{t}
\newcommand{\GoodEvent}{\mathcal{E}}
\newcommand{\SmallGapSet}{\mathcal{S}}

\let\emptyset\varnothing

\newtheorem{definition}{Definition}
\newtheorem{theorem}{Theorem}

\newtheorem{lemma}{Lemma}

\newcommand{\CBAILong}{constrained linear best-arm identification\xspace}
\newcommand{\CBAIShort}{CBAI\xspace}
\newcommand{\AlgLong}{Adaptive Constraint Learning\xspace}
\newcommand{\AlgShort}{ACOL\xspace}
\newcommand{\AlgGreedyLong}{Greedy Adaptive Constraint Learning\xspace}
\newcommand{\AlgGreedyShort}{G-ACOL\xspace}

\newcommand{\MaxRewF}{MaxRew-\ensuremath{\FeasibleSet}\xspace}
\newcommand{\MaxRewU}{MaxRew-\ensuremath{\UncertainSet}\xspace}

\newcommand{\paragraphsmall}[1]{\vspace{-1em}\paragraph{#1}}
\unskip
\newcommand{\legendMaxvarAllTuned}{\includegraphics[width=4ex]{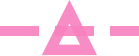}}%
\newcommand{\legendMaxvarAll}{\includegraphics[width=4ex]{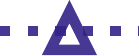}}%
\newcommand{\legendAdaptiveStatic}{\includegraphics[width=4ex]{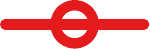}}%
\newcommand{\legendAdaptiveTuned}{\includegraphics[width=4ex]{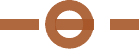}}%
\newcommand{\legendUniformAll}{\includegraphics[width=4ex]{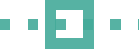}}%
\newcommand{\legendUniformAllTuned}{\includegraphics[width=4ex]{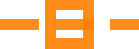}}%
\newcommand{\legendAdaptiveUniformTuned}{\includegraphics[width=4ex]{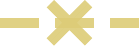}}%
\newcommand{\legendAdaptiveUniform}{\includegraphics[width=4ex]{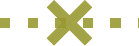}}%
\newcommand{\legendAdaptive}{\includegraphics[width=4ex]{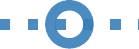}}%
\newcommand{\legendGAllocation}{\includegraphics[width=4ex]{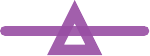}}%
\newcommand{\legendOracle}{\includegraphics[width=4ex]{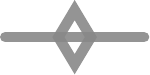}}%
\newcommand{\legendUniform}{\includegraphics[width=4ex]{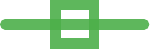}}%
\newcommand{\legendMaxRewU}{\includegraphics[width=4ex]{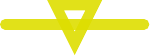}}%
\newcommand{\legendMaxRewUTuned}{\includegraphics[width=4ex]{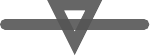}}%
\newcommand{\legendDriverTargetVelocity}{\raisebox{0.5ex}{\includegraphics[width=4ex]{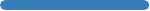}}\xspace}%
\newcommand{\legendDriverTargetLocation}{\raisebox{0.5ex}{\includegraphics[width=4ex]{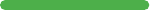}}\xspace}%
\newcommand{\legendDriverBlocked}{\raisebox{0.5ex}{\includegraphics[width=4ex]{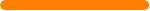}}\xspace}%

\newcommand{\codeurl}{\url{https://github.com/lasgroup/adaptive-constraint-learning}}

\begin{document}

\twocolumn[
\icmltitle{
Interactively Learning Preference Constraints in Linear Bandits
}

\icmlsetsymbol{equal}{*}

\begin{icmlauthorlist}
\icmlauthor{David Lindner}{eth}
\icmlauthor{Sebastian Tschiatschek}{uov}
\icmlauthor{Katja Hofmann}{msft}
\icmlauthor{Andreas Krause}{eth}
\end{icmlauthorlist}

\icmlaffiliation{eth}{ETH Zurich, Switzerland}
\icmlaffiliation{uov}{University of Vienna, Austria}
\icmlaffiliation{msft}{Microsoft Research Cambridge, UK}

\icmlcorrespondingauthor{David Lindner}{david.lindner@inf.ethz.ch}

\icmlkeywords{Machine Learning, ICML}

\vskip 0.3in
]

\printAffiliationsAndNotice{}  %

\begin{abstract}
\looseness -1 We study sequential decision-making with known rewards and unknown constraints, motivated by situations where the constraints represent expensive-to-evaluate human preferences, such as safe and comfortable driving behavior. We formalize the challenge of interactively learning about these constraints as a novel linear bandit problem which we call \emph{\CBAILong}. To solve this problem, we propose the \emph{\AlgLong} (\AlgShort) algorithm. We provide an instance-dependent lower bound for \CBAILong and show that \AlgShort's sample complexity matches the lower bound in the worst-case. In the average case, \AlgShort's sample complexity bound is still significantly tighter than bounds of simpler approaches. In synthetic experiments, \AlgShort performs on par with an oracle solution and outperforms a range of baselines. As an application, we consider learning constraints to represent human preferences in a driving simulation. \AlgShort is significantly more sample efficient than alternatives for this application. Further, we find that learning preferences as constraints is more robust to changes in the driving scenario than encoding the preferences directly in the reward function.
\end{abstract}

\section{Introduction}\label{sec:introduction}

Often, (sequential) decision-making problems are formalized as maximizing an unknown reward function that captures an expensive-to-evaluate objective, for example, user preferences (see Chapter 1 of \citet{lattimore2020bandit} for examples). However, in many practical situations, it can be more natural to model problems with a known reward function and unknown, expensive-to-evaluate constraints.

For example, a cookie manufacturer might want to create a low-calorie cookie.\footnotemark[4] The cookie should have the lowest amount of calories possible, but at least $95\%$ of customers should like it. To evaluate this constraint, the manufacturer has to produce specific cookies and test them with customers. The reward, i.e., the amount of calories for a recipe, is easy to evaluate without producing a cookie. Because customer trials are expensive, the cookie manufacturer wants to find the best constrained solution with as few trials as possible.

\begin{figure*}\centering
   \newlength{\DriverImageWidth}
   \setlength{\DriverImageWidth}{2.7cm}
   \setlength{\tabcolsep}{0.5pt}
   \renewcommand{\arraystretch}{0.5}
   \small
   \subfigure[Reward Penalty]{
       \begin{tabular}{M{\DriverImageWidth} M{\DriverImageWidth} M{\DriverImageWidth}}
           \specialcell{Base \\ Scenario} \vspace*{0.7pt} & \specialcell{Different \\ Reward} \vspace*{0.7pt} & \specialcell{Different \\ Environment} \vspace*{0.7pt} \\ \includegraphics[width=\DriverImageWidth]{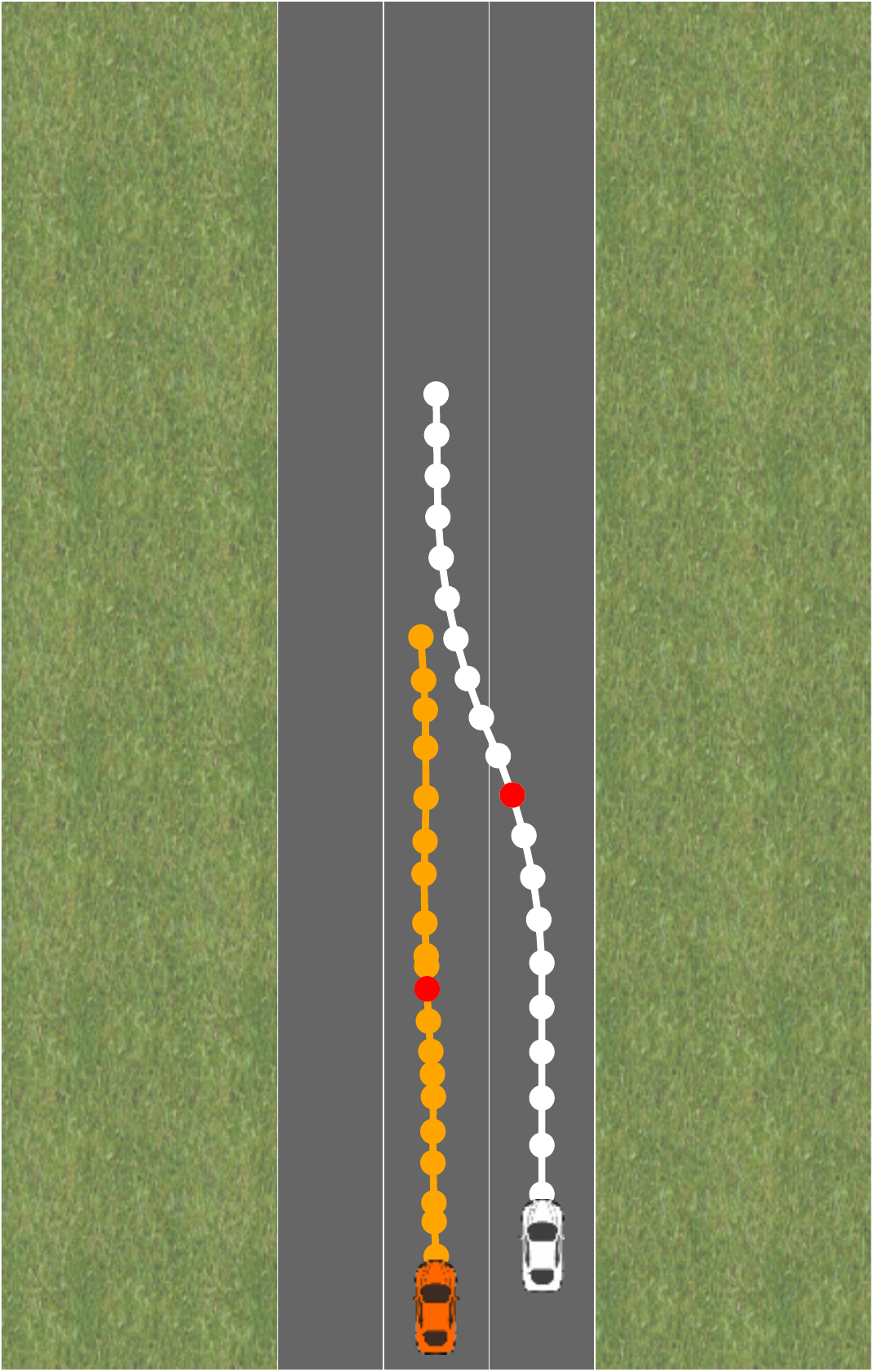} & \includegraphics[width=\DriverImageWidth]{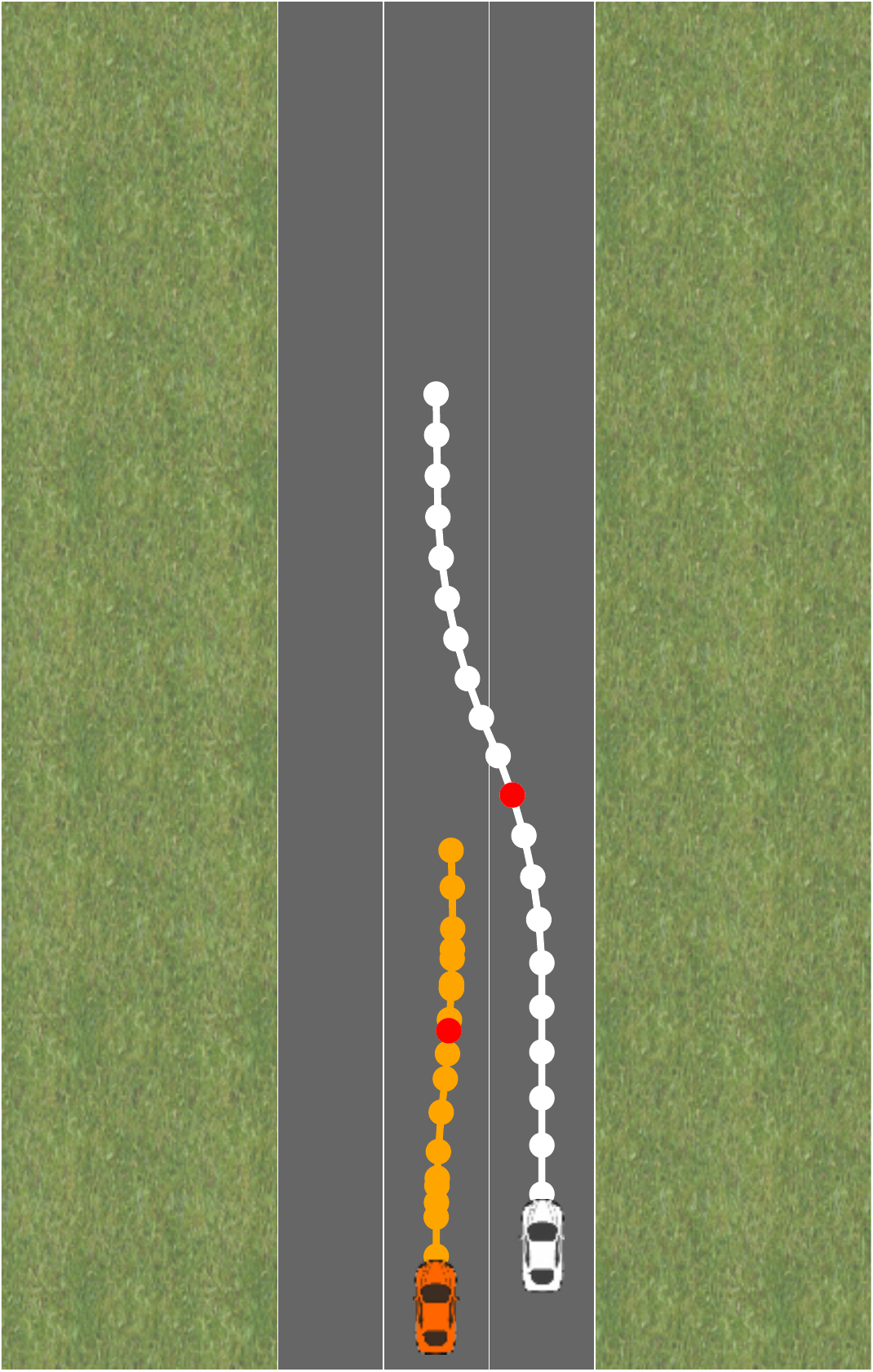} & \includegraphics[width=\DriverImageWidth]{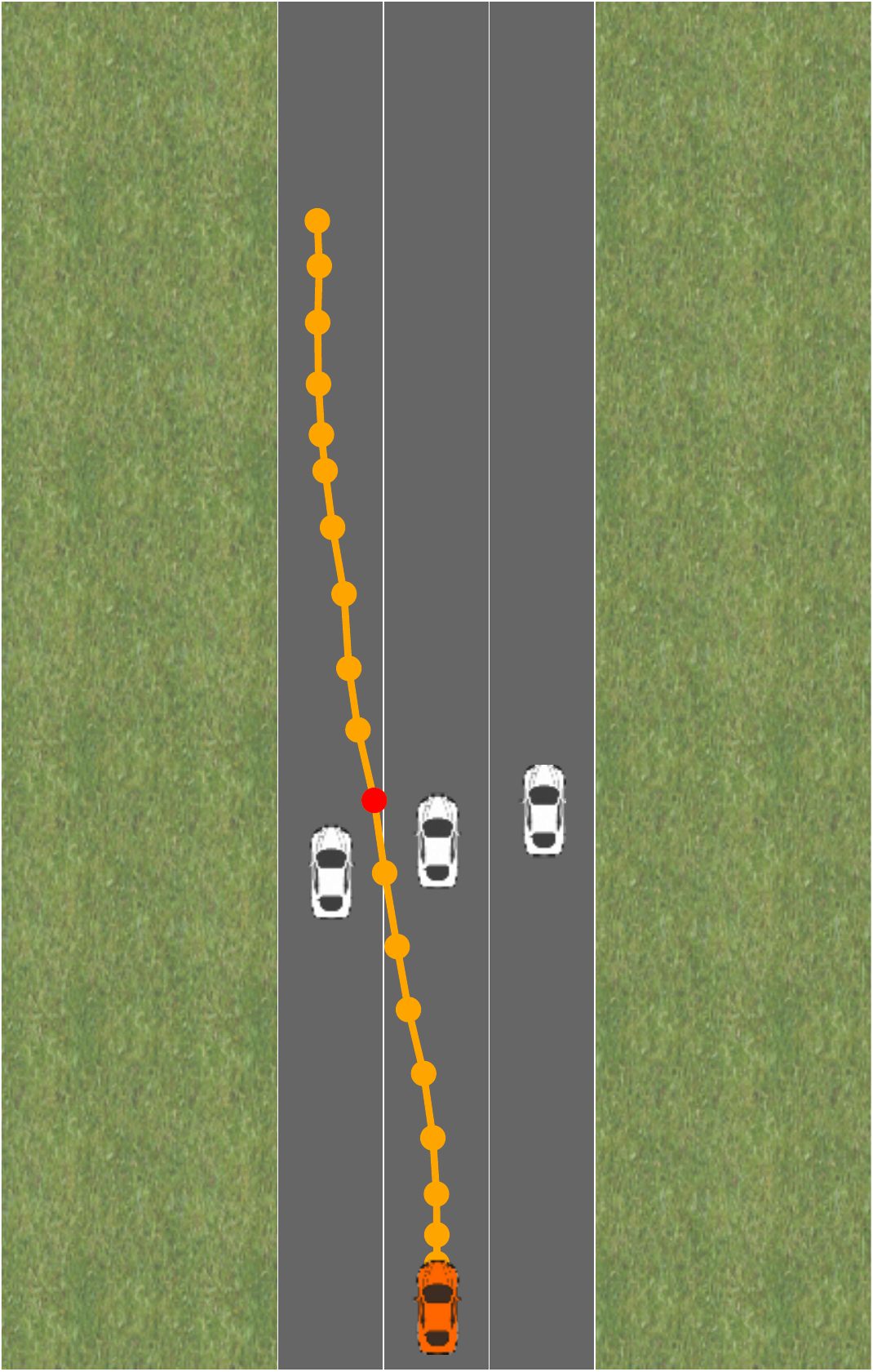}
       \end{tabular}
       \label{subfig:driver_reward_penalty}
   }\hspace{0.1cm}
   \subfigure[Constraint]{
       \begin{tabular}{M{\DriverImageWidth} M{\DriverImageWidth} M{\DriverImageWidth}}
           \specialcell{Base \\ Scenario} \vspace*{0.7pt} & \specialcell{Different \\ Reward} \vspace*{0.7pt} & \specialcell{Different \\ Environment} \vspace*{0.7pt} \\
           \includegraphics[width=\DriverImageWidth]{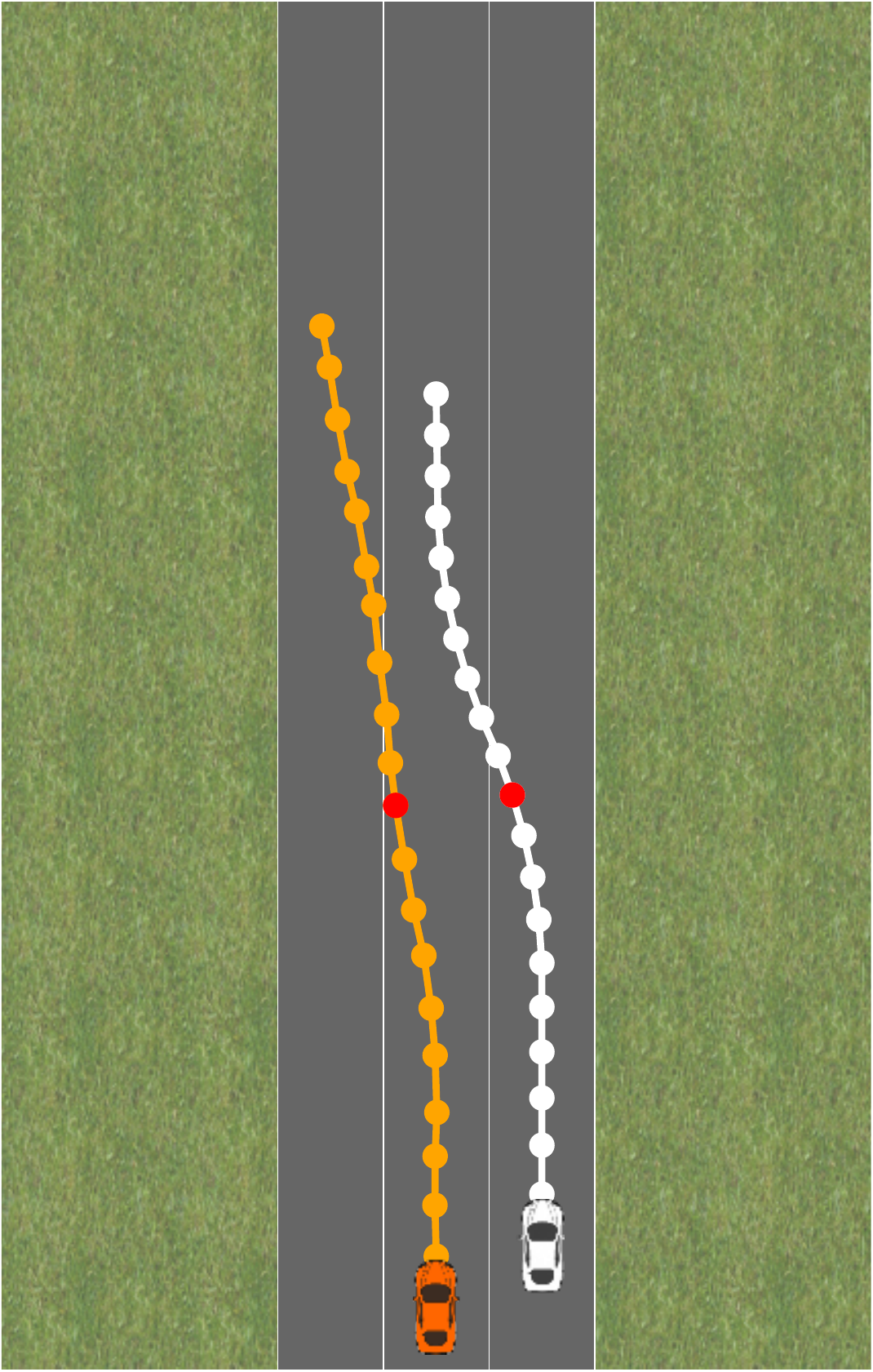} & \includegraphics[width=\DriverImageWidth]{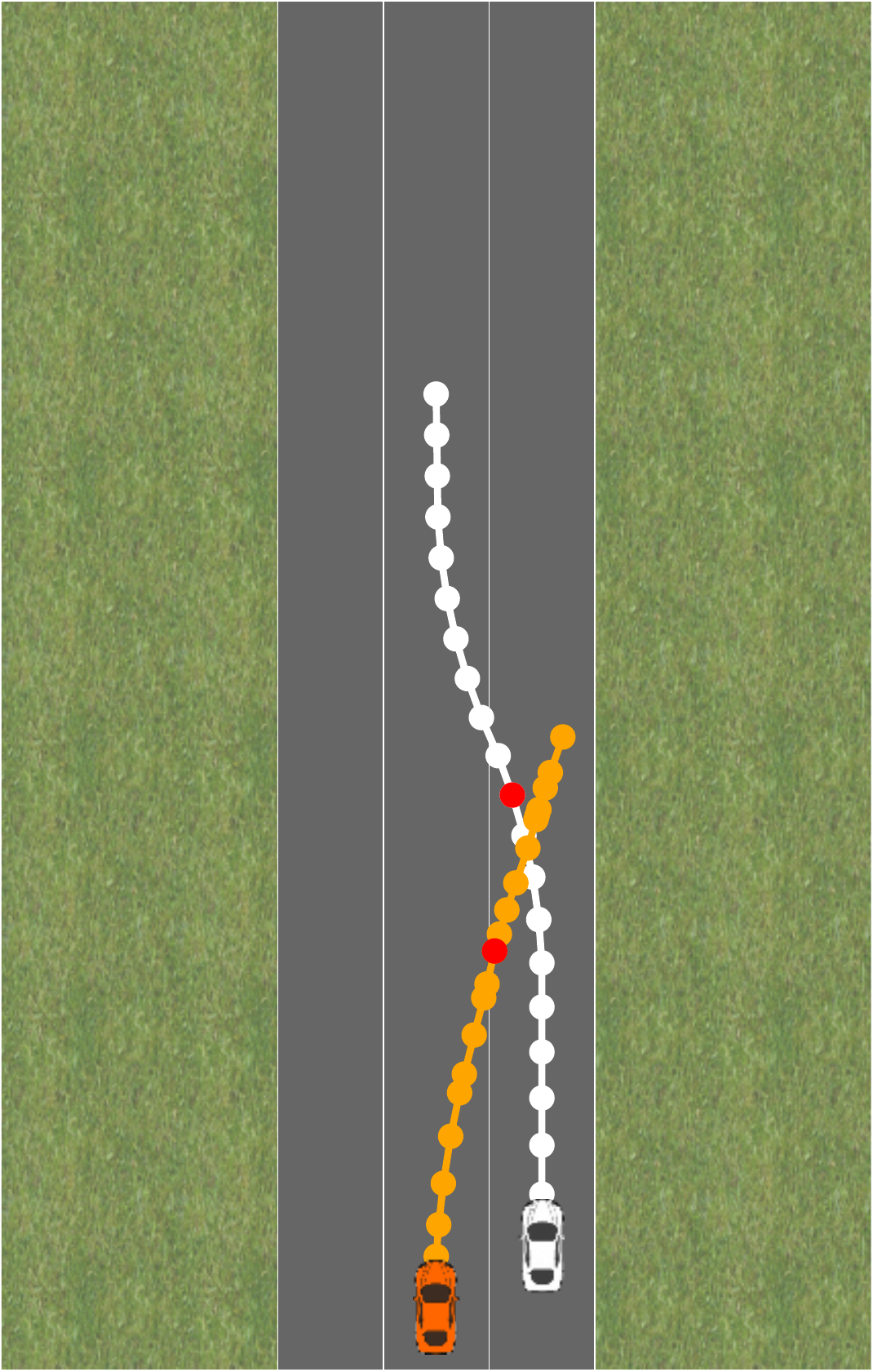} & \includegraphics[width=\DriverImageWidth]{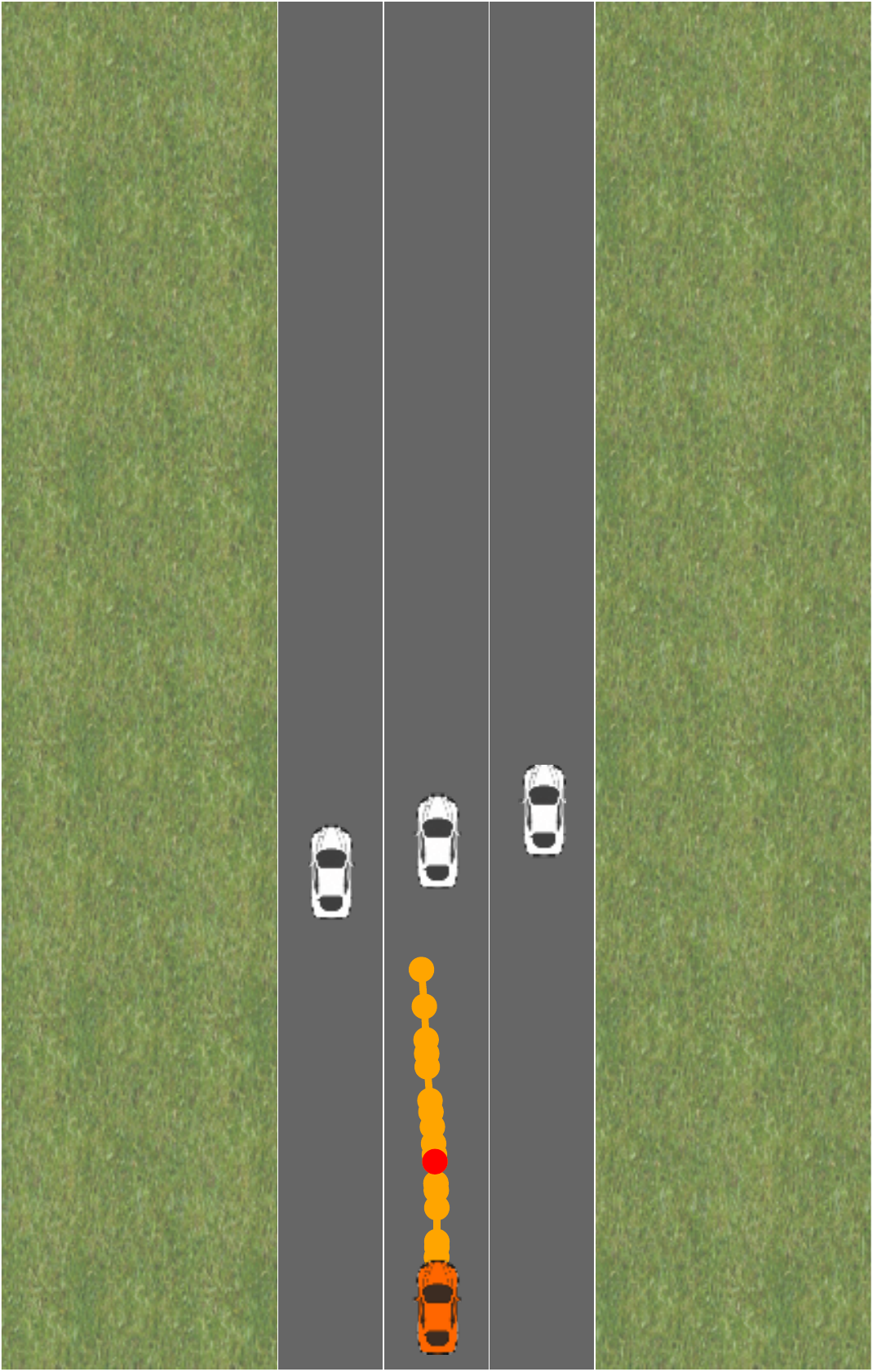}
       \end{tabular}
       \label{subfig:driver_constraint}
   }
   \caption{We want to select a controller for driving the orange car. In the base scenario (left image), the car should drive at velocity $v$, which we encode as reward. We model other driving rules, like ``usually drive in a lane'' or ``don't get too close to other cars'', either as a penalty on the reward function (in \subref{subfig:driver_reward_penalty}) or as a constraint (in \subref{subfig:driver_constraint}). In the ``different reward'' scenario (middle image), the car should pull over to the right of the street instead of keeping the velocity. If we reuse the same reward penalty for this task, the controller does not achieve the task because the penalty is too strong. So, we would have to tune the penalty for this new task, whereas the constrained controller still completes the new task safely without any tuning. In the ``different environment'' scenario (right image), the goal of driving at a target velocity remains the same, but the environment changes. In this changed environment, three vehicles block the road. Here, the penalized controller trades off the penalty with achieving high reward and tries to go through the cars to keep the velocity, which is too dangerous. On the other hand, the constraint formulation does not allow violating a constraint such as ``don't get too close to other cars'' which makes the orange car stop before the street is blocked. In \Cref{sec:driving_experiments}, we study this example in more detail.}
   \label{fig:driver_example}
\end{figure*}

\looseness -1 As a second example, consider finding safe control parameters for an autonomous car. A car manufacturer might have a set of controllers to choose from that perform a specific task, such as reaching a target destination as quickly as possible. The ideal controller achieves this task well and drives safely and comfortably. Whereas the objective -- travel time -- is easy to specify as a reward function, the constraints -- perceived safety and comfort -- may require feedback from human drivers and passengers. Similarly as in the previous example, the manufacturer's goal is to find the best, safe controller with as few trials that involve human feedback as possible. We assume that the controllers are evaluated in a simulation, so it is acceptable to evaluate an unsafe controller during training; however, the constraints have to be satisfied during deployment.

\looseness -1 In both examples, the decision-making problem is naturally characterized by an easy-to-evaluate part (the reward) and an expensive-to-evaluate part (the constraint). Additionally, we observe that constraints are more robust to changes in the environment and can be transferred to selecting controllers for different goals, in contrast to encoding the constraints as a penalty in the reward function (see \Cref{fig:driver_example}). Hence, in this paper, we study {\em learning about unknown, expensive-to-evaluate constraints}. 

\footnotetext[4]{\vspace{1cm}Example adapted from \citet{gelbart2014bayesian}.}

Specifically, we propose a two-phase approach to solving problems with unknown constraints. In the first phase, we learn to estimate the expensive-to-evaluate constraint function well enough for solving the constraint optimization problem. In the second phase, we recommend a solution. Constraint violations are allowed in the first phase, but the final recommendation has to satisfy the constraints.

\looseness -1
\paragraphsmall{Contributions.} We formalize learning about unknown constraints to find the best constrained solution as a novel linear bandit problem (\Cref{sec:cbai}) which we call \emph{\CBAILong} (\CBAIShort). We provide an instance-dependent sample complexity lower bound (\Cref{sec:lower_bound}) and propose \emph{\AlgLong} (\AlgShort), an algorithm that almost matches this lower bound (\Cref{sec:algorithm}). Our empirical evaluation shows that \AlgShort gets close to the performance of an oracle solution that has access to the true constraint function while outperforming a range of simpler baselines (\Cref{sec:bandit_experiments}). As a concrete application, we consider learning driving behavior in a simulation, where the constraints represent human preferences about driving behavior (\Cref{sec:driving_experiments}). We demonstrate empirically that \AlgShort can learn these constraints and propose heuristic variants of the algorithm that empirically improve sample efficiency. Additionally, we quantify the observation that learning driving preferences as constraints instead of rewards increases the robustness and transferability of the learned preferences.

\section{Related Work}

Learning constraints is similar to actively classifying arms as ``feasible'' or ``infeasible''; but, in contrast to typical active learning \citep{settles2009active}, we do not need to classify all arms. Instead, we only want to find the best feasible arm, which can require fewer samples than classifying all arms.
Our problem formalization as a linear \emph{multi-armed bandit} \emph{best-arm identification} problem \cite{audibert2010best} is similar to \citet{soare2014best} in the unconstrained setting, but focused on learning constraints.

\looseness -1 Much prior work on constraints in multi-armed bandits considers other notions of constraints than we do. For example, constraints holding in expectation rather than with high probability \cite{pacchiano2021stochastic}, or constraints in the form of a lower bound (threshold) on the reward \cite{locatelli2016optimal,kazerouni2017conservative,kano2019good,khezeli2020safe}.

\citet{amani2019linear} and \citet{moradipari2021safe} consider a linear bandit setting with a separate (linear) constraint function. Both differ from our work in three important ways: (1) they assume an unknown reward function whereas we assume the reward to be known, (2) they focus on cumulative regret minimization whereas we focus on best-arm identification, and (3) they require the constraints to be satisfied during exploration whereas we only require them to be satisfied for the final recommendation. These works adapt bandit algorithms based on upper confidence bounds \citep{amani2019linear} or Thompson sampling \citep{moradipari2021safe} to minimize regret in the constraint setting. To enable safe exploration, they need to assume a convex and compact set of arms; we do not require this assumption.

\citet{wang2021best} also study best-arm identification with linear constraints. In contrast to our work, they assume unknown rewards and focus on safety constraints that must be satisfied during exploration. To make this possible, they need to make more assumptions about the structure of the set of arms. In particular, they assume that the decision maker can only query in each dimension independently. Because of this, their algorithm cannot be applied to our setting without significant changes.

Our algorithm is conceptually similar to other bandit algorithms based on the principle of eliminating sub-optimal arms step-by-step. Our theoretical analysis employs similar tools as that used for best-arm-identification in unconstrained linear bandits \citep{soare2014best,fiez2019sequential}.

Some works in Bayesian optimization (BO) also study the problem of exploring to find the best constrained solution to a problem with an expensive-to-evaluate constraint function. Common approaches heuristically extend BO methods to incorporate an unknown constraint function \citep{gardner2014bayesian,hernandez2016general,perrone2019constrained}. In contrast to this line of work, we obtain sample complexity guarantees by focusing on linear constraint functions. Similar to the bandit literature, most work on BO with constraints focuses on the setting where safety constraints must hold during exploration \citep[e.g.,][]{sui2015safe}.

We apply our algorithm to the problem of learning from human preferences, which is essential for building systems with hard-to-specify goals, e.g., in robotics \citep{daniel2014active}. We use an environment by \citet{dorsa2017active} who model human preferences as rewards rather than constraints.

\section{The Linear Constrained Best-Arm Identification Problem}\label{sec:cbai}

We seek to find the best constrained solution from a discrete set of options represented by feature vectors $\Act \in \ActSet \subset \R^d$. We assume that both the known reward function and the unknown constraint function are linear in $\Act$.

\begin{definition}
A \emph{\CBAILong} (\CBAIShort) problem $\Problem = (\ActSet, \RewParam, \ConstParam, \tau)$ consists of a finite action set $\ActSet\subset\R^d$, a reward parameter $\RewParam \in \R^d$, a constraint parameter $\ConstParam \in \R^d$, and threshold $\tau \in \R$. The decision-maker knows $\ActSet$ and $\RewParam$, but not $\ConstParam$ and $\tau$. In each iteration, the decision-maker selects an arm $\Act\in\ActSet$ and observes $\ConstParam^T \Act + \noise_\Act$, where $\noise_\Act$ is sub-Gaussian noise. Their goal is to identify a constrained optimal arm
\[
\Act^* \in \argmax\limits_{\Act \in \ActSet, \ConstParam^T \Act \leq \tau} \RewParam^T \Act
\]
within as few iterations as possible.
\end{definition}
In our initial example, $\ActSet$ contains all potential cookie recipes. $\RewParam^T \Act$ encodes the amount of calories for recipe $\Act$, which the decision-maker knows and wants to minimize. $\ConstParam^T x$ encodes the unknown customer preferences, which the decision-maker must infer from as few experiments as possible. In the following, we assume w.l.o.g. $\tau = 0$ but generalization to $\tau \neq 0$ is straightforward. If $\tau$ is unknown, we can simply model it as a constant shift in the const To simplify notation, we omit $\tau$ and talk about a \CBAIShort problem $\Problem = (\ActSet, \RewParam, \ConstParam)$.

\subsection{Lower Bounds}\label{sec:lower_bound}

We first provide a lower bound on the sample complexity of solving a given \CBAIShort problem. The following theorem states how many samples are necessary to distinguish a given \CBAIShort instance from the closest instance with a different solution, which is necessary to solve an instance.

\begin{restatable}[\CBAIShort lower bound]{theorem}{LowerBound}\label{thm:lower_bound}
Assume $\noise_\Act \sim \NormalDist(0, 1)$ for all $\Act\in\ActSet$. For any \CBAIShort problem $\Problem = (\ActSet, \RewParam, \ConstParam)$, there exists another \CBAIShort problem $\Problem' = (\ActSet, \RewParam, \ConstParam')$ with the same set of actions $\ActSet$ and reward parameter $\RewParam$ but a different constraint parameter and optimal arm, such that the expected number of iterations $\StoppingTime$ needed by any allocation strategy that can distinguish between $\Problem$ and $\Problem'$ with probability at least $1-\delta$ is lower bounded as
\[
\Expectation[\StoppingTime] \geq 2 \log \left( \frac{1}{2.4\delta} \right) \max_{\Act \in \BetterActSet(\BestAct)} \frac{\| \Act \|_{ \DesMat^{-1} }^2}{(\ConstParam^T \Act)^2},
\]
\looseness -1 where $\Design$ is a probability distribution over arms which the allocation strategy follows, i.e., $\Design(\Act)$ is the probability that it pulls arm $\Act$, $\DesMat = \sum_\Act \Design(\Act) \Act \Act^T$ is the design matrix, $\BetterActSet(\BestAct) = \{ \Act' \in \ActSet | \RewParam^T \Act' \geq \RewParam^T \BestAct \}$ is the set of all arms with reward no less than $\BestAct$, the optimal arm for problem $\Problem$.
\end{restatable}

\looseness -1 The proof in \Cref{app:proof_lower_bound} uses a proof strategy similar to that for lower bounds for standard linear bandits \citep{soare2014best,soare2015sequential,fiez2019sequential}. We consider the log-likelihood ratio of making a series of observations in instance $\Problem$ compared to $\Problem'$ and consider how we can choose $\Problem'$ to have a different solution but a small log-likelihood ratio, i.e., the decision-maker makes similar observations as if they were in $\Problem$. In contrast to the standard linear bandit case, we need to carefully reason about the constraints when ensuring that $\Problem'$ has a different solution than $\Problem$. We distinguish the case that the solution of $\Problem$ is infeasible in $\Problem'$ and the case that an arm with larger reward is feasible in $\Problem'$ but not $\Problem$. Reasoning about these two cases yields the result.

Our lower bound has a similar form as those for best-arm identification in linear bandits \citep{soare2014best}. In particular, we have the same uncertainty term in the numerator. Instead of a suboptimality gap in the denominator, we get the distance to the constraint boundary: the problem is harder if arms are closer to the constraint boundary. However, our maximization is over individual arms instead of directions, i.e., pairs of arms, and the set $\BetterActSet(\BestAct)$ does not appear in the linear bandit case.

We want to characterize the sample complexity of different algorithms for solving the \CBAIShort problem. To this end, let us define the sample complexity of a given problem instance using the lower bound we just derived.

\begin{definition}[\CBAIShort sample complexity]\label{def:sample_complexity}
We define the \emph{sample complexity} of a \CBAIShort problem $\Problem$ as
\begin{align*}
    \Complexity(\Problem) \defeq \min_\lambda \max_{\Act \in \BetterActSet(\BestAct)} \frac{\| \Act \|_{A_{\lambda}^{-1}}^2}{(\TrueConst^T \Act)^2}
\end{align*}
\end{definition}

This describes the best sample complexity that any algorithm can achieve on \CBAIShort problem $\Problem$. It will also be helpful to have a worst-case upper bound on $\Complexity(\Problem)$ as a point of comparison, which the next proposition provides.

\begin{restatable}{proposition}{InstanceIndependentLowerBound}\label{thm:instance_independent_lower_bound}
For any \CBAIShort problem $\Problem$, we have
$\Complexity(\Problem) \leq d / (\MinConst)^2$,
where $\MinConst = \min_{\Act\in\ActSet} |\TrueConst^T \Act|$. This bound is tight, i.e, there is an instance $\Problem$, such that we have $\Complexity(\Problem) = d / (\MinConst)^2$.
\end{restatable}
This results indicates that a \CBAIShort problem is harder if it has a larger dimension $d$, or if the distance of the arm that is closest to the constraint boundary ($\MinConst$) is smaller.
This worst case bound corresponds to situations where all arms are linearly independent and pulling one arm does not provide any information about any other arm.

\paragraphsmall{Oracle solution.}
We can make the definition of sample complexity more concrete by considering an oracle solution that has access to the true constraint value to select which arms to query. The oracle selects arms by explicitly minimizing $\Complexity$:
\[
\lambda^{\star} \in \argmin_\lambda \max_{\Act \in \BetterActSet(\BestAct)} \frac{\| \Act \|_{A_{\lambda}^{-1}}^2}{(\TrueConst^T \Act)^2}.
\]
The oracle prefers arms with high uncertainty (high $\| \Act \|_{A_{\lambda}^{-1}}^2$) and arms close to the constraint boundary (low $(\TrueConst^T \Act)^2$). Moreover, it focuses on reducing the uncertainty about arms that have higher reward than the true optimal arm (arms in $\BetterActSet(\BestAct)$). In \Cref{app:oracle}, we show that this oracle solution has sample complexity on order $\Complexity(\Problem)$, i.e., it is indeed optimal.

\subsection{Confidence Intervals for Linear Regression}

Our algorithms rely on high probability confidence intervals on the linear constraints constructed from observations. Hence, let us briefly review how to construct such confidence intervals from observations with sub-Gaussian noise.

Suppose, an algorithm queried a sequence of arms $\vx_\Round = (\Act_1, \dots, \Act_\Round)$. For a given $\Act_i$, it observed $\tilde{y}_i = \TrueConst^T \Act_i + \noise_{\Act_i}$, where $\TrueConst$ is the true constraint parameter, and $\noise_{\Act_i}$ is sub-Gaussian noise. We now aim to find confidence intervals such that $\TrueConst^T \Act \in [\LowerConst^\Round(\Act), \UpperConst^\Round(\Act)]$ with probability at least $1-\delta$, where $\LowerConst^\Round(\Act) = \hat{\ConstParam}^T \Act - \sqrt{\beta_\Round} \| \Act \|_{A_{\vx_\Round}^{-1}}$ and $\UpperConst^\Round(\Act) = \hat{\ConstParam}^T \Act + \sqrt{\beta_\Round} \| \Act \|_{A_{\vx_\Round}^{-1}}$. Based on these confidence intervals, we can decide whether a given arm is likely feasible or not.

\looseness -1 If the queries follow a distribution that does not depend on the observations, it is straightforward to derive confidence intervals (e.g., Chapter 20 in \citet{lattimore2020bandit}).

\begin{restatable}[]{proposition}{OLSBoundStatic}
\label{thm:ols_bound_static}
Let $\vx_\Round = (\Act_1, \dots, \Act_\Round)$ be a sequence of arms from a fixed allocation for which we have observed $\TrueConst^T \Act_i + \noise_{\Act_i}$ where $\noise_{\Act_i}$ is independent sub-Gaussian noise. If we estimate $\hat{\ConstParam}$ from the observations using least-squares regression and choose $\beta_\Round = \sqrt{2\log(|\ActSet|/\delta)}$ then we have
$
P(\exists \Act\in\ActSet: \TrueConst^T \Act \notin [\LowerConst^\Round(\Act), \UpperConst^\Round(\Act)]) \leq \delta
$.
\end{restatable}
However, in sequential decision-making we usually want to adapt our strategy after making observations. In this case, we need to be more careful in constructing confidence intervals, as observed by \citet{abbasi2011improved}. Unfortunately, the resulting confidence intervals are weaker than those for static allocations by a factor of $\sqrt{d}$.

\begin{restatable}[Theorem 2 by \citet{abbasi2011improved}]{proposition}{OLSBoundAdaptive}
\label{thm:ols_bound_adaptive}
Let $\vx_\Round = (\Act_1, \dots, \Act_\Round)$ be a sequence of points selected with a possibly adaptive strategy for which we have observed $\TrueConst^T x_i + \noise_{\Act_i}$ where $\noise_{\Act_i}$ is independent sub-Gaussian noise. Assume, that $\| \ConstParam \|_2 \leq S$ and $\| \Act \|_2 \leq L$ for all $\Act\in\ActSet$. If we estimate $\hat{\ConstParam}$ from the observations using least-squares regression, then for every $\Act \in \R^d$ and for all $t \geq 0$:
$
P(\exists \Act\in\ActSet: \TrueConst^T \Act \notin [\LowerConst^\Round(\Act), \UpperConst^\Round(\Act)]) \leq \delta
$
with $\beta_\Round = \sqrt{d\log\left( (1+tL^2/\lambda) / \delta \right)} + \sqrt{\lambda} S$.
\end{restatable}

\subsection{Algorithms Using Static Confidence Intervals}\label{sec:algorithm}

To design an algorithm for solving \CBAIShort problems, we need to decide (1) which arms to pull during exploration and (2) when we can stop the algorithm and return the correct arm with high probability. First, let us address the second question and then get back to the first one.

\paragraphsmall{Stopping condition.}
Using the past observations, we can define confidence intervals for the constraint value of each arm. Let $\LowerConst^\Round(\Act)$ and $\UpperConst^\Round(\Act)$ be such that we know with high probability (w.h.p.) $\TrueConst^T \Act \in [\LowerConst^\Round(\Act), \UpperConst^\Round(\Act)]$. Now we can also determine w.h.p.\ that all arms with $\LowerConst^\Round(\Act) > 0$ are infeasible, and all arms with $\UpperConst^\Round(\Act) \leq 0$ are feasible. Moreover, we can identify suboptimal arms by considering $\bar{r} = \max_{\UpperConst^\Round(\Act) \leq 0} \RewParam^T \Act$. The solution to this optimization problem are the highest-reward arms that are feasible w.h.p. Therefore, all arms with reward less than $\bar{r}$ are clearly suboptimal. Combining these observations, we can define a set of arms that we are uncertain about, i.e., that could still be optimal:
\[
\UncertainSet_\Round = \{ \Act\in\ActSet | \LowerConst^\Round(\Act) \leq 0 \text{ and } \UpperConst^\Round(\Act) > 0 \text{ and } \RewParam^T \Act > \bar{r} \}
\]
Note, that if $\UncertainSet_\Round$ is empty, we can stop and return an arm in $\argmax_{\UpperConst^\Round(\Act) \leq 0} \RewParam^T \Act$. This arm will be optimal w.h.p.

\paragraphsmall{Arm selection criterion.} In each iteration, we have to decide which arm to pull. We could, e.g., combine the above stopping condition with querying uniformly random arms. This algorithm would return the correct optimal arm with high probability. However, random querying will usually not be the most sample efficient approach. Another natural approach is to select the arms that we are most uncertain about, which is sometimes called \emph{uncertainty sampling}. We could, e.g., choose a fixed allocation
\[
\Design^{\mathrm{G}} \in \argmin_{\Design} \max_{\Act\in\ActSet} \| \Act \|_{\DesMat^{-1}}.
\]
This approach is also called \emph{G-Allocation} in the experimental design literature. We show in \Cref{app:g-allocation} that G-Allocation matches the worst-case lower bound in \Cref{thm:instance_independent_lower_bound}. However, we can do better by {\em focusing on arms that we cannot yet exclude as being certainly feasible, infeasible, or suboptimal}. Concretely, we modify G-Allocation to reduce uncertainty only about arms in $\UncertainSet_\Round$:
\[
\Design^{\mathrm{ACOL}} \in \argmin_{\Design} \max_{\Act\in\UncertainSet_\Round} \| \Act \|_{\DesMat^{-1}}
\]

\paragraphsmall{Rounding.} All algorithms implementing a static allocation require a rounding procedure to translate an allocation $\Design$ into a finite sequence of arms $\Act_1, \dots, \Act_n$. The experimental design literature provides various efficient rounding procedures that are $\varepsilon$-approximate. We use a standard procedure described in Chapter 12 of \citet{friedrich2006optimal}.

\paragraphsmall{\AlgLong (\AlgShort).} \Cref{adaptive_algorithm} shows the full algorithm we call \emph{\AlgLong} (\AlgShort). The algorithm proceeds in rounds. In each round $\Round$ it pulls arms to reduce the uncertainty about arms in $\UncertainSet_\Round$, then updates $\UncertainSet_\Round$, and decides if it can stop and return a recommendation. The round length $N_t$ is chosen carefully to allow us to provide a tight sample complexity result.

\begin{algorithm}[t]
\caption{\AlgLong (\AlgShort).}
\label{adaptive_algorithm}
\begin{algorithmic}[1]
    \STATE \textbf{Input:} significance $\delta$
    \STATE $\UncertainSet_1 \gets \ActSet$ \hfill (uncertain arms)
    \STATE $\FeasibleSet_1 \gets \emptyset$ \hfill (feasible arms)
    \STATE $\Round \gets 1$ \hfill (round)
    \WHILE{$\UncertainSet_\Round \neq \emptyset$}
        \STATE $\delta_\Round \gets \delta^2 / \Round^2$
        \STATE $\Design^*_\Round \gets \argmin_{\Design} \max_{\Act \in \UncertainSet_\Round} \| \Act \|_{\DesMat^{-1}}^2$
        \STATE $\rho^*_\Round \gets \min_{\Design} \max_{\Act \in \UncertainSet_\Round} \| \Act \|_{\DesMat^{-1}}^2$
        \STATE $N_\Round \gets \max \left\{ \left\lceil 2^{2\Round + 3} \log\left(\frac{|\ActSet|}{\delta_t}\right) (1+\varepsilon) \rho^*_\Round \right\rceil, r(\varepsilon) \right\}$
        \STATE $\vx_{N_\Round} \gets \mathtt{Round}(\Design_\Round^*, N_\Round)$
        \STATE Pull arms $\Act_1, \dots, \Act_{N_\Round}$ and observe constraint values
        \STATE $\Round \gets \Round + 1$
        \STATE Update $\hat{\ConstParam}_\Round$ and $A$ based on new data
        \STATE $\LowerConst^\Round(\Act) \gets \hat{\ConstParam}_\Round^T \Act - \sqrt{\beta_t} \| \Act \|_{A^{-1}}$ for all arms $\Act$
        \STATE $\UpperConst^\Round(\Act) \gets \hat{\ConstParam}_\Round^T \Act + \sqrt{\beta_t} \| \Act \|_{A^{-1}}$ for all arms $\Act$
        \STATE $\FeasibleSet_\Round \gets \FeasibleSet_{\Round-1} \cup \{ \Act | \UpperConst^\Round(\Act) \leq 0 \}$
        \STATE $\bar{r} \gets \max_{\Act\in\FeasibleSet_\Round} \RewParam^T \Act$
        \STATE $\UncertainSet_\Round \gets \UncertainSet_{\Round-1} \setminus \{ \Act | \LowerConst^\Round(\Act) > 0 \} \setminus \{ \Act | \UpperConst^\Round(\Act) \leq 0 \}$ \\
        \hspace{1.7cm}$\setminus \{ \Act | \RewParam^T \Act < \bar{r} \}$
    \ENDWHILE
    \STATE \textbf{return} $\Act^* \in \argmax_{\Act\in\FeasibleSet_\Round} \RewParam^T \Act$
\end{algorithmic}
\vspace{0.5cm} %
\end{algorithm}

The following theorem -- the main theoretical result of our paper -- establishes that \AlgShort returns the correct optimal solution to any \CBAIShort problem and provides an upper bound on the number of samples necessary.

\begin{restatable}[\AlgShort sample complexity]{theorem}{AdaptiveAlgorithmComplexity}
\label{thm:adaptive_algorithm_complexity}
Assume \Cref{adaptive_algorithm} is implemented with an $\varepsilon$-approximate rounding strategy. Then, after $N$ iterations the algorithm returns an optimal arm with probability at least $1-\delta$, and we have:
\begin{align*}
N &\leq 8 \log\left( \frac{|\ActSet|\bar{t}^2}{\delta^2} \right) (1+\varepsilon) \sum_{\Round=1}^{\bar{t}} \min_{\lambda} \max_{\Act \in \UncertainSet_\Round} \frac{\| \Act \|_{\DesMat^{-1}}^2}{(\TrueConst^T \Act)^2} + \bar{t} \\
&\leq 8 \log\left( \frac{|\ActSet|\bar{t}^2}{\delta^2} \right) (1+\varepsilon) \bar{t} \WeakComplexity(\Problem) + \bar{\Round}
\end{align*}
where $\WeakComplexity(\Problem) = \min_{\Design} \max_{\Act \in \ActSet} \| \Act \|_{\DesMat^{-1}}^2 / (\TrueConst^T \Act)^2$, and $\bar{t} = \left\lceil - \log_2 \MinConst \right\rceil$. Moreover, $\WeakComplexity(\Problem) \leq d/(\MinConst)^2$
\end{restatable}

We prove the theorem in \Cref{app:proofs}. The key step uses \Cref{thm:ols_bound_static} to show that the confidence intervals shrink exponentially. This implies that in a logarithmic number of rounds, the largest confidence interval will be less than $\MinConst$; and once this is the case, $\UncertainSet_\Round$ is empty and the algorithm returns the correct solution. Combining this with the round lengths of $N_t$ allows us to prove the result.

The sample complexity of \AlgShort is of order $\WeakComplexity(\Problem)$, except for logarithmic factors. Also, we show that $\WeakComplexity \leq d/(\MinConst)^2$, so the bound matches the lower bound of \Cref{thm:instance_independent_lower_bound} for worst-case instances, but it is much tighter for benign instances. In particular, the bound in \Cref{thm:adaptive_algorithm_complexity} contains the same min-max problem as the instance dependent sample complexity $\Complexity(\Problem)$, only with the maximization being over different sets, namely $\UncertainSet_\Round$ instead of $\BetterActSet(\BestAct)$. Note, that we cannot expect a practical algorithm to only explore arms in $\BetterActSet(\BestAct)$ because we do not know $\BestAct$ a priori. Instead, \AlgShort explores in $\UncertainSet_\Round$, a conservative estimate of $\BetterActSet(\BestAct)$ that shrinks over time given the knowledge so far. \Cref{thm:adaptive_algorithm_complexity} does not exactly match the instance dependent lower bound, but the difference only depends on how well $\UncertainSet_\Round$ approximates the set of relevant arms.

\subsection{Algorithms Using Adaptive Confidence Intervals}

\begin{algorithm}[t]
\caption{\AlgGreedyLong (\AlgGreedyShort).}
\label{fully_adaptive_algorithm}
\begin{algorithmic}[1]
    \STATE \textbf{Input:} $\beta_t$, $\lambda$
    \STATE initialize $\hat{S}_1(\hat{\ConstParam})$, $\UncertainSet_1 \gets \ActSet$, $\FeasibleSet_1 \gets \emptyset$, $A \gets \lambda I$, $\Round \gets 1$
    \WHILE{$\UncertainSet_\Round \neq \emptyset$}
        \STATE $\Act^* \gets \max_{\Act \in \UncertainSet_\Round} \| \Act \|_{A^{-1}}^2$
        \STATE Pull arm $\Act^*$ and observe constraint value
        \STATE $\Round \gets \Round + 1$, $A \gets A + \Act \Act^T$
        \STATE $\LowerConst^\Round(\Act) \gets \hat{\ConstParam}_\Round^T \Act - \sqrt{\beta_t} \| \Act \|_{A^{-1}}$ for all arms $\Act$
        \STATE $\UpperConst^\Round(\Act) \gets \hat{\ConstParam}_\Round^T \Act + \sqrt{\beta_t} \| \Act \|_{A^{-1}}$ for all arms $\Act$
        \STATE $\FeasibleSet_\Round \gets \FeasibleSet_{\Round-1} \cup \{ \Act | \UpperConst^\Round(\Act) \leq 0 \}$
        \STATE $\bar{r} \gets \max_{\Act\in\FeasibleSet_\Round} \RewParam^T \Act$
        \STATE $\UncertainSet_\Round \gets \UncertainSet_{\Round-1} \setminus \{ \Act | \LowerConst^\Round(\Act) > 0 \} \setminus \{ \Act | \UpperConst^\Round(\Act) \leq 0 \}$ \\
        \hspace{1.7cm}$\setminus \{ \Act | \RewParam^T \Act < \bar{r} \}$
    \ENDWHILE
    \STATE \textbf{return} $\Act^* \in \argmax_{\Act\in\FeasibleSet_\Round} \RewParam^T \Act$
\end{algorithmic}
\end{algorithm}

Whereas the algorithm we just introduced comes with a strong sample complexity guarantee, it is impractical in various ways, primarily because of the round-based structure. In particular, the algorithm requires a rounding procedure to determine a sequence of actions; it then follows this sequence for a predefined round length and can not stop before finishing a round. Also, in between rounds, the algorithm discards all previously made observations, which is necessary to apply \Cref{thm:ols_bound_static}.

Next, we present an alternative version of this algorithm that uses the adaptive confidence intervals of \Cref{thm:ols_bound_adaptive}. This allows us to remove the round-based structure in favor of a greedy algorithm that does not have the same limitation. This algorithm, which we call \emph{\AlgGreedyLong} (\AlgGreedyShort), is shown in \Cref{fully_adaptive_algorithm}. Unfortunately, for \AlgGreedyShort, we can only provide significantly weaker sample complexity guarantees; but we find it performs well empirically.

Since the adaptive confidence intervals hold for all $\Round > 0$ simultaneously, we can now check the stopping condition after each sample. Instead of determining a static allocation that reduces uncertainty about the uncertain arms, we now greedily select the arm to pull that reduces uncertainty within $\UncertainSet_\Round$ the most.  Thanks to \Cref{thm:ols_bound_adaptive}, this algorithm still stops and returns the correct solution. However, it achieves worse sample complexity due to the additional factor of $\sqrt{d}$ in \Cref{thm:ols_bound_adaptive}.

\paragraphsmall{Heuristic modifications.}
There is a variety of heuristic modifications that we can make to \AlgGreedyShort to improve its practical performance at the cost of losing some theoretical guarantees. First, we could use a different query rule within the set of uncertain arms, such as uniformly random querying, which reduces computational cost. Second, the $\beta_t$ resulting from \Cref{thm:ols_bound_adaptive} tends to be very large. In practice, we can try to tune $\beta_t$ to get good confidence intervals that are much smaller than the ones suggested by the theory. Third, we can turn the algorithm into an ``anytime'' algorithm by defining a recommendation rule, such as recommending the best arm that is certainly feasible. Then, we can stop the algorithm after an a priori unknown budget of queries and receive a best guess for the optimal arm.

\section{Experiments}\label{sec:experiments}

\begin{figure*}
    \centering
    \subfigure[Irrelevant dimensions]{
        \includegraphics[width=0.235\linewidth]{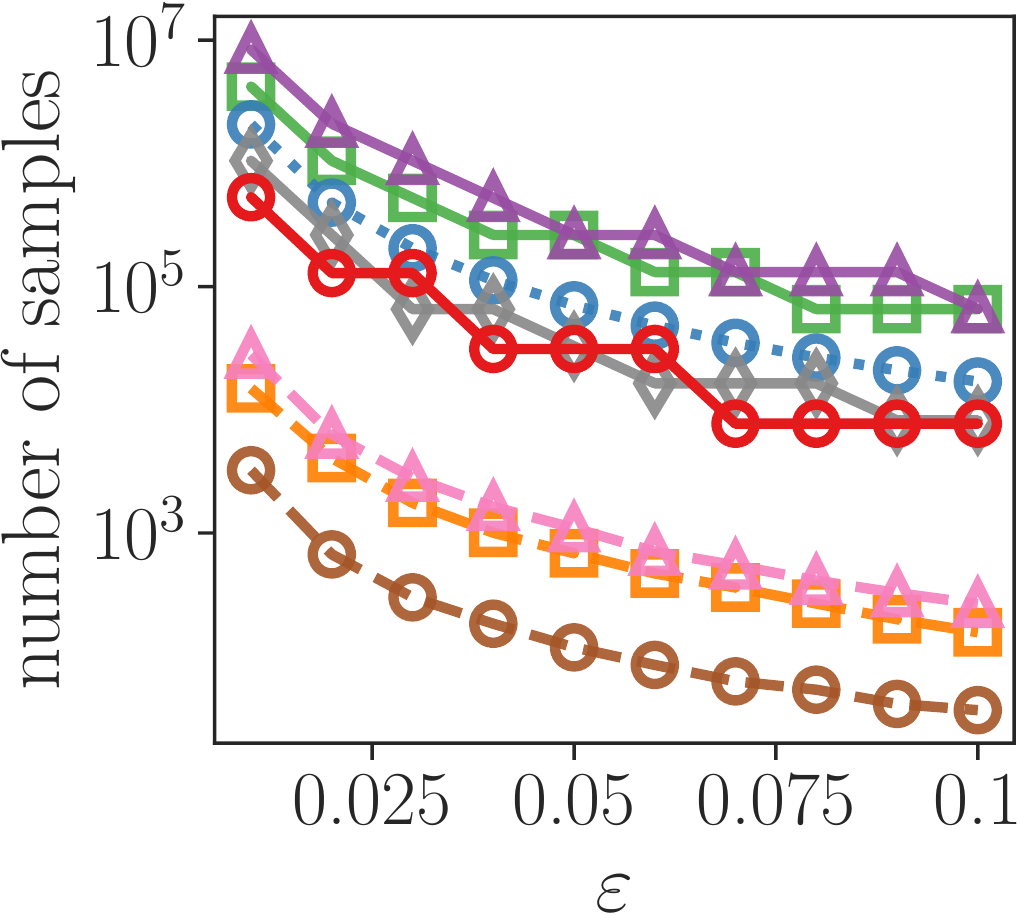}\hspace{0.1em}
        \includegraphics[width=0.235\linewidth]{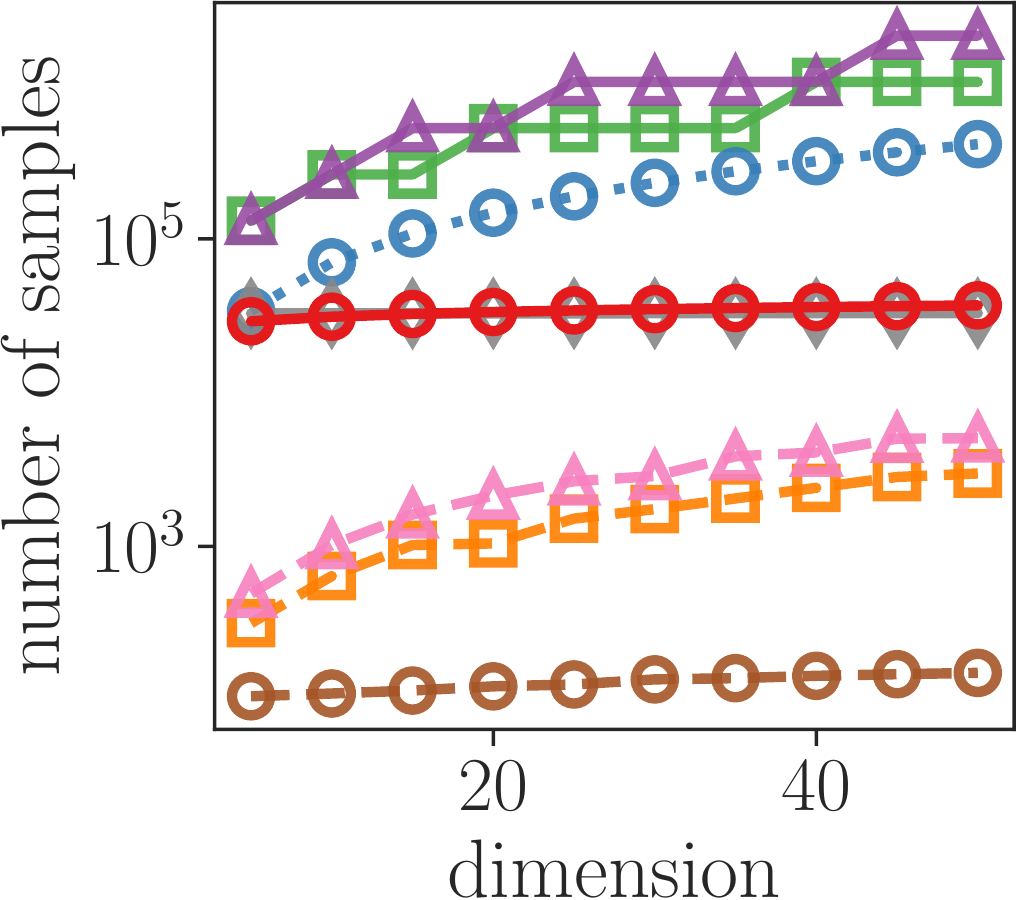}
        \label{subfig:results_irrelevant_dimensions}
    }\hfill
    \subfigure[Unit sphere]{
        \includegraphics[width=0.235\linewidth]{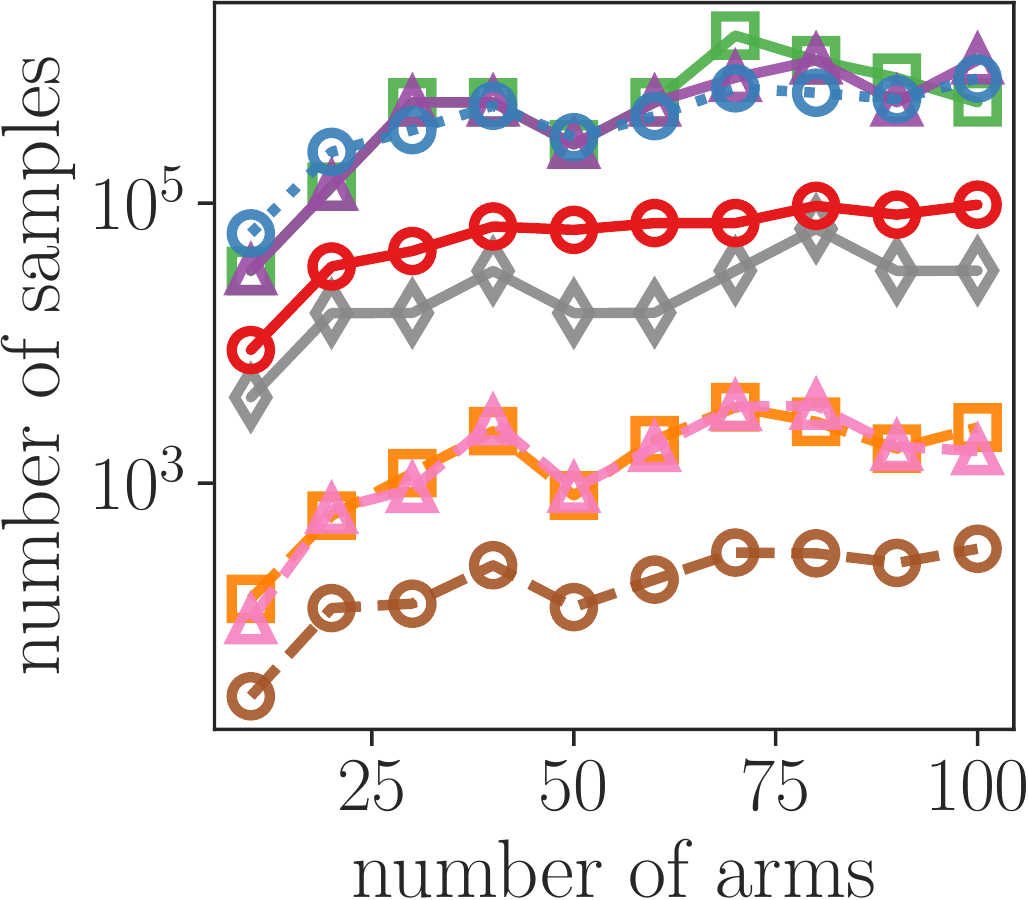}\hspace{0.1em}
        \includegraphics[width=0.235\linewidth]{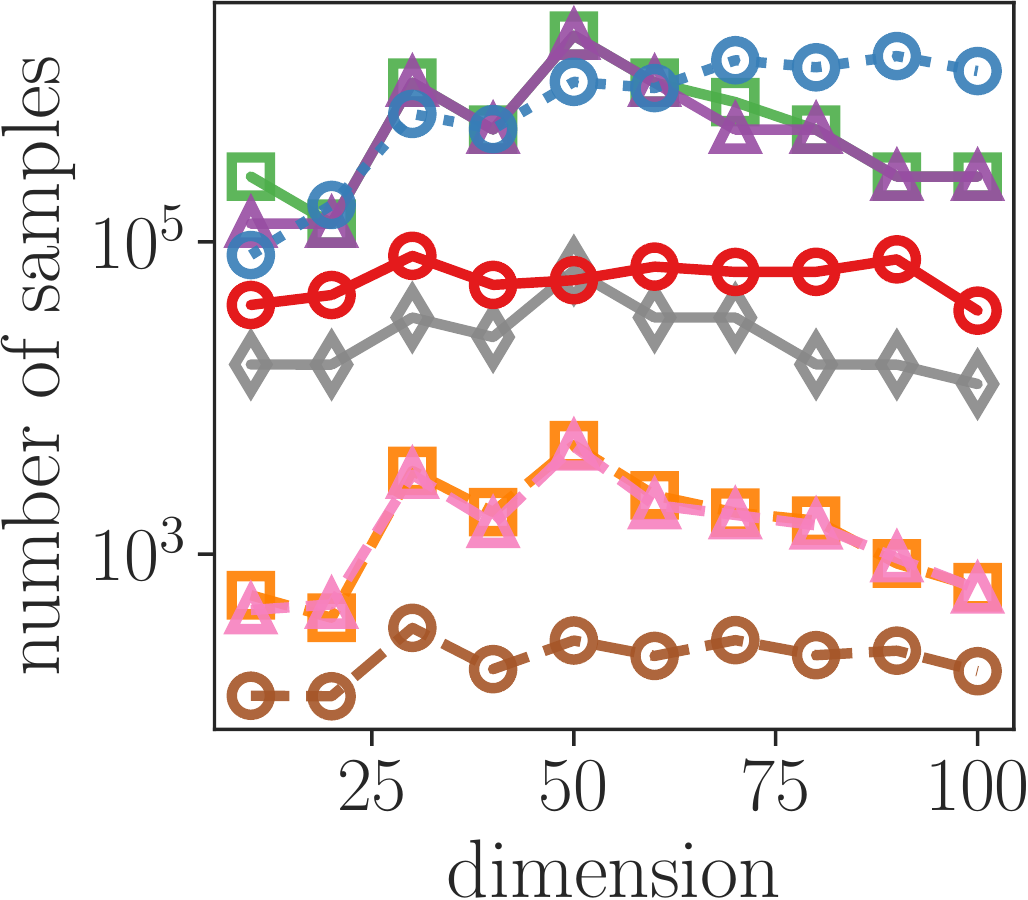}
        \label{subfig:results_unit_sphere}
    }
    {\small
    \begin{tabular}{llllllll}
        \legendUniform & Uniform &
        \legendGAllocation & G-Allocation &
        \legendAdaptiveStatic & \AlgShort (ours) &
        \legendOracle & Oracle \\
        \legendUniformAllTuned & Adaptive Uniform (tuned) &
        \legendMaxvarAllTuned & Greedy MaxVar (tuned) &
        \legendAdaptiveTuned & \AlgGreedyShort (tuned) &
        \legendAdaptive & \AlgGreedyShort (theory)
    \end{tabular}
    }
    \caption{For both synthetic instances, we plot the median number of iterations for finding the constrained optimal solution as a function of different parameters of the problem instance. All methods return the correct constrained optimal solution. For ``irrelevant dimensions'', we vary $\varepsilon$ for fixed $d=10$, and $d$ for fixed $\varepsilon = 0.05$. For ``unit sphere'', we vary $n$ for fixed $d=30$, and $d$ for fixed $n=30$. Note that for ``unit sphere'', the instances are randomly sampled for each random seed, whereas the ``irrelevant dimensions'' instance stays the same. For legibility, we only show the median computed over 30 random seeds, and omit a few of the baselines we evaluated. For plots with all baselines that include confidence intervals, see \Cref{app:more_results}. Overall, \AlgShort is the most sample efficient approach of all algorithms that provide theoretical guarantees. In rare cases, it even needs fewer samples than the oracle. However, this is mostly an artifact of both algorithms using slightly different round lengths (cf. \Cref{app:experiment_details}). By ``tuning'' $\beta_t$, we can gain several orders of magnitude in sample efficiency at the cost of theoretical guarantees. \AlgGreedyShort remains the most sample efficient among these tuned approaches.
    }
    \label{fig:bandit_results}
\end{figure*}

\looseness -1 
We perform three experiments. First, in \Cref{sec:bandit_experiments}, we consider synthetic \CBAIShort instances to evaluate \AlgShort and compare it to natural baselines. Additionally, we investigate the effect of various heuristic modifications to the algorithm. Second, in \Cref{sec:regret_compare_experiment}, we compare \AlgShort to algorithms that safely minimize regret. And, third, in \Cref{sec:driving_experiments}, we consider learning constraints that represent human preferences in a simulated driving scenario. This experiment illustrates how to model preference learning problems as \CBAIShort problems. In the driving simulation, we also demonstrate the benefits of learning constraints in terms of robustness and transferability.

We provide more details on the experiments in \Cref{app:experiment_details} and we provide the full source code to reproduce our experiments.\footnote{\codeurl} For all experiments we use a significance of $\delta = 0.05$ and, if not stated differently, observations have Gaussian noise with $\sigma = 0.05$.

\subsection{Synthetic Experiments}\label{sec:bandit_experiments}

We consider two synthetic \CBAIShort instances and a range of baselines and multiple variants of \AlgShort / \AlgGreedyShort. 

\paragraphsmall{Instance 1 -- Irrelevant dimensions.}
First, we consider \CBAIShort instances which contain a number of dimensions that are irrelevant for learning the correct constraint boundary. The problems have dimension $d$, and $d+1$ arms: $\Act_1, \dots, \Act_{d+1}$. For each $i=1, \dots, d-1$, we have $\Act_i = \mathbf{e}_i$, whereas $\Act_d = (1-\varepsilon) \mathbf{e}_d$, and $\Act_{d+1} = (1+\varepsilon) \mathbf{e}_d$, for some $\varepsilon > 0$. $\mathbf{e}_i$ denotes the $i$-th unit vector. The reward and constraint parameter are both $\RewParam = \ConstParam = \mathbf{e}_d$. We define a threshold $\tau=1$; hence, $\Act_1, \dots, \Act_{d-1}$ are feasible but suboptimal, $\Act_d$ is optimal and $\Act_{d+1}$ is infeasible. Importantly, the arms $\Act_1, \dots, \Act_{d-1}$ are ``irrelevant'' to finding the correct constraint boundary between $\Act_d$ and $\Act_{d+1}$. An ideal algorithm would focus its queries primarily on $\Act_d$ and $\Act_{d+1}$. We can vary the problem difficulty by changing $\varepsilon$ (more difficult for small values), and $d$ (more difficult for large values).

\paragraphsmall{Instance 2 -- Unit sphere.}
To create \CBAIShort instances with a range of different reward and constraint functions, we sample arms $\Act_1, \dots, \Act_n$ uniformly from a $d$-dimensional unit sphere. We also sample the reward parameter $\RewParam$ from the unit sphere. As constraint parameter, we choose $\ConstParam = \Act_i - \Act_j$ where $\Act_i$ and $\Act_j$ are the two closest arms in $\ell_2$-distance. We can increase the problem difficulty by increasing the dimension $d$ and the number of arms $n$.

\paragraphsmall{Baselines.}
We compare \AlgShort and \AlgGreedyShort to various baselines. The \emph{Oracle} solution uses knowledge of the true constraint parameter to choose the best possible static allocation (cf. \Cref{app:oracle}). In practice, we cannot implement the oracle because we do not know the constraint parameter; but, it yields a performance upper bound to which we can compare other algorithms. \emph{G-Allocation} uses a static allocation that uniformly reduces uncertainty (cf. \Cref{app:g-allocation}), whereas \emph{Uniform} pulls all arms with equal probability. We also consider variants of these algorithms that use the adaptive confidence interval in \Cref{thm:ols_bound_adaptive}. We call the adaptive version of G-Allocation \emph{Greedy MaxVar} because it greedily selects arms with the highest uncertainty esimate from $\UncertainSet_\Round$. We call uniform sampling with the adaptive confidence intervals \emph{Adaptive Uniform} respectively. For all algorithms that use adaptive confidence intervals, in addition to the version using \Cref{thm:ols_bound_adaptive}, we test a ``tuned'' version that considers $\beta_t$ as a numeric hyperparameter instead (indicated by the name of the algorithms followed by \textit{(tuned)}). We chose $\beta_t = \frac14$, for all experiments, which we determined from minimal tuning on the ``irrelevant dimensions'' instance for the Greedy MaxVar algorithm.
For clarity, we omit a few of the baselines that perform poorly in our plots. \Cref{app:more_results} provides the full results.

\paragraphsmall{Results.}
\Cref{fig:bandit_results} shows our results in the synthetic \CBAIShort instances. All algorithms find the correct solution, but their sample efficiency varies widely. From all algorithms with theoretical guarantees, the (unrealistic) oracle solution needs the fewest number of iterations, as expected. But \AlgShort can get close to the oracle performance and outperforms G-Allocation and uniform sampling in all cases. For example, if we increase the number of irrelevant dimensions in the first experiment, G-Allocation and uniform sampling need more samples to determine which dimension is relevant. In contrast, both \AlgShort quickly focuses on the relevant dimension. Therefore, the number of iterations it needs does not increase when adding irrelevant dimensions to the problem, similar to the oracle solution.

Methods that use adaptive confidence intervals with $\beta_t$ suggested by \Cref{thm:ols_bound_adaptive} turn out to be less sample efficient than their round-based counterparts using static confidence intervals, including \AlgGreedyShort performing worse than \AlgShort. The reason for this is that the confidence interval in \Cref{thm:ols_bound_adaptive} is quite loose. We can heuristically choose smaller confidence intervals and consider $\beta_t$ as a tunable hyperparameter. We find that we can achieve orders of magnitude better sample complexity without much tuning and still always find the correct solution. Even though this approach loses the theoretical guarantees, it could be very valuable in practical applications.

\begin{figure}\centering
   \includegraphics[width=0.7\linewidth]{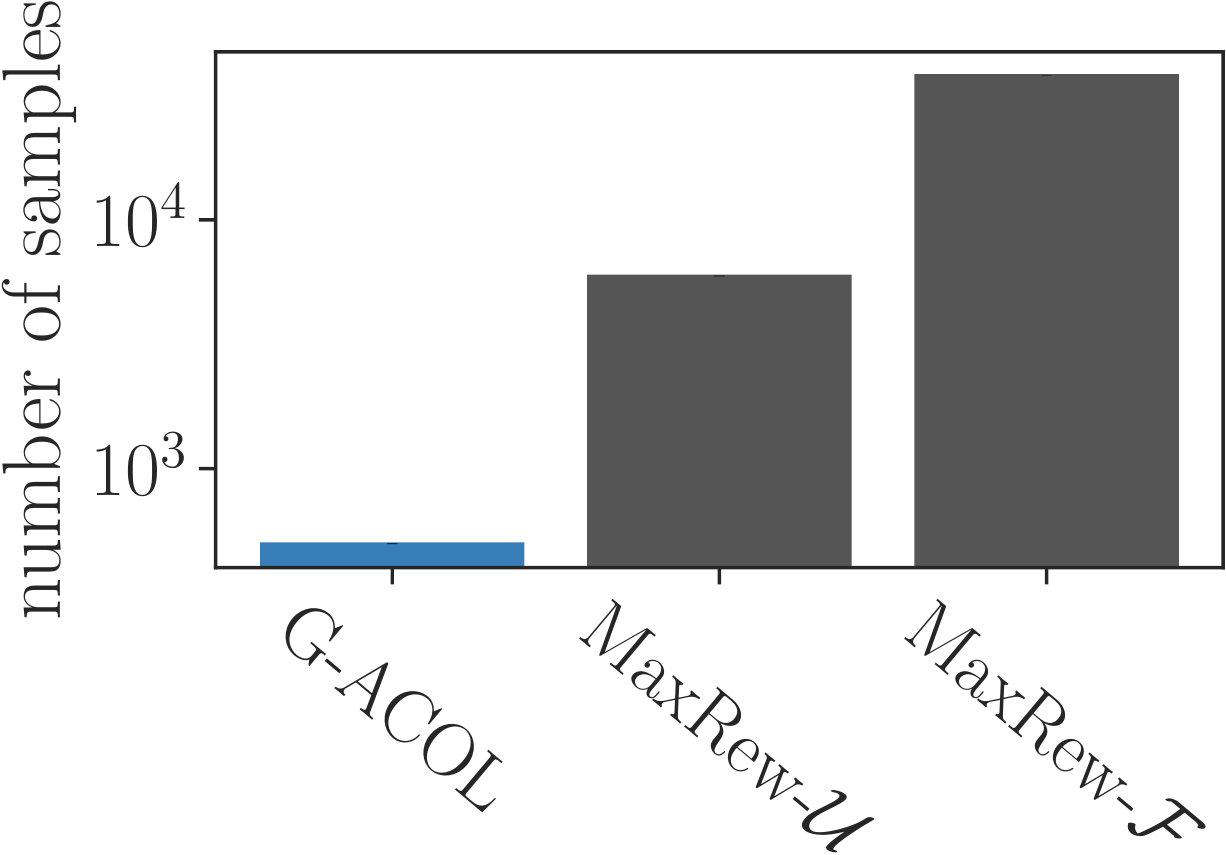}
   \caption{We compare \AlgGreedyShort to \MaxRewF and \MaxRewU, that adapt regret minimization approaches to the \CBAIShort setting. We focus on a simple 1-dimensional problem, where we ensure the set of feasible arms is connected. We find that \MaxRewF is particularly sample inefficient because it only selects arms that are certainly feasible. \MaxRewU is also less sample efficient than \AlgGreedyShort because it selects arms with high reward over other arms that would be more informative during exploration.}
   \label{fig:maxrew}
\end{figure}

\subsection{Comparing \AlgShort to Regret Minimization}
\label{sec:regret_compare_experiment}

To highlight the difference of our \CBAILong setting to regret minimization with constraints, we perform an experiment to compare \AlgGreedyShort to the approaches by \citet{amani2019linear} and \citet{moradipari2021safe}. The algorithm by \citet{amani2019linear} performs UCB and the algorithm by \citet{moradipari2021safe} performs Thompson sampling, both within the set of certainly feasible arms.

We can translate both approaches to our setting with known rewards by greedily selecting arms from $\FeasibleSet_\Round$ w.r.t.\ their reward. Because we do not start with a known safe arm, we add an additional phase in which we select arms randomly until $\FeasibleSet_\Round$ is not empty. Let us call this approach \emph{\MaxRewF}. As a hybrid of this approach and \AlgShort, we can design an algorithm that greedily select arms from $\UncertainSet_\Round$ w.r.t. their reward. Let us call this algorithm \emph{\MaxRewU}.

Unfortunately, \MaxRewF gets stuck in our synthetic instances because we do not make any assumptions on the safe set such as convexity and compactness.  To evaluate these algorithms, we, therefore, consider a third synthetic instance in which the safe set is connected. We consider $10$ arms in $d=1$ that are equally spaced between $0$ and $1$. The reward and constraint vectors are $\RewParam = \ConstParam = 1$, and the threshold is $\tau=0.25$. Here the safe set is connected, but we can learn the constraint boundary more efficiently if we are allowed to violate the constraint during exploration.

We compare \AlgGreedyShort to \MaxRewF and \MaxRewU in \Cref{fig:maxrew}. We find that \AlgGreedyShort explores much more efficiently than both of the other approaches. \MaxRewF is particularly sample inefficient, because it ensures feasibility during exploration, which is not necessary in our case. In \Cref{app:more_results}, we provide results for \MaxRewU in all of our environments. We cannot provide these results for \MaxRewF because it gets stuck in all other environments.

\begin{figure}\centering
   \includegraphics[height=10em]{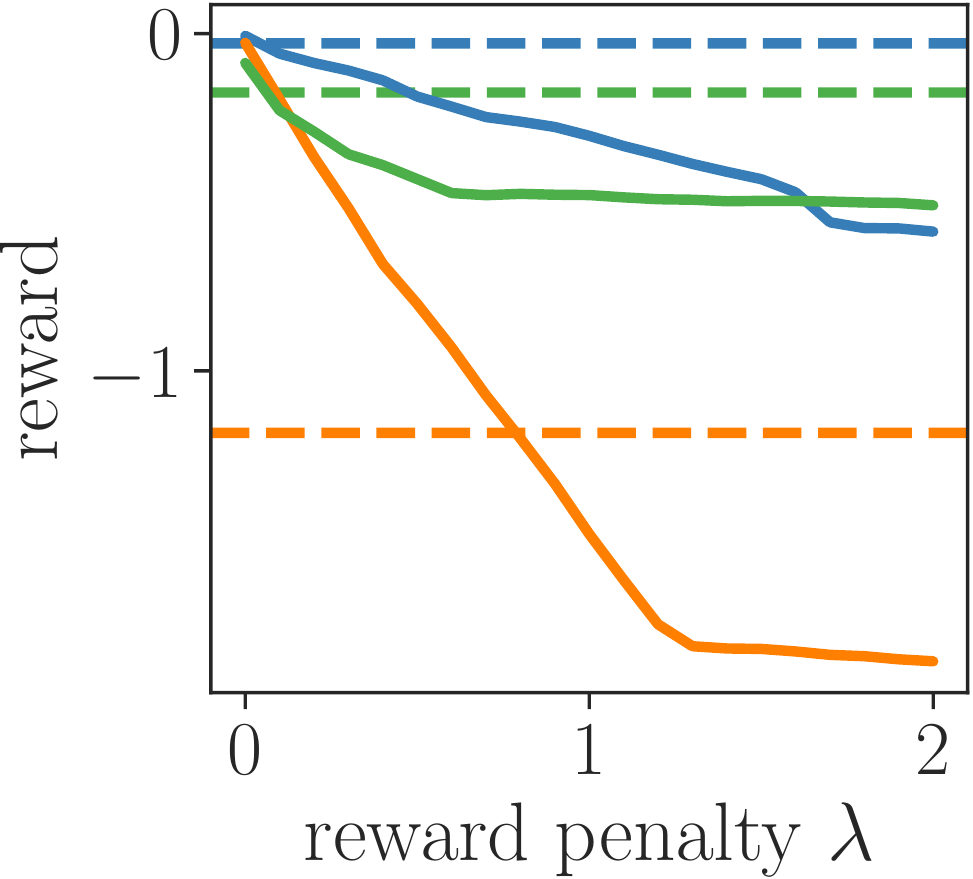}\hfill \includegraphics[height=10em]{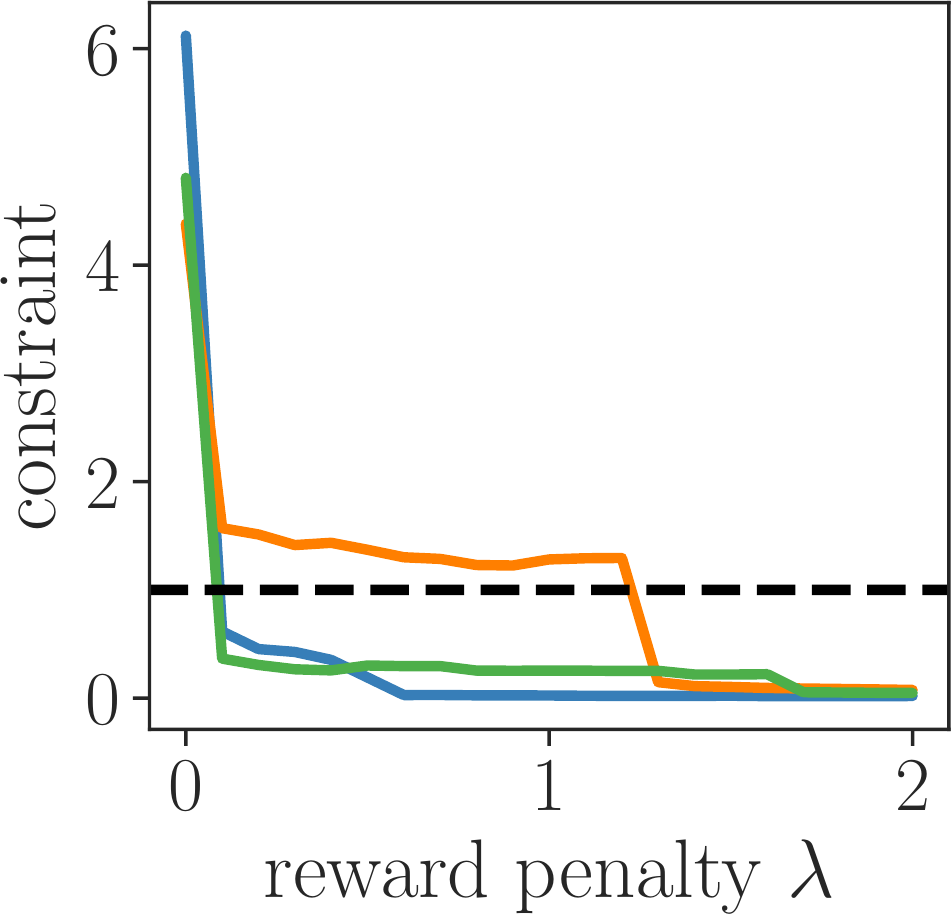}
   \caption{We quantify our finding that learning constraints is more robust to changes in the environment than learning a penalized reward function. We consider the three scenarios from \Cref{fig:driver_example}: the base scenario (\legendDriverTargetVelocity), a scenario with a different goal (\legendDriverTargetLocation), and a scenario with a change in the environment (\legendDriverBlocked). We find a policy that optimizes the reward function $\RewParam^T \Act - \lambda \ConstParam^T \Act$ and plot the reward and the constraint of the solution for different values of $\lambda$. In particular, we need to choose a different value of $\lambda$ for each environment to find the best solution with a constraint value below $1$. The dashed horizontal lines in the reward plot show the reward a constrained solution obtains on the corresponding instance, which does not require any tuning. For each scenario, the smallest $\lambda$ we find to yield a feasible solution still gives a worse solution in terms of reward than the constrained solution.}
   \label{fig:driver_robustness}
\end{figure}

\begin{figure*}\centering
   \hspace{0.2cm}
   \begin{minipage}{0.23\linewidth}
   \includegraphics[width=\linewidth]{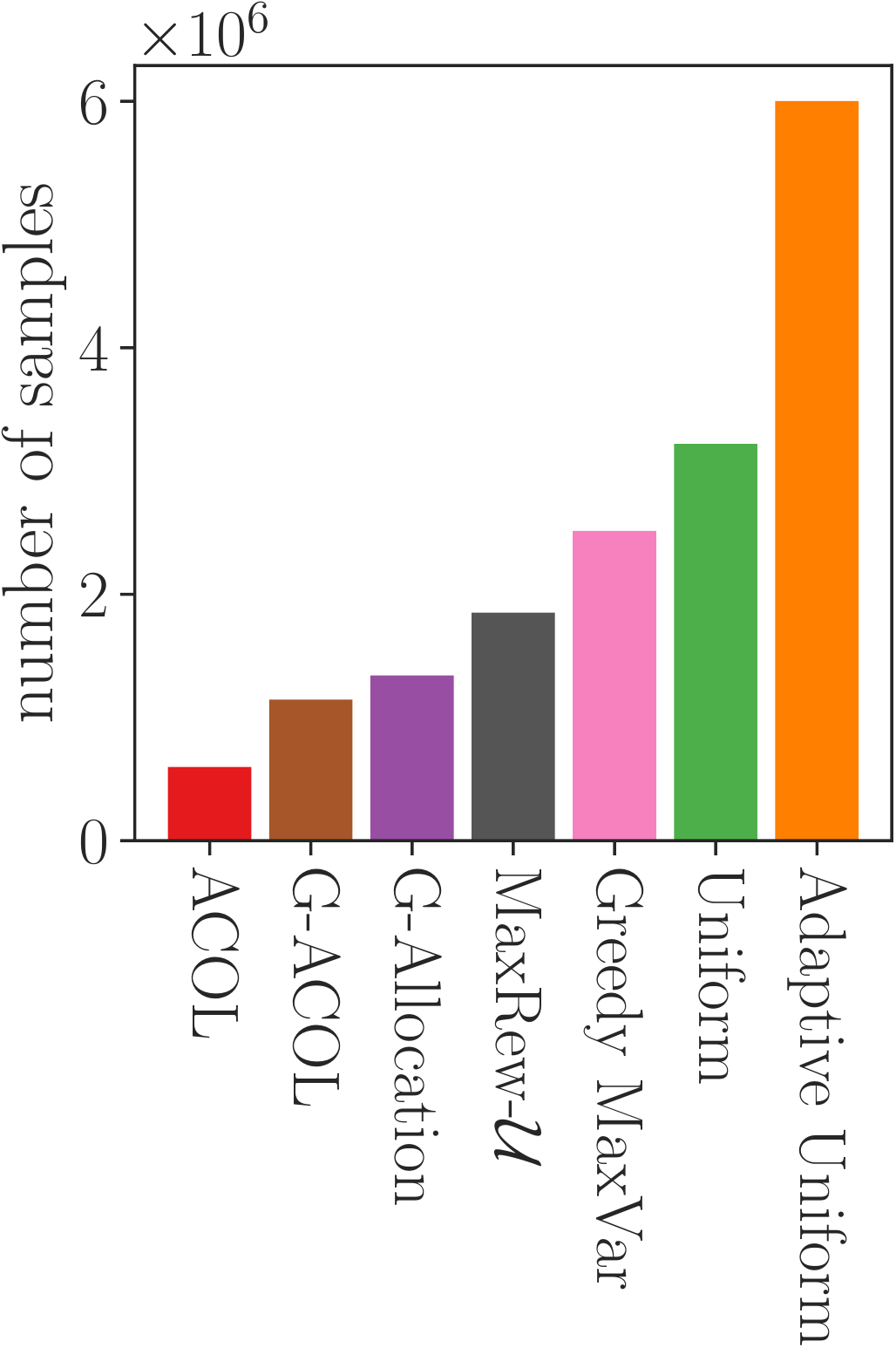}
   \end{minipage}
   \begin{minipage}{0.7\linewidth}
   \centering
   \includegraphics[width=0.4\linewidth]{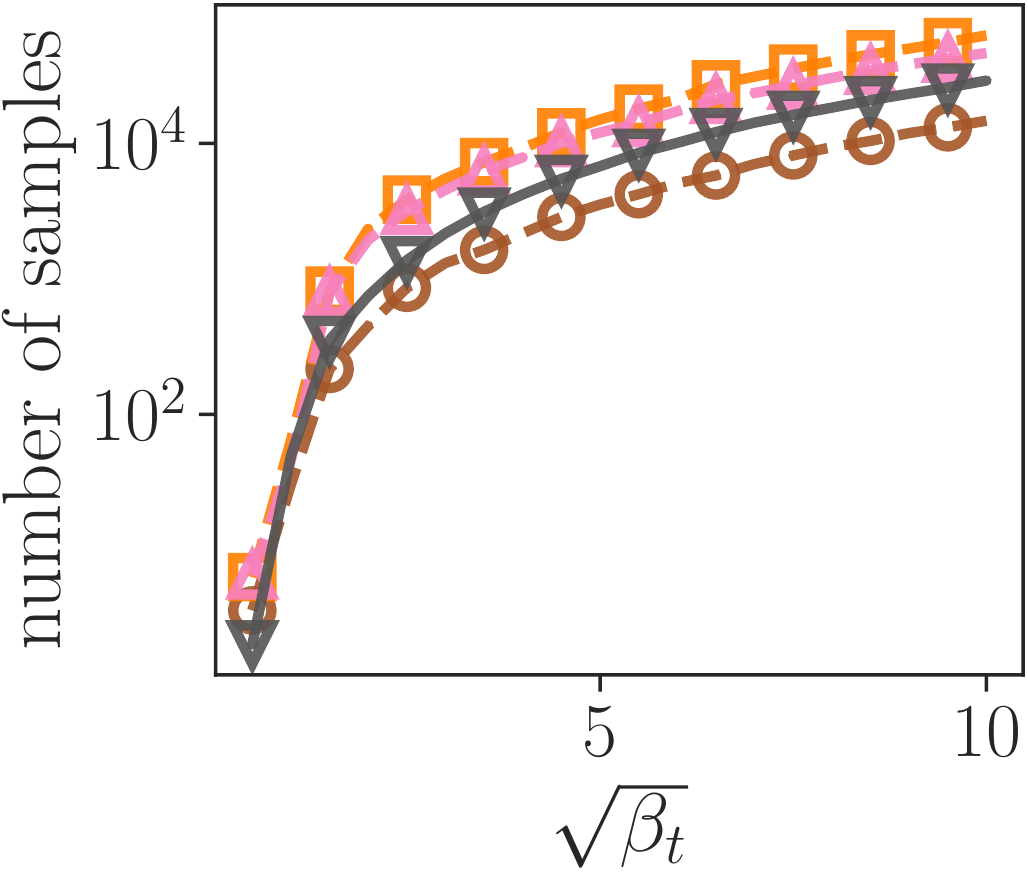}\hspace{0.6cm}
   \includegraphics[width=0.4\linewidth]{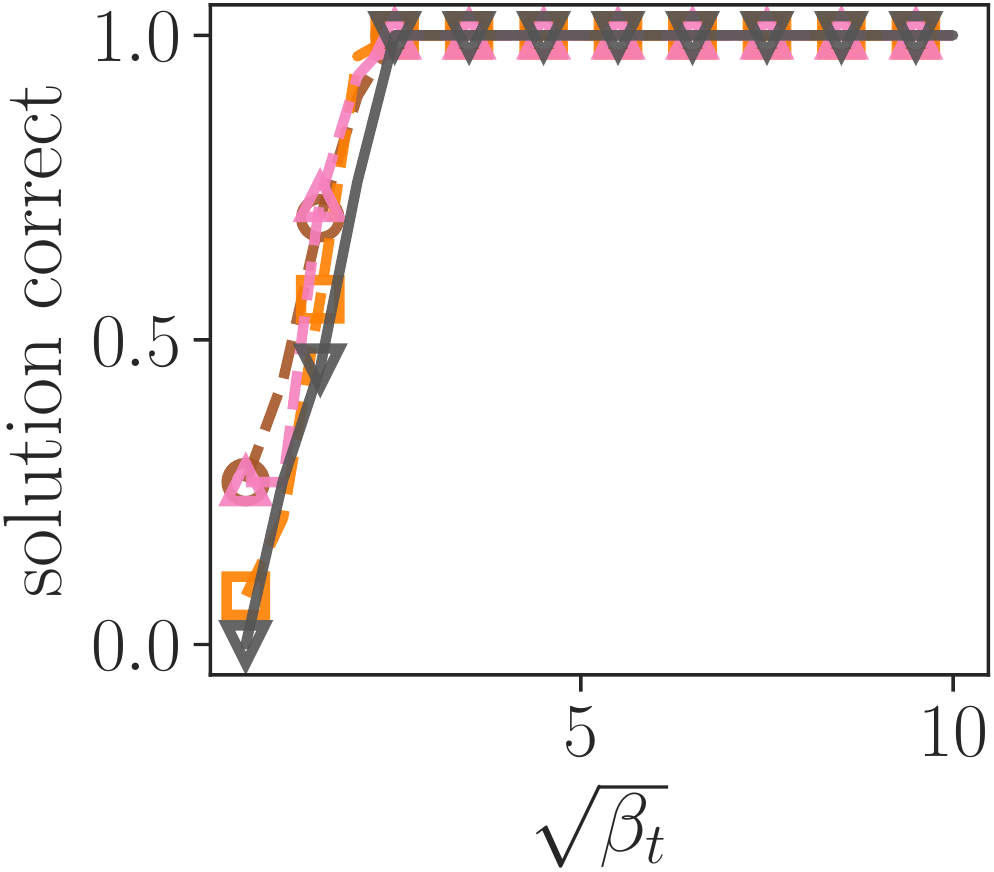} \\\vspace{1em}
   {\small\centering
    \hspace{0.8cm}
    \begin{tabular}{llll}
        \legendUniformAllTuned & Adaptive Uniform
        &\legendMaxvarAllTuned & Greedy MaxVar \\
        \legendMaxRewUTuned & \MaxRewU
        &\legendAdaptiveTuned & \AlgGreedyShort
    \end{tabular}
    }
    \end{minipage}
   \caption{The left chart shows the number of iterations that all algorithms with theoretical guarantees need to find the correct solution in the driving scenario. \AlgShort is the fastest, but it still needs $\sim 10^5$ samples. Instead, we can use heuristic confidence intervals where we consider $\beta_t$ as a hyperparameter instead of choosing the values suggested by theory. The two right plots shows the number of iterations and the percentage of the times the methods return a correct solution as a function of $\beta_t$. None of these algorithms is guaranteed to return the correct solution. But, empirically, we find that for $\sqrt{\beta_t}$ beyond the vertical line, the algorithms always return the correct solution. This again shows that tuning $\beta_t$ can drastically improve the sample efficiency while still returning the correct solution empirically.}
   \label{fig:driver_results}
\end{figure*}

\subsection{Preference Learning Experiments}\label{sec:driving_experiments}

We now consider the application that initially motivated us to define the \CBAIShort problem. As discussed in \Cref{sec:introduction}, we are interested in situations where the reward parameter $\RewParam$ describes an easy-to-specify goal or metric, and the constraint parameter $\ConstParam$ describes expensive-to-evaluate human preferences. 

As an example of this, we consider a driving simulator, which \citet{dorsa2017active} originally introduced to study learning reward functions to represent human preferences about driving behavior. Instead, we change the setting to have the reward $\RewParam$ represent an easy-to-specify goal such as ``drive at velocity $v$'', and the constraint $\ConstParam$ represent other driving rules such as ``usually drive in a lane'' or ``don't get too close to other cars'', as shown in \Cref{fig:driver_example}. \Cref{app:experiment_details} provides more details on the environment.

\looseness -1 The decision-maker has to select a controller to drive the car from a set of precomputed controllers $\ActSet$, i.e., the set of ``arms''. The optimal controller $\Act^*$ maximizes $\RewParam^T \Act^*$ and satisfies $\ConstParam^T \Act^* \leq \tau$. The decision-maker can try out individual controllers to get feedback on whether they are feasible. In contrast to our previous experiments, the feedback is binary. However, we can still model it via a sub-Gaussian noise model by ensuring the constraint values are in $[0, 1]$ and interpreting them as probabilities. Therefore, this is a \CBAIShort problem, and we can apply the same algorithms we applied to our synthetic problems.

\paragraphsmall{Robustness of learning constraints.}
First, we want to quantify the observation of \Cref{fig:driver_example} that constraints can be a particularly robust representation of human preferences. Specifically, using constraints to represent human preferences can increase robustness to changes in the environment and allow to transfer the constraints to different reward functions. Constraints are more robust than modeling the same preferences as a penalty on the reward function. \Cref{fig:driver_robustness} quantifies this by directly comparing the two options in terms of the reward and constraint values they achieve. In particular, we find that the magnitude of the reward penalty often has to be updated if the environment changes, whereas the constraint formulation is robust to such changes.

\paragraphsmall{Results of learning constraints.}
We consider the driving scenario as a \CBAIShort problem and study learning the constraint function. Here, we only report results for the base scenario in \Cref{fig:driver_example}. \Cref{app:more_results} contains similar results for the other two scenarios which are qualitatively similar. In \Cref{fig:driver_results}, we compare the performance of \AlgShort and other algorithms with theoretical correctness guarantees to versions of these algorithms with heuristic confidence intervals. In both cases \AlgShort or \AlgGreedyShort is the most sample efficient algorithm. By choosing the heuristic confidence intervals, we can reduce the number of samples necessary by two orders of magnitude from $\sim 10^5$ to $\sim 10^3$, at the cost of theoretical guarantees. In all cases, using \AlgShort is preferable over alternatives because it finds the correct solution with fewer queries about the constraint function.

\section{Conclusion}

It is natural to formalize sequential decision-making problems in many practical situations as optimizing a known reward function subject to unknown, expensive-to-evaluate constraints. We studied \CBAILong (\CBAIShort), a linear bandit setting to learn about constraints efficiently, and proposed \AlgLong (\AlgShort) to efficiently solve this problem.

\paragraphsmall{Limitations and future work.} Our theoretical analysis is limited to a single constraint function, which might not be appropriate for applications where the constraints are non-additive. It should be possible to extend the same theoretical ideas to multiple linear constraints that all have to be satisfied, which would allow to apply \AlgShort to such situations. From the empirical perspective, we found that modelling human preferences as constraints rather than rewards can be more robust. Future work should study using constraints to model human preferences in more practical applications.

\paragraphsmall{Broader impact.} Sample efficient methods to learn about human preferences could help to avoid misspecified objectives in ML \citep{amodei2016concrete}. By focusing on learning constraints, it might be possible to make preference learning more robust and interpretable. Of course, such algorithms could be misused, but we are optimistic that robust methods to learn from humans will lead to safer ML methods overall.

\section*{Acknowledgements}

This project has received funding from the Microsoft Swiss Joint Research Center (Swiss JRC), and from the European Research Council (ERC) under the European Union's Horizon 2020 research and innovation programme grant agreement No 815943. 
We thank Ilija Bogunovic and Alexandru \textcommabelow{T}ifrea for valuable feedback on early drafts of this paper.

\newcommand{\ICML}{Proceedings of International Conference on Machine Learning (ICML)}
\newcommand{\RSS}{Proceedings of Robotics: Science and Systems (RSS)}
\newcommand{\NeurIPS}{Advances in Neural Information Processing Systems}
\newcommand{\IJCAI}{Proceedings of International Joint Conferences on Artificial Intelligence (IJCAI)}
\newcommand{\ICLR}{International Conference on Learning Representations (ICLR)}
\newcommand{\CoRL}{Conference on Robot Learning (CoRL)}
\newcommand{\UAI}{Uncertainty in Artificial Intelligence (UAI)}
\newcommand{\AAAI}{AAAI Conference on Artificial Intelligence}
\newcommand{\COLT}{Conference on Learning Theory (COLT)}
\newcommand{\AISTATS}{International Conference on Artificial Intelligence and Statistics (AISTATS)}

\bibliography{references}
\bibliographystyle{icml2022}

\clearpage
\appendix
\onecolumn
\renewcommand{\thetable}{\Alph{section}.\arabic{table}}

\section{Proofs}
\label{app:proofs}

This section provides the full proofs of our key results of the paper: the sample complexity lower bound for \CBAIShort problems (\Cref{app:proof_lower_bound}) and the sample complexity of \AlgShort (\Cref{app:proof_acol}).

\subsection{Lower Bounds}
\label{app:proof_lower_bound}

\LowerBound*

\begin{proof}
Our proof has a similar structure to the proof of Theorem 3.1 by \citet{soare2015sequential}. Let us denote the optimal arm of problem $\Problem$ with $\BestAct$ and the optimal arm of $\Problem'$ with $\BestActP$. Let $\cA$ be a $\delta$-PAC algorithm to solve constrained linear bandit problems, and let $A$ be the event that $\cA$ recommends $\BestAct$ as the optimal arm. If we denote by $P_\Problem(A)$ the probability of $A$ happening for instance $\Problem$, and by $P_{\Problem'}(A)$ the probability for instance $\Problem'$, we have
$P_\Problem(A) \geq 1 - \delta$ and $P_{\Problem'}(A) \leq \delta$.

Let $\tilde{\varepsilon} = \ConstParam' - \ConstParam$, and let $\StoppingTime$ be the stopping time of $\cA$. Let $ (\Act_1, \dots, \Act_\StoppingTime)$ be the sequence of arms $\cA$ pulls and $(z_1, \dots, z_t)$ the corresponding observed noisy constraint values $z_i = \Act_i^T \ConstParam + \noise_{\Act_i}$ with $\noise_{\Act_i} \sim \NormalDist(0, 1)$ being independent Gaussian noise.

Now, consider the log-likelihood ratio of these observations under algorithm $\cA$:
\begin{align*}
    L_\StoppingTime &= \log \left( \prod_{s=1}^\StoppingTime \frac{P_\Problem(z_s | \Act_s)}{P_{\Problem'}(z_s | \Act_s)} \right)
    = \sum_{s=1}^\StoppingTime \log \left( \frac{P_\Problem(z_s | \Act_s)}{P_{\Problem'}(z_s | \Act_s)} \right)
    = \sum_{s=1}^\StoppingTime \log \left( \frac{P_\Problem(\eta_s)}{P_{\Problem'}(\eta_s')} \right)
    = \sum_{s=1}^\StoppingTime \log \left( \frac{\exp(-\eta_s^2/2)}{\exp(-\eta_s'^2/2)} \right) \\
    &= \sum_{s=1}^\StoppingTime \frac12 ((z_s - \Act_s^T \ConstParam')^2 - (z_s - \Act_s^T \ConstParam)^2)
    = \sum_{s=1}^\StoppingTime \frac12 (z_s^2 - 2 z_s \Act_s^T \ConstParam' + (\Act_s^T \ConstParam')^2 - z_s^2 + 2 z_s \Act_s^T \ConstParam - (\Act_s^T \ConstParam)^2) \\
    &= \sum_{s=1}^\StoppingTime \frac12 (2 z_s \Act_s^T (\ConstParam - \ConstParam') + (\Act_s^T \ConstParam' - \Act_s^T \ConstParam) (\Act_s^T \ConstParam' + \Act_s^T \ConstParam))
    = \sum_{s=1}^\StoppingTime \frac12 (-2 z_s \Act_s^T \tilde{\varepsilon} + \Act_s^T \tilde{\varepsilon} (\Act_s^T \ConstParam + \Act_s^T \tilde{\varepsilon} + \Act_s^T \ConstParam)) \\
    &= \sum_{s=1}^\StoppingTime (\Act_s^T \tilde{\varepsilon}) \frac{-2 z_s + 2 \Act_s^T \ConstParam + x_s^T \tilde{\varepsilon}}{2}
    = \sum_{s=1}^\StoppingTime (\Act_s^T \tilde{\varepsilon}) \left( \frac{\Act_s^T \tilde{\varepsilon}}{2} - \eta_s \right)
\end{align*}

Taking the expectation of this log-likelihood ratio gives:
\begin{align*}
    \Expectation_{\Problem} [L_\StoppingTime]
    &= \Expectation_{\Problem} \left[ \sum_{s=1}^\StoppingTime (\Act_s^T \tilde{\varepsilon}) \left( \frac{\Act_s^T \tilde{\varepsilon}}{2} - \eta_s \right) \right]
    = \frac12 \Expectation_{\Problem} \left[ \sum_{s=1}^\StoppingTime (\Act_s^T \tilde{\varepsilon})^2 \right] - \underbrace{\Expectation_{\Problem} \left[ \eta_s  \right]}_{= 0} \\
    &= \frac12 \Expectation_{\Problem} \left[ \sum_{s=1}^\StoppingTime \tilde{\varepsilon}^T \Act_s \Act_s^T \tilde{\varepsilon} \right]
    = \frac12 \Expectation_{\Problem} \left[ \sum_{\Act\in\ActSet}  \Expectation_{\Problem} [\StoppingTime] \lambda(x) \tilde{\varepsilon}^T \Act \Act^T \tilde{\varepsilon} \right] \\
    &= \frac12 \Expectation_{\Problem} [\StoppingTime] \Expectation_{\Problem} \left[ \sum_{\Act\in\ActSet}  \lambda(x) \tilde{\varepsilon}^T \Act \Act^T \tilde{\varepsilon} \right]
    = \frac12 \Expectation_{\Problem} [\StoppingTime] \tilde{\varepsilon}^T \DesMat \tilde{\varepsilon}
\end{align*}

Next, we can apply Lemma 19 from \citet{kaufmann2016complexity}:
\[
\Expectation_{\Problem} [L_\StoppingTime]
= \frac12 \Expectation_{\Problem} [\StoppingTime] \tilde{\varepsilon}^T \DesMat \tilde{\varepsilon}
\geq \KL(P_\Problem(A), P_{\Problem'}(A)) \geq \log \frac{1}{2.4 \delta}
\]
\begin{align}\label{eq:bound_with_eps}
\Expectation_{\Problem} [\StoppingTime]
\geq 2 \log\left(\frac{1}{2.4\delta}\right) \frac{1}{\tilde{\varepsilon}^T \DesMat \tilde{\varepsilon}}
\end{align}

To obtain a lower bound, we now aim to find the smallest $\tilde{\varepsilon}$ such that $\Problem$ and $\Problem'$ have different constrained optimal arms.

Let $\BetterActSet(\Act) = \{ \Act' \in \ActSet | \RewParam^T \Act' \geq \RewParam^T \Act \}$ be the set of arms with higher reward than $\Act$. There are two ways we can modify $\Problem$ to change its optimal arm. We can change $\ConstParam$ to $\ConstParam'$ such that either, \textbf{Case (i)}, the previous optimum $\Act_\Problem^*$ becomes infeasible in $\Problem'$, or, \textbf{Case (ii)}, a solution $\Act_\Problem^* \in \BetterActSet(\Act_\Problem^*)$ that was infeasible in $\Problem$ is now feasible in $\Problem'$. We will consider both cases separately, and aim to find an $\tilde{\varepsilon}$ for each case that minimizes $\tilde{\varepsilon}^T \DesMat \tilde{\varepsilon}$.

\paragraph{Case (i).} We want to find $\tilde{\varepsilon}$ that minimizes $\frac{1}{2} \varepsilon^T \DesMat \varepsilon$ such that $\ConstParam'^T \Act_\Problem^* > 0$, i.e., the previously optimal arm becomes infeasible. We can write this constraint equivalently as
\begin{align*}
    \ConstParam'^T \Act_\Problem^* > 0
    \equiv \ConstParam^T \Act_\Problem^* - \ConstParam'^T \Act_\Problem^* < \ConstParam^T \Act_\Problem^*
    \equiv \varepsilon^T \Act_\Problem^* < \ConstParam^T \Act_\Problem^*
    \equiv \varepsilon^T \Act_\Problem^* - \ConstParam^T \Act_\Problem^* < 0
\end{align*}
Which results in the following optimization problem:
\begin{equation*}
    \min_\varepsilon \frac{1}{2} \varepsilon^T \DesMat \varepsilon ~~~
    \textrm{s.t.} ~~~ \varepsilon^T \Act_\Problem^* - \ConstParam^T \Act_\Problem^* + \alpha \leq 0,
\end{equation*}
where $\alpha > 0$. The Lagrangian is
$L(\varepsilon, \gamma) = \frac{1}{2} \varepsilon^T \DesMat \varepsilon - \gamma (\varepsilon^T \Act_\Problem^* - \ConstParam^T \Act_\Problem^* + \alpha)$, and requiring $\pd{L}{\varepsilon} = \pd{L}{\gamma} = 0$ yields:
\begin{align*}
    \pd{L}{\varepsilon} &= \DesMat \varepsilon - \gamma \Act_\Problem^*  = 0
    \equiv \DesMat\varepsilon = \gamma \Act_\Problem^*
    \equiv \DesMat^{\frac12} \varepsilon = \gamma \DesMat^{-\frac12} \Act_\Problem^* \\
    \pd{L}{\gamma} &= \varepsilon^T \Act_\Problem^* - \ConstParam^T \Act_\Problem^* + \alpha = 0 \equiv \varepsilon^T \Act_\Problem^* = \ConstParam^T \Act_\Problem^* - \alpha
\end{align*}
From the first equation, it follows that
\begin{align*}
    {\Act_\Problem^*}^T \varepsilon
    &= {\Act_\Problem^*}^T \DesMat^{-\frac12} \DesMat^{\frac12} \varepsilon
    = \gamma {\Act_\Problem^*}^T \DesMat^{-1} \Act_\Problem^*
    = \gamma \| {\Act_\Problem^*} \|_{\DesMat^{-1}}^2 \\
    {\Act_\Problem^*}^T \varepsilon
    &= {\Act_\Problem^*}^T \DesMat^{-\frac12} \DesMat^{\frac12} \varepsilon
    = \frac{1}{\gamma} \varepsilon^T \DesMat \varepsilon
    = \frac{1}{\gamma} \| \varepsilon \|_{\DesMat}^2
\end{align*}
and therefore
\begin{align*}
    {\Act_\Problem^*}^T \varepsilon &= \| {\Act_\Problem^*} \|_{\DesMat^{-1}} \| \varepsilon \|_{\DesMat} = \ConstParam^T \Act_\Problem^* - \alpha \\
    \| \varepsilon \|_{\DesMat} &= \frac{\ConstParam^T \Act_\Problem^* - \alpha}{\| {\Act_\Problem^*} \|_{\DesMat^{-1}}} > \frac{\ConstParam^T \Act_\Problem^*}{\| {\Act_\Problem^*} \|_{\DesMat^{-1}}}
\end{align*}
where the last inequality follows because $\alpha > 0$ and $\DesMat$ is positive definite.

\paragraph{Case (ii).} We want to find $\tilde{\varepsilon}$ that minimizes $\frac{1}{2} \varepsilon^T \DesMat \varepsilon$ such that there exists an $\Act \in \ActSet$ for which $\RewParam^T \Act > \RewParam^T \Act_\Problem^*$ and $\ConstParam'^T \Act \leq 0$, i.e., $\Act$ has higher reward than $\Act_\Problem^*$ and it is feasible in $\Problem'$.
We can write these constraints as
\begin{align*}
    \RewParam^T \Act > \RewParam^T \Act_\Problem^*
    &\equiv \RewParam^T (\Act_\Problem^* - \Act) + \alpha \leq 0  \\
    \ConstParam'^T \Act \leq 0
    &\equiv \varepsilon^T \Act + \ConstParam^T \Act \leq 0
\end{align*}
with $\alpha > 0$. This results in the following optimization problem:
\begin{equation*}
\begin{aligned}
    \min_\varepsilon \quad & \frac{1}{2} \varepsilon^T \DesMat \varepsilon \\
    \textrm{s.t.} \quad \exists \Act:~ & \RewParam^T (\Act_\Problem^* - \Act) + \alpha &\leq 0 \\
    & \varepsilon^T \Act + \ConstParam^T \Act &\leq 0
\end{aligned}
\end{equation*}
The Lagrangian of this problem is
\[
    L(\varepsilon, \gamma, \delta) = \frac{1}{2} \varepsilon^T \DesMat \varepsilon - \gamma(\RewParam^T (\Act_\Problem^* - \Act) + \alpha) - \delta (\varepsilon^T \Act + \ConstParam^T \Act)
\]
Requiring $\pd{L}{\varepsilon} = \pd{L}{\delta} = 0$ results in
\begin{align*}
    \pd{L}{\varepsilon} &= \DesMat \varepsilon - \delta \Act = 0
    \equiv \DesMat \varepsilon = \delta \Act
    \equiv \DesMat^{\frac12} \varepsilon = \delta \DesMat^{-\frac12} \Act \\
    \pd{L}{\delta} &= \varepsilon^T \Act + \ConstParam^T \Act = 0
    \equiv \varepsilon^T \Act = - \ConstParam^T \Act
\end{align*}
It follows that
\begin{align*}
    \Act^T \varepsilon
    &= \Act^T \DesMat^{-\frac12} \DesMat^{\frac12} \varepsilon
    = \delta \Act^T \DesMat^{-1} \Act
    = \delta \| \Act \|_{\DesMat^{-1}}^2 \\
    \Act^T \varepsilon
    &= \Act^T \DesMat^{-\frac12} \DesMat^{\frac12} \varepsilon
    = \frac{1}{\delta} \varepsilon^T \DesMat \varepsilon
    = \frac{1}{\delta} \| \varepsilon \|_{\DesMat}^2
\end{align*}
and therefore
\begin{align*}
    \Act^T \varepsilon = \| \Act \|_{\DesMat^{-1}} \| \varepsilon \|_{\DesMat} = \ConstParam^T \Act
    ~\Rightarrow~ \| \varepsilon \|_{\DesMat} = \frac{\ConstParam^T \Act}{\| \Act \|_{\DesMat^{-1}}}
\end{align*}
Combining this result with the remaining constraint $\RewParam^T \Act > \RewParam^T \Act_\Problem^*$ which implies $\Act \in \BetterActSet(\Act_\Problem^*)$, we can conclude
\begin{align*}
    \| \varepsilon \|_{\DesMat} \geq \min_{\Act \in \BetterActSet(\Act_\Problem^*)} \frac{\ConstParam^T \Act}{\| \Act \|_{\DesMat^{-1}}}
\end{align*}

\paragraph{Combining cases (i) and (ii).} We can conclude that the $\varepsilon$ that minimizes $\| \varepsilon \|_{\DesMat}^2$ while still ensuring that $\Problem'$ has a different solution than $\Problem$, satisfies:
\begin{align*}
    \| \varepsilon \|_{\DesMat} \geq \min \Bigl[
        \underbrace{\frac{\ConstParam^T \Act_\Problem^*}{\| {\Act_\Problem^*} \|_{\DesMat^{-1}}}}_{\text{Case (i)}},
        \underbrace{\min_{\Act \in \BetterActSet(\Act_\Problem^*)} \frac{\ConstParam^T \Act}{\| \Act \|_{\DesMat^{-1}}}}_{\text{Case (ii)}}
    \Bigr]
\end{align*}
But because $\Act_\Problem^* \in \BetterActSet(\Act_\Problem^*)$, it is simply
\begin{align*}
    \| \varepsilon \|_{\DesMat} \geq \min_{\Act \in \BetterActSet(\Act_\Problem^*)} \frac{\ConstParam^T \Act}{\| \Act \|_{\DesMat^{-1}}}
\end{align*}

Combining this result with \cref{eq:bound_with_eps}, gives the final bound:
\begin{align*}
    \Expectation_{\Problem} [\StoppingTime]
    \geq 2 \log\left(\frac{1}{2.4\delta}\right) \frac{1}{\left( \min_{\Act \in \BetterActSet(\Act_\Problem^*)} \frac{\ConstParam^T \Act}{\| \Act \|_{\DesMat^{-1}}} \right)^2}
    = 2 \log\left(\frac{1}{2.4\delta}\right) \max_{\Act \in \BetterActSet(\Act_\Problem^*)} \frac{\| \Act \|_{\DesMat^{-1}}^2}{(\ConstParam^T \Act)^2}
\end{align*}
\end{proof}

Next, we derive the worst case bound on the quantity making up the \CBAIShort lower bound.

\InstanceIndependentLowerBound*

\begin{proof}
\begin{align*}
    \Complexity(\Problem)
    = \min_\lambda \max_{\Act \in \BetterActSet(\BestAct)} \frac{\| \Act \|_{A_{\lambda^*}^{-1}}^2}{(\TrueConst^T \Act)^2}
    \leq \frac{1}{\MinConst^2} \min_\lambda \max_{\Act \in \BetterActSet(\BestAct)} \| \Act \|_{A_{\lambda^*}^{-1}}^2
    \leq \frac{d}{\MinConst^2}
\end{align*}
where the last inequality uses the well-known result by \citet{kiefer1960equivalence}. Equality holds, for example, if all $\Act \in \ActSet$ are linearly independent and have the same constraint value $\MinConst$.
\end{proof}

\subsection{Adaptive Constraint Learning}
\label{app:proof_acol}

In this section, we analyse the sample complexity of \AlgShort and prove our main result.

\AdaptiveAlgorithmComplexity*

\begin{proof}

Let $\GoodEvent_{\Round} \coloneqq \{ \UncertainSet_\Round \subseteq \SmallGapSet_{\Round} \}$ where $\SmallGapSet_{\Round} \coloneqq \{ \Act \in \ActSet | \UpperConst^\Round(\Act) - \LowerConst^\Round(\Act) \leq 2^{-\Round} \}$. So, $\GoodEvent_\Round$ is the event that all arms in $\UncertainSet_\Round$ have confidence interval smaller than $2^{-\Round}$. We will first show that $P(\GoodEvent_1) \geq 1 - \delta_1$ and $P(\GoodEvent_\Round | \GoodEvent_{\Round-1} ) \geq 1 - \delta_\Round$, which ensures that the set of arms we are uncertain about shrinks exponentially in the rounds $\Round$.

Let $\Act \in \UncertainSet_\Round$. Then, using \Cref{thm:ols_bound_static}, and the $\varepsilon$-approximate rounding strategy, it holds with probability at least $1-\delta_t$ that:
\begin{align*}
\UpperConst^\Round(\Act) - \LowerConst^\Round(\Act)
\leq 2 \sqrt{2\log\left(\frac{|\ActSet|}{\delta_\Round}\right) \frac{1+\varepsilon}{N_\Round}} \| \Act \|_{A_{\Design_\Round^*}^{-1}}
\end{align*}

Using the length of a round $N_\Round = \left\lceil 2^{2\Round+3} \log\left(\frac{|\ActSet|}{\delta_t}\right) (1+\varepsilon) \rho^*_\Round \right\rceil$, and that we select arms to reduce uncertainty in $\UncertainSet_\Round$, we get
\begin{align*}
\UpperConst^\Round(\Act) - \LowerConst^\Round(\Act) &\leq 2^{-\Round} \sqrt{ \left( \min_{\Design} \max_{\tilde{\Act} \in \UncertainSet_\Round} \| \tilde{\Act} \|_{\DesMat^{-1}}^2 \right)^{-1} } \| \Act \|_{A_{\Design_\Round^*}^{-1}} \\
&\leq 2^{-\Round} \sqrt{ \left( \min_{\Design} \max_{\tilde{\Act} \in \UncertainSet_\Round} \| \tilde{\Act} \|_{\DesMat^{-1}}^2 \right)^{-1} } \left( \min_\Design \max_{\tilde{\Act} \in \UncertainSet_\Round} \| \tilde{\Act} \|_{\DesMat^{-1}} \right)
\leq 2^{-\Round}
\end{align*}

Note, that $\Act$ can only be in $\UncertainSet_\Round$ if $\UpperConst^\Round(\Act) > 0$ and $\LowerConst^\Round(\Act) \leq 0$. It follows that $P(\GoodEvent_\Round | \GoodEvent_{\Round-1}) \geq 1-\delta_\Round$.

Now consider round $\bar{\Round} \coloneqq \left\lceil \log_2 \frac{1}{\MinConst} \right\rceil$. We show $P(\UncertainSet_{\bar{\Round}} = \emptyset | \GoodEvent_{\bar{\Round}} ) = 1$. Assume $\GoodEvent_{\bar{\Round}}$, i.e., $\UncertainSet_\Round \subseteq \SmallGapSet_\Round$. Let $\Act \in \UncertainSet_{\bar{\Round}}$, then:
\[
| \TrueConst^T \Act | \leq 2^{-\bar{\Round}} \leq 2^{-\log_2{1/\MinConst}} = \MinConst
\]
which is a contradiction because otherwise $\Act$ would have a smaller constraint value than $\MinConst$. Consequently, the set of uncertain arms $\UncertainSet_{\bar{\Round}}$ is empty and the algorithm returns the correct solution given $\GoodEvent_{\bar{\Round}}$.
\Cref{lemma:delta_bound} shows that the unconditional probability of the algorithm returning the correct solution after round $\bar{\Round}$ is at least $1-\delta$.

Finally, we can compute the total number of samples the algorithm needs to return the correct solution:
\begin{align*}
N &= \sum_{\Round=1}^{\bar{\Round}} \lceil 2^{2\Round+3} \log\left(\frac{|\ActSet|}{\delta_\Round}\right) (1+\varepsilon) \rho^*_\Round \rceil
\leq \sum_{\Round=1}^{\bar{\Round}} 2^{2\Round+3} \log\left(\frac{|\ActSet|}{\delta_\Round}\right) (1+\varepsilon) \rho^*_\Round + \bar{\Round} \\
&\leq 8 \log\left(\frac{|\ActSet|\bar{\Round}^2}{\delta^2}\right) (1+\varepsilon) \sum_{\Round=1}^{\bar{\Round}} (2^\Round)^2 \rho^*_\Round + \bar{\Round}
= 8 \log\left(\frac{|\ActSet|\bar{\Round}^2}{\delta^2}\right) (1+\varepsilon) \sum_{\Round=1}^{\bar{\Round}} (2^\Round)^2 \min_{\lambda} \max_{\tilde{\Act} \in \UncertainSet_\Round} \| \tilde{\Act} \|_{\DesMat^{-1}}^2 + \bar{\Round} \\
&= 8 \log\left(\frac{|\ActSet|\bar{\Round}^2}{\delta^2}\right) (1+\varepsilon) \sum_{\Round=1}^{\bar{\Round}} \min_{\lambda} \max_{\tilde{\Act} \in \UncertainSet_\Round} \frac{\| \tilde{\Act} \|_{\DesMat^{-1}}^2}{(2^{-\Round})^2} + \bar{\Round}
\overset{(a)}{\leq} 8 \log\left(\frac{|\ActSet|\bar{\Round}^2}{\delta^2}\right) (1+\varepsilon) \sum_{\Round=1}^{\bar{\Round}} \min_{\lambda} \max_{\tilde{\Act} \in \UncertainSet_\Round} \frac{\| \tilde{\Act} \|_{\DesMat^{-1}}^2}{(\TrueConst^T \tilde{\Act})^2} + \bar{\Round} \\
&\overset{(b)}{\leq} 8 \log\left(\frac{|\ActSet|\bar{\Round}^2}{\delta^2}\right) (1+\varepsilon) \sum_{\Round=1}^{\bar{\Round}} \min_{\lambda} \max_{\tilde{\Act} \in \ActSet} \frac{\| \tilde{\Act} \|_{\DesMat^{-1}}^2}{(\TrueConst^T \tilde{\Act})^2} + \bar{\Round}
\leq 8 \log\left(\frac{|\ActSet|\bar{\Round}^2}{\delta^2}\right) (1+\varepsilon) \bar{\Round} \WeakComplexity(\Problem) + \bar{\Round}
\end{align*}
where (a) follows because we showed that $|\TrueConst^T \Act| \leq 2^{-\Round}$ w.h.p. for $\Act\in\UncertainSet_\Round$, and (b) follows simply because $\UncertainSet_\Round \subseteq \ActSet$. In the last step, we defined $\WeakComplexity(\Problem) = \min_{\lambda} \max_{\tilde{\Act} \in \ActSet} \frac{\| \tilde{\Act} \|_{\DesMat^{-1}}^2}{(\TrueConst^T \tilde{\Act})^2}$

Moreover,
\begin{align*}
   \WeakComplexity(\Problem) = \min_{\Design} \max_{\Act \in \ActSet} \frac{\| \Act \|_{\DesMat^{-1}}^2}{(\TrueConst^T \Act)^2}
   \leq \frac{1}{\MinConst^2} \min_{\Design} \max_{\Act \in \ActSet} \| \Act \|_{\DesMat^{-1}}^2
   \leq \frac{1}{\MinConst^2} \min_{\Design} \max_{\Act \in \R^d} \| \Act \|_{\DesMat^{-1}}^2
   \leq \frac{d}{\MinConst^2}
\end{align*}
using the result by \citet{kiefer1960equivalence}.

\end{proof}

\begin{lemma}\label{lemma:delta_bound}
Let $\GoodEvent_1, \dots, \GoodEvent_T$ be a Markovian sequence of events such that $P(\GoodEvent_1) \geq 1-\delta_1$ and $P(\GoodEvent_\Round | \GoodEvent_{\Round-1}) \geq 1-\delta_\Round$ for all $\Round = 2, \dots, T$, where $\delta_\Round = \delta^2 / \Round^2$ and $\delta \in (0, 1)$. $\GoodEvent_\Round$ is independent of other events conditioned on $\GoodEvent_{\Round-1}$. Then $P(\GoodEvent_T) \geq 1-\delta$.
\end{lemma}

\begin{proof}
\begin{align*}
P(\GoodEvent_T)
= \left(\prod_{\Round=2}^{T} P(\GoodEvent_\Round | \GoodEvent_{\Round-1})\right) P(\GoodEvent_1)
\geq \left( \prod_{\Round=2}^{\bar{\Round}} (1-\delta_\Round) \right) (1-\delta_1)
\geq \prod_{\Round=1}^{\infty} \left( 1-\frac{\delta^2}{\Round^2} \right)
= \frac{\sin(\pi \delta)}{\pi \delta}
\geq 1-\delta
\end{align*}
where the last inequality holds for $0 \leq \delta \leq 1$.
\end{proof}

\section{Alternative Algorithms}

\begin{algorithm}[t]
\caption{Round based algorithm with a generic allocation $\Design^*$ with hyperparamater $v \in (1, 2)$. For $\Design^* \in \argmin_{\Design} \max_{\Act\in\BetterActSet(\BestAct)} \| \Act \|_{\DesMat^{-1}} / | \TrueConst^T \Act |$ this algorithm becomes the oracle solution. For $\Design^* \in \argmin_{\Design} \max_{\Act\in\ActSet} \| \Act \|_{\DesMat^{-1}}$ it becomes \emph{G-Allocation}.}
\label{alg:generic_round_based}
\begin{algorithmic}[1]
    \STATE \textbf{Input:} static design $\Design^*$, significance $\delta$
    \STATE $\UncertainSet_1 \gets \ActSet$ \hfill (uncertain arms)
    \STATE $\FeasibleSet_1 \gets \emptyset$ \hfill (feasible arms)
    \STATE $\Round \gets 1$ \hfill (round)
    \WHILE{$\UncertainSet_\Round \neq \emptyset$}
        \STATE $\delta_t \gets \delta^2 / \Round^2$
        \STATE $N_\Round \gets \lceil v^\Round \log(|\ActSet|/\delta_\Round) \rceil$
        \STATE $\vx_{N_\Round} \gets \mathtt{Round}(\Design^*, N_\Round)$
        \STATE Pull arms $\Act_1, \dots, \Act_{N_\Round}$ and observe constraint values
        \STATE $\Round \gets \Round + 1$
        \STATE Update $\hat{\ConstParam}_\Round$ and $A$ based on new data
        \STATE $\LowerConst^\Round(\Act) \gets \hat{\ConstParam}_\Round^T \Act - \sqrt{\beta_t} \| \Act \|_{A^{-1}}$ for all arms $\Act\in\ActSet$
        \STATE $\UpperConst^\Round(\Act) \gets \hat{\ConstParam}_\Round^T \Act + \sqrt{\beta_t} \| \Act \|_{A^{-1}}$ for all arms $\Act\in\ActSet$
        \STATE $\FeasibleSet_\Round \gets \FeasibleSet_{\Round-1} \cup \{ \Act | \UpperConst^\Round(\Act) \leq 0 \}$
        \STATE $\bar{r} \gets \max_{\Act\in\FeasibleSet_\Round} \RewParam^T \Act$
        \STATE $\UncertainSet_\Round \gets \UncertainSet_{\Round-1} \setminus \{ \Act | \LowerConst^\Round(\Act) > 0 \} \setminus \{ \Act | \UpperConst^\Round(\Act) \leq 0 \} \setminus \{ \Act | \RewParam^T \Act < \bar{r} \}$
    \ENDWHILE
    \STATE \textbf{return} $\Act^* \in \argmax_{\Act\in\FeasibleSet_\Round} \RewParam^T \Act$
\end{algorithmic}
\end{algorithm}

Given any static design $\Design^*$, we can consider different round-based algorithms using the static confidence intervals from \Cref{thm:ols_bound_static}. \Cref{alg:generic_round_based} shows the general algorithm. It uses the same stopping condition as \AlgShort but uses a more straightforward round length of $v^\Round \log(|\ActSet|/\delta_\Round)$ with $v$ a hyperparameter, and a fixed static allocation. In this section, we analyze two versions of this generic algorithm that are of particular interest: the \emph{oracle} solution (\Cref{app:oracle}) and \emph{G-Allocation} (\Cref{app:g-allocation}).

\subsection{Oracle Solution}
\label{app:oracle}

The oracle solution allocates samples according to $\Design^* \in \argmin_{\Design} \max_{\Act\in\BetterActSet(\BestAct)} \| \Act \|_{\DesMat^{-1}} / | \TrueConst^T \Act |$ in \Cref{alg:generic_round_based}. Note that this design exactly matches the term in our instance dependent lower-bound in \Cref{thm:lower_bound}. Therefore, this is the ideal allocation to achieve good sample complexity. However, this oracle solution requires knowledge of $\TrueConst$, which we do not know in practice.

As expected, this algorithm matches the sample complexity lower bound, i.e., it is instance-optimal apart from logarithmic factors. The following theorem formalizes this.

\begin{theorem}[Oracle sample complexity]\label{thm:oracle_sample_complexity}
The oracle algorithm finds the optimal solution to a constrained linear best-arm identification problem $\Problem = (\ActSet, \RewParam, \ConstParam)$ within $N \propto \Complexity(\Problem)$ with probability at least $1-\delta$.
\end{theorem}

\begin{proof}
Assuming a $(1+\varepsilon)$-approximate rounding procedure, in round $\Round$ we have:
$\| \Act \|_{A_{\vx_{N_\Round}^*}^{-1}}^2 \leq \frac{1+\varepsilon}{N_\Round} \| \Act \|_{A_{\Design^*}^{-1}}^2
$.
It follows, similar to the proof of \Cref{thm:adaptive_algorithm_complexity}, that in round $\Round$, for each $\Act \in \BetterActSet$ if $\TrueConst^T \Act > \TrueConst^T \BestAct$:
\begin{align*}
    \TrueConst^T \Act - \LowerConst^\Round(\Act)
    \leq \sqrt{2\log(|\ActSet|/\delta_\Round)} \| \Act \|_{A_{\vx_n^*}^{-1}}
    \leq \sqrt{2 (1+\varepsilon) \log(|\ActSet|/\delta_\Round) / {N_\Round}} \| \Act \|_{A_{\Design^*}^{-1}}
\end{align*}
A similar argument gives for $\BestAct$:
\begin{align*}
    \UpperConst^\Round(\BestAct) - \TrueConst^T \BestAct
    \leq \sqrt{2\log(|\ActSet|/\delta_\Round)} \| \BestAct \|_{A_{\vx_n^*}^{-1}}
    \leq \sqrt{2 (1+\varepsilon) \log(|\ActSet|/\delta_\Round) / {N_\Round}} \| \BestAct \|_{A_{\Design^*}^{-1}}
\end{align*}

Let us call the event that these confidence bounds hold $\GoodEvent_\Round$. We have $P(\GoodEvent_\Round | \GoodEvent_{\Round-1}) \geq 1-\delta_\Round$.
Now, consider round $\bar{\Round} = \lceil \log_{v} \left( 2 (1+\varepsilon) \Complexity(\nu) \right) \rceil $ with length $N_{\bar{\Round}} = \lceil 2 (1+\varepsilon) \log(|\ActSet|/\delta_\Round) \Complexity(\nu) \rceil$. For all $\Act \in \BetterActSet$ if $\TrueConst^T \Act > \TrueConst^T \BestAct$:
\begin{align*}
    \TrueConst^T \Act - \LowerConst^{\bar{\Round}}(\Act)
    \leq \sqrt{\frac{1}{\Complexity(\nu)}} \| \Act \|_{A_{\Design^*}^{-1}}
    \leq \sqrt{\frac{(\TrueConst^T \Act)^2}{\| \Act \|_{A_{\Design^*}^{-1}}^2}} \| \Act \|_{A_{\Design^*}^{-1}} \leq |\TrueConst^T \Act|
\end{align*}
Note that $\Act$ is infeasible and $\TrueConst^T \Act$, which implies $\LowerConst^{\bar{\Round}}(\Act) \geq 0$ and in turn $\Act \notin \UncertainSet_{\bar{\Round}}$.
Similarly, $\UpperConst^{\bar{\Round}}(\BestAct) - \TrueConst^T \BestAct \leq | \TrueConst^T \BestAct |$. $\BestAct$ is feasible and $\TrueConst^T \BestAct \leq 0$. Hence, $\UpperConst^{\bar{\Round}}(\BestAct) \leq 0$ and $\BestAct \notin \UncertainSet_{\bar{\Round}}$. This implies that $\UncertainSet_{\bar{\Round}} = \emptyset$ and, conditioned on $\GoodEvent_{\bar{\Round}}$, the oracle algorithm solves the problem in round $\bar{\Round}$ with probability $1$.
We can apply \Cref{lemma:delta_bound} to conclude that, unconditionally, the algorithm solves the problem in round $\bar{\Round}$ with a probability of at least $1-\delta$.

Let us compute the total iterations necessary:
\begin{align*}
    N = \sum_{\Round=1}^{\bar{\Round}} \lceil 2 (1+\varepsilon) \log(|\ActSet|/\delta_t) \Complexity(\nu) \rceil
    \leq \bar{\Round} (1 + 2 (1+\varepsilon) \log(|\ActSet|\bar{\Round}^2/\delta^2) \Complexity(\nu)) \propto \Complexity(\nu)
\end{align*}
So, $N$ is on order $\Complexity(\nu)$ except for logarithmic factors, concluding the proof.
\end{proof}

\subsection{G-Allocation}
\label{app:g-allocation}

We obtain G-Allocation by choosing $\Design^* \in \argmin_{\Design} \max_{\Act\in\ActSet} \| \Act \|_{\DesMat^{-1}}$ in \Cref{alg:generic_round_based}. G-Allocation uniformly reduce the uncertainty about the constraint function for all arms. This is not ideal because it does not focus on which arms are plausible optimizers according to the known reward function.

Still, the following theorem shows that G-Allocation achieves sample complexity on order $d/\MinConst^2$, so it matches the worst-case lower bound in \Cref{thm:instance_independent_lower_bound}.

\begin{theorem}\label{thm:g_allocation_sample_complexity}
G-Allocation finds the optimal arm within $N \propto d/\MinConst^2$ iterations with probability at least $1-\delta$.
\end{theorem}

\begin{proof}
As in the proof of \Cref{thm:oracle_sample_complexity}, we have in round $\Round$, for each $\Act \in \BetterActSet$ if $\TrueConst^T \Act > \TrueConst^T \BestAct$:
\begin{align*}
    \UpperConst^\Round(\Act) - \LowerConst^\Round(\Act)
    \leq 2 \sqrt{2\log(|\ActSet|/\delta_\Round)} \| \Act \|_{A_{\vx_n^*}^{-1}}
    \leq 2 \sqrt{2 (1+\varepsilon) \log(|\ActSet|/\delta_\Round) / {N_\Round}} \| \Act \|_{A_{\Design^*}^{-1}}
\end{align*}

Again, we call the event that these confidence bounds hold $\GoodEvent_\Round$, and have $P(\GoodEvent_\Round | \GoodEvent_{\Round-1}) \geq 1-\delta_\Round$.
Now, consider round
\begin{align*}
\bar{\Round} &= \log_{v} \left( 8 (1+\varepsilon) \argmin_\Design \max_{\Act\in\ActSet} \| \Act \|_{\DesMat^{-1}} / \MinConst^2 \right) \\ N_{\bar{\Round}} &= \left\lceil 8 (1+\varepsilon) \log(|\ActSet|/\delta_\Round) \argmin_\Design \max_{\Act\in\ActSet} \| \Act \|_{\DesMat^{-1}} / \MinConst^2 \right\rceil
\end{align*}
For all $\Act \in \UncertainSet_{\Round}$ it follows that:
\begin{align*}
    \UpperConst^{\bar{\Round}}(\Act) - \LowerConst^{\bar{\Round}}(\Act)
    \leq \sqrt{\frac{\MinConst^2}{\| \Act \|_{A_{\Design^*}^{-1}}}} \| \Act \|_{A_{\Design^*}^{-1}}
    \leq \MinConst
\end{align*}
This implies that G-Allocation solves the problem in round $\bar{\Round}$ with probability $1$, similar to the proof  of \Cref{thm:adaptive_algorithm_complexity}.
We can apply \Cref{lemma:delta_bound} to conclude that, unconditionally, the algorithm solves the problem in round $\bar{\Round}$ with a probability of at least $1-\delta$.

Let us compute the total iterations necessary:
\begin{align*}
    N &= \sum_{\Round=1}^{\bar{\Round}} \left\lceil 8 (1+\varepsilon) \log(|\ActSet|/\delta_\Round) \argmin_\Design \max_{\Act\in\ActSet} \| \Act \|_{\DesMat^{-1}} / \MinConst^2 \right\rceil \\
    &\leq \bar{\Round} \left(1 + 8 (1+\varepsilon) \log(|\ActSet|/\delta_\Round) \argmin_\Design \max_{\Act\in\ActSet} \| \Act \|_{\DesMat^{-1}} / \MinConst^2 \right) \\
    &\leq \bar{\Round} \left(1 + 8 (1+\varepsilon) \log(|\ActSet|/\delta_\Round) d / \MinConst^2 \right) \propto d/\MinConst^2
\end{align*}
where the last inequality uses the result by \citet{kiefer1960equivalence}.
\end{proof}

\section{Experimental Details About the Driving Environment}\label{app:experiment_details}

This section provides details on the driving environment we use in \Cref{sec:driving_experiments}. We provide full source code for all of our experiments at: \codeurl

We extend the \emph{Driver} proposed by \citet{dorsa2017active} and \citet{biyik2019asking}, to incorporate different tasks. Here, we provide a brief description of the dynamics and features of the environment.

The \emph{Driver} environment uses point-mass dynamics with a continuous state and action space. The state $s = (x, y, \varphi, v)$ consists of the agent's position $(x, y)$, its heading $\varphi$, and its velocity $v$. The actions $a = (a_1, a_2)$ consist of a steering input and an acceleration. The environment dynamics are given by
\begin{align*}
s_{t+1} = (x_{t+1}, y_{t+1}, \varphi_{t+1}, v_{t+1}) &= (x_{t} + \Delta x, y_{t} + \Delta y, \varphi_{t} + \Delta \varphi, \mathrm{clip}(v_{t} + \Delta v, -1, 1)) \\
(\Delta x, \Delta y, \Delta \varphi, \Delta v) &= (v\cos{\varphi}, v\sin{\varphi}, v a_1, a_2 - \alpha v)
\end{align*}
where $\alpha = 1$ is a friction parameter, and the velocity is clipped to $[-1, 1]$ at each timestep.

The environment represents a highway with three lanes. In addition to the agent, the environment contains a second car that moves on a predefined trajectory. The reward and the constraint functions are linear in a set of features
\[
f(s) = (f_1(s), f_2(s), f_3(s), f_4(s), f_5(s), f_6(s), f_7(s), f_8(s), 1)
\]
that are described in detail in \Cref{tab:driver_features}.

The (known) rewards for the three scenarios are:
\begin{align*}
    &\text{Base scenario:} \quad \theta_1 = (1, 0, 0, 0, 0, 0, 0, 0, 0) \\
    &\text{Different reward:} \quad \theta_2 = (0, 1, 0, 0, 0, 0, 0, 0, 0) \\
    &\text{Different environment:} \quad \theta_3 = (1, 0, 0, 0, 0, 0, 0, 0, 0)
\end{align*}
The (unknown) constraint is:
\begin{align*}
    \phi = (0, 0, 0.3, 0.05, 0.02, 0.5, 0.3, 0.8), \quad \tau = 1
\end{align*}

Our \emph{Driver} environment uses a fixed time horizon $T=20$, and policies are represented simply as sequences of $20$ actions because the environment is deterministic.

\subsection{Cross-Entropy Method for Constrained RL}

We find policies in the Driver environment with a given reward function using the cross-entropy method \citep{rubinstein2004cross}. For the constrained reinforcement learning problem, we use a modified cross-entropy method, proposed by \citet{wen2020constrained}, that takes the feasibility of solutions into account. \Cref{alg:cross_entropy_driver} contains pseudocode of this method.

\begin{algorithm}[t]
\caption{Cross-entropy method for (constrained) reinforcement learning. For more details on the cross-entropy method, see \citet{rubinstein2004cross}, and for the application to constrained RL, see \citet{wen2020constrained}.}
\label{alg:cross_entropy_driver}
\begin{algorithmic}[1]
  \STATE \textbf{Input}: $n_\text{iter}$, $n_\text{samp}$, $n_\text{elite}$
  \STATE Initialize policy parameters $\mu\in\R^d$, $\sigma\in\R^d.$
  \FOR{iteration = $1, 2, \dots, n_\text{iter}$}
      \STATE Sample $n_\text{samp}$ samples of $\omega_i \sim \Gaussian(\mu, \diag(\sigma))$
      \STATE Evaluate policies $\omega_1, \dots, \omega_{n_\text{samp}}$ in the environment
      \IF{constrained problem}
          \STATE Sort $\omega_i$ in ascending order of constraint value $J(\omega_i)$
          \STATE Let $E$ be the first $n_\text{elite}$ policies
          \IF{$J(\omega_{n_\text{elite}}) \leq 0$}
              \STATE Sort $\{\omega_i | J(\omega_i) \leq 0\}$ in descending order of return $G(\omega_i)$
              \STATE Let $E$ be the first $n_\text{elite}$ policies
          \ENDIF
      \ELSE
          \STATE Sort $\omega_i$ in descending order of return $G(\omega_i)$
          \STATE Let $E$ be the first $n_\text{elite}$ policies
      \ENDIF
      \STATE Fit Gaussian distribution with mean $\mu$ and diagonal covariance $\sigma$ to $E$
  \ENDFOR
  \STATE \textbf{return} $\mu$
\end{algorithmic}
\end{algorithm}

\subsection{Binary Feedback}

So far, we considered numerical observations of the constraint value $\ConstParam^T \Act + \noise$ where $\noise$ is subgaussian noise. In the driving environment, we (more realistic) binary observations in $\{-1, 1\}$.

If we assume that all true constraint values are in $[-1, 1]$, we can define the observation model $P(y = 1 | \ConstParam, \Act) = (\ConstParam^T \Act + 1) / 2$. We can consider this as bounded, sub-gaussian noise on the constraint value, and so all our analysis still applies.

\subsection{Setup}

To translate learning the unknown constraint function in the Driver environment into a constrained linear best arm identification problem, we consider a set of pre-computed policies $\Pi$. This set of policies corresponds to the arms of a linear bandit problem, and both the return $G(\pi)$ of a policy and the constraint function $J(\pi)$ are linear in the expected feature counts of the policy:
$G(\pi) = \vf(\pi) \cdot \vr$ and $J(\pi) = \vf(\pi) \cdot \vc$.

For binary observations, we normalize the features of all policies such that all constraint values are between $-1$ and $1$.

\begin{table}
\centering
\begin{tabular}{ccccccccc}
  Feature & Description & Type & Definition & & $\theta_1$ & $\theta_2$ & $\theta_3$ & $\phi$ \\\toprule
  $f_1(s)$ & Target velocity & Numerical & $-(v - 0.4)^2$ & & 1 & 0 & 1 & 0 \\\midrule
  $f_2(s)$ & Target location & Numerical & \specialcell{$-(x - x_r)^2$ \\ $x_r$ center of right lane } & & 0 & 1 & 0 & 0 \\\midrule
  $f_3(s)$ & Stay on street & Binary & $1$ iff off street & \raisebox{-.5\height}{\includegraphics[width=0.1\linewidth]{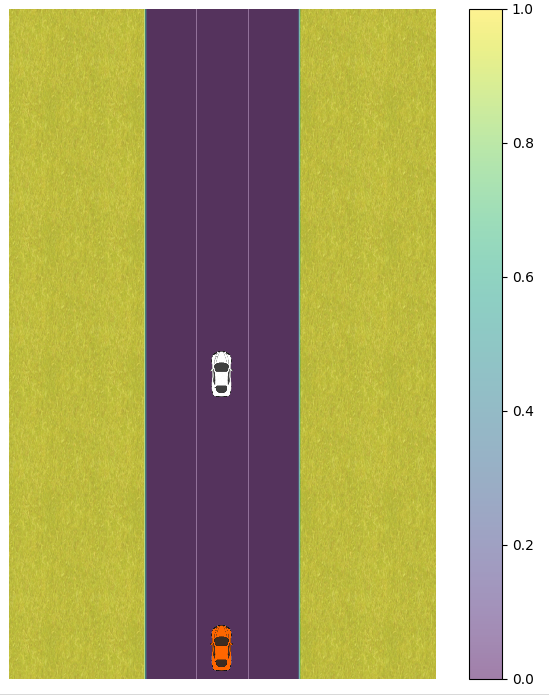}} & 0 & 0 & 0 & 0.3 \\\midrule
  $f_4(s)$ & Stay in lane & Numeric & \specialcell{$\frac{1}{1 + \exp(-bd+a)}$, \\ $d$ distance to closest lane center, \\ $b=10000$, $a=10$} & \raisebox{-.5\height}{\includegraphics[width=0.1\linewidth]{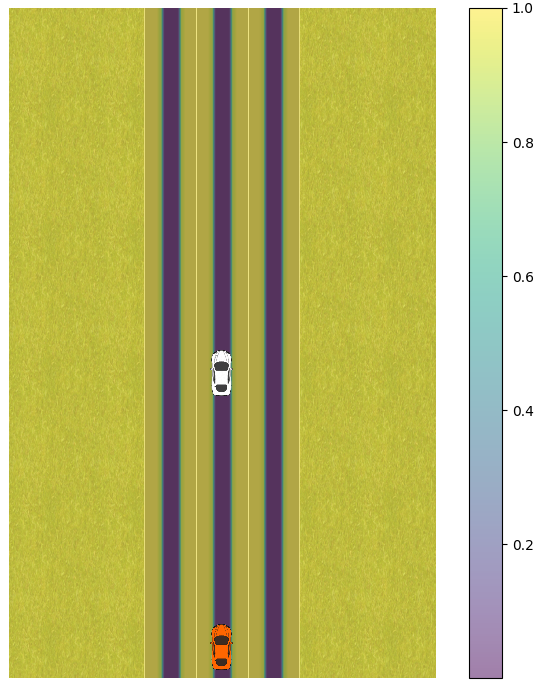}} & 0 & 0 & 0 & 0.05 \\\midrule
  $f_5(s)$ & \specialcell{Stay aligned \\ with street} & Numeric & $|\cos(\theta)|$ & & 0 & 0 & 0 & 0.02 \\\midrule
  $f_6(s)$ & \specialcell{Don't drive \\ backwards} & Binary & $1$ iff $v < 0$ & & 0 & 0 & 0 & 0.5 \\\midrule
  $f_7(s)$ & \specialcell{Stay within \\ speed limit} & Binary & $1$ iff $v > 0.6$ & & 0 & 0 & 0 & 0.3 \\\midrule
  $f_8(s)$ & \specialcell{Don't get too close \\ to other cars} & Numeric & \specialcell{$\exp(-b(c_1 d_x^2 + c_2 d_y^2) + ba)$, \\ $a=0.01$, $b=30$, \\ $c_1=4$, $c_2=1$} & \raisebox{-.5\height}{\includegraphics[width=0.1\linewidth]{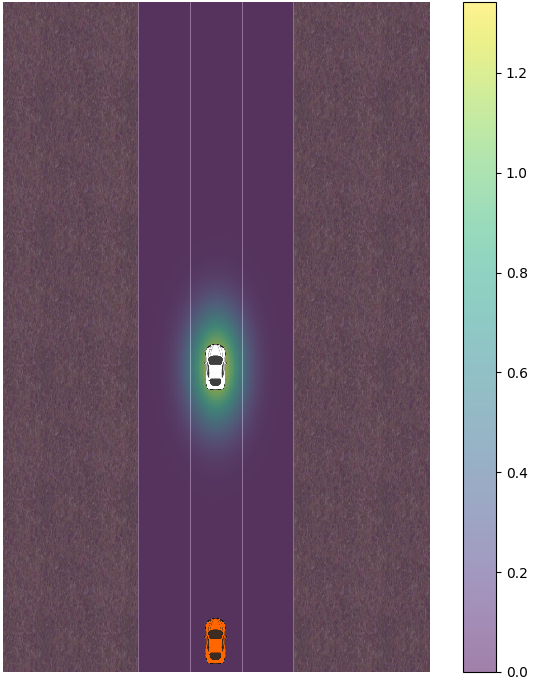}} & 0 & 0 & 0 & 0.8 \\\bottomrule
\end{tabular}
\caption{Features for representing the reward and constraint function in the Driver environment. The last four columns contain the reward weights for the three scenarios and the constraint weight that is shared.}
\label{tab:driver_features}
\end{table}

\section{Additional Experimental Results}\label{app:more_results}

Here, we provide the additional results for the experiments discussed in the main paper. Full results are shown in \Cref{fig:bandit_more_results} for the bandit results and \Cref{fig:driver_more_results} for the driving scenario. \Cref{tab:baselines} contains an overview of all algorithms and baselines that we evaluated.

We find that methods that select arms from $\UncertainSet$ randomly (\AlgGreedyShort Uniform) or by maximizing the reward (\MaxRewU) can perform quite well in some cases with tuned confidence intervals. Indeed, \MaxRewU outperforms \AlgGreedyShort in the unit sphere experiment. This is not consistent across environments, and \AlgGreedyShort performs comparable or better in all other environments. Still, in some cases, when theoretical guarantees are not required, these heuristic approaches might be valuable alternatives.

\begin{table*}[t]
   \centering
   \begin{tabular}{lcccc}
      \toprule
      Name & \specialcell{Confidence \\ Intervals} & \specialcell{Selection \\ Criterion} & \specialcell{Select From \\ Arms} & \specialcell{Plot \\ Color} \\
      \midrule
      Oracle & static & Oracle & All & \legendOracle \\
      G-Allocation & static & MaxVar & All & \legendGAllocation \\
      Uniform & static & Uniform & All & \legendUniform \\
      \AlgShort & static & MaxVar & Uncertain & \legendAdaptiveStatic \\
      Greedy MaxVar & adaptive & MaxVar & All & \legendMaxvarAll \\
      Adaptive Uniform & adaptive & Uniform & All & \legendUniformAll \\
      \MaxRewU & adaptive & Max Rew & Uncertain & \legendMaxRewU \\
      \MaxRewF & adaptive & Max Rew & Feasible & - \\
      \AlgGreedyShort & adaptive & MaxVar & Uncertain & \legendAdaptive \\
      \AlgGreedyShort Uniform & adaptive & Uniform & Uncertain & \legendAdaptiveUniform \\
      Greedy MaxVar (tuned) & adaptive tuned & MaxVar & All & \legendMaxvarAllTuned \\
      Adaptive Uniform (tuned) & adaptive tuned & Uniform & All & \legendUniformAllTuned \\
      \AlgGreedyShort (tuned) & adaptive tuned & MaxVar & Uncertain & \legendAdaptiveTuned \\
      \AlgGreedyShort Uniform (tuned) & adaptive tuned & Uniform & Uncertain & \legendAdaptiveUniformTuned \\
      \MaxRewU (tuned) & adaptive tuned & Max Rew & Uncertain & \legendMaxRewUTuned \\
      \MaxRewF (tuned) & adaptive tuned & Max Rew & Feasible & - \\
      \bottomrule
   \end{tabular}
   \caption{Overview of all algorithms we evaluate.}
   \label{tab:baselines}
\end{table*}

\begin{figure*}[t]
    \centering
    \subfigure[Irrelevant dimensions]{
        \includegraphics[width=0.23\linewidth]{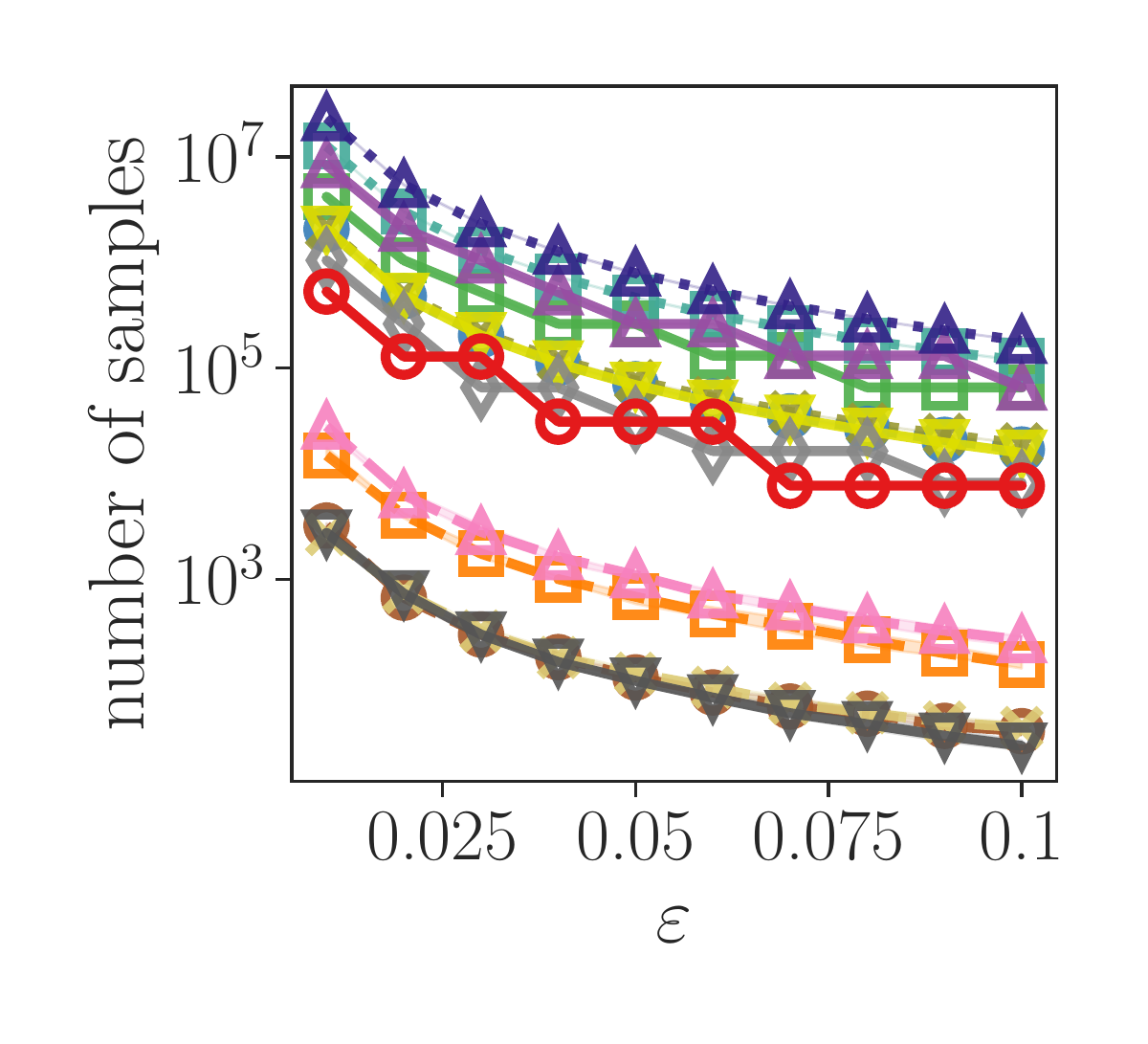}\hspace{0.5em}
        \includegraphics[width=0.23\linewidth]{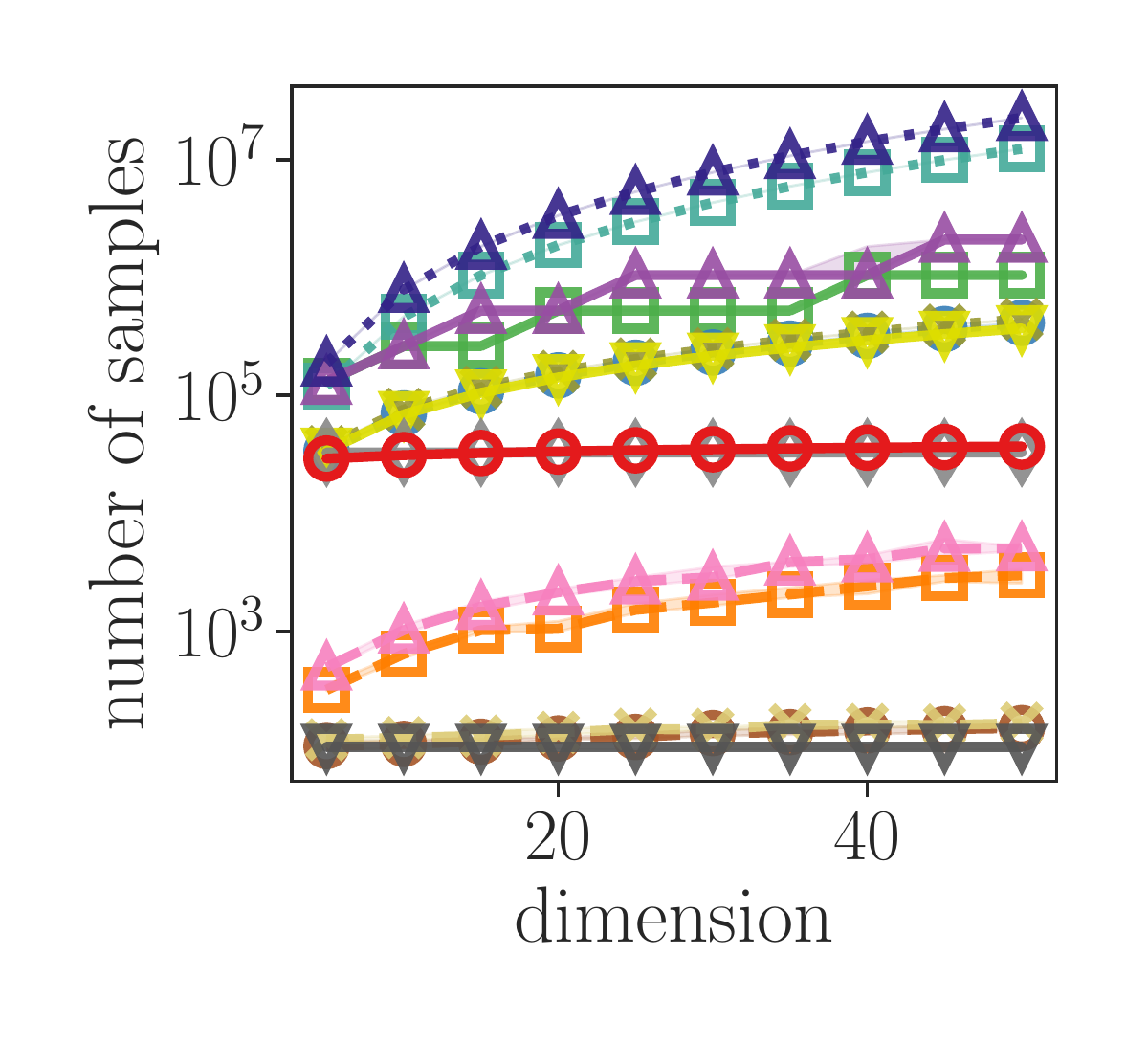}
    }\hfill
    \subfigure[Unit sphere]{
        \includegraphics[width=0.23\linewidth]{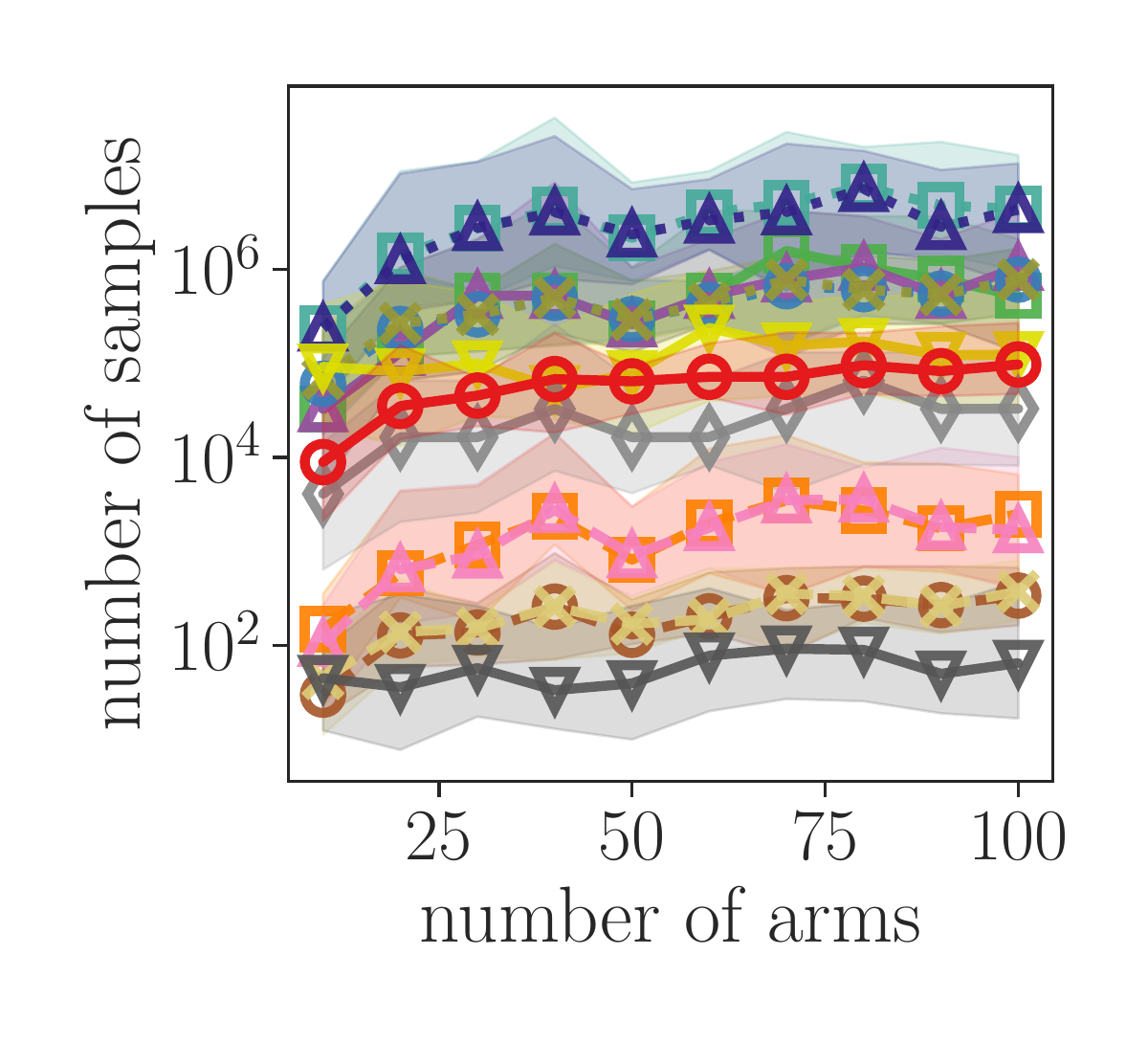}\hspace{0.5em}
        \includegraphics[width=0.23\linewidth]{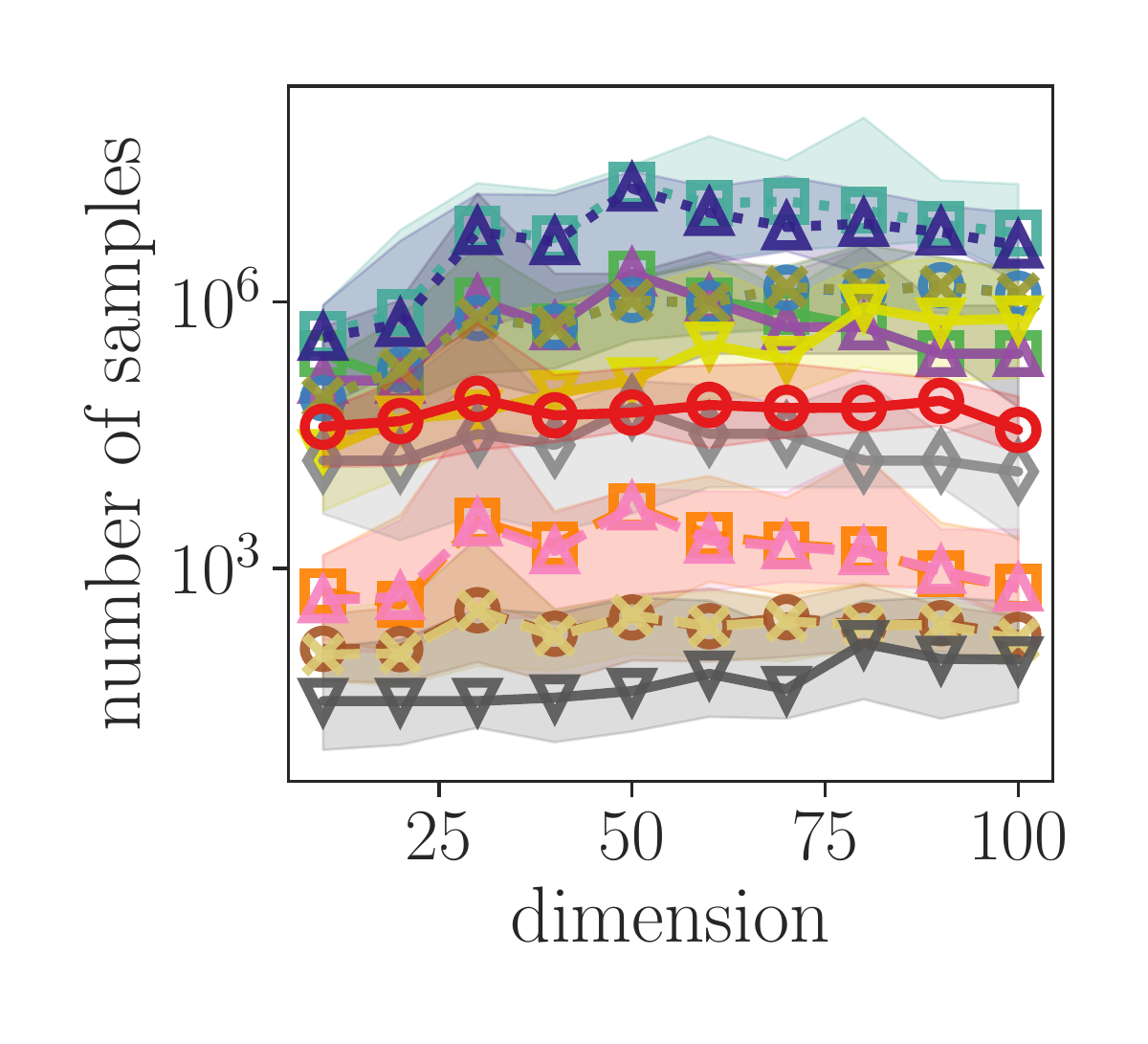}
    }
    {\small
    \begin{tabular}{llllll}
        \legendUniform & Uniform &
        \legendUniformAll & Adaptive Uniform &
        \legendUniformAllTuned & Adaptive Uniform (tuned) \\
        \legendGAllocation & G-Allocation &
        \legendMaxvarAll & Greedy MaxVar &
        \legendMaxvarAllTuned & Greedy MaxVar (tuned) \\
        \legendAdaptiveStatic & \AlgShort (ours) &
        \legendAdaptive & \AlgGreedyShort &
        \legendAdaptiveTuned & \AlgGreedyShort (tuned) \\
        \legendOracle & Oracle &
        \legendAdaptiveUniform & \AlgGreedyShort Uniform &
        \legendAdaptiveUniformTuned & \AlgGreedyShort Uniform (tuned) \\
        & & \legendMaxRewU & \MaxRewU & \legendMaxRewUTuned & \MaxRewU (tuned)
    \end{tabular}
    }
    \caption{
        Similar plots to \Cref{fig:bandit_results}, including some additional algorithms: the ``non-tuned'' versions of algorithms use the confidence interval from \Cref{thm:ols_bound_adaptive}, and \emph{\AlgGreedyShort Uniform} is \AlgGreedyShort with uniform sampling instead of the maximum variance objective. \Cref{tab:baselines} provides an overview of all baselines. Moreover, the plots here show the 25th and 75th percentiles over 30 random seeds. For ``irrelevant dimensions'', these are close to the median, but for ``unit sphere'', there is much more randomness because the instances are randomly generated.
    }
    \label{fig:bandit_more_results}
\end{figure*}

\begin{figure*}
    \centering
    \subfigure[``Base scenario'']{
       \begin{minipage}{0.2\linewidth}
        \includegraphics[width=\linewidth]{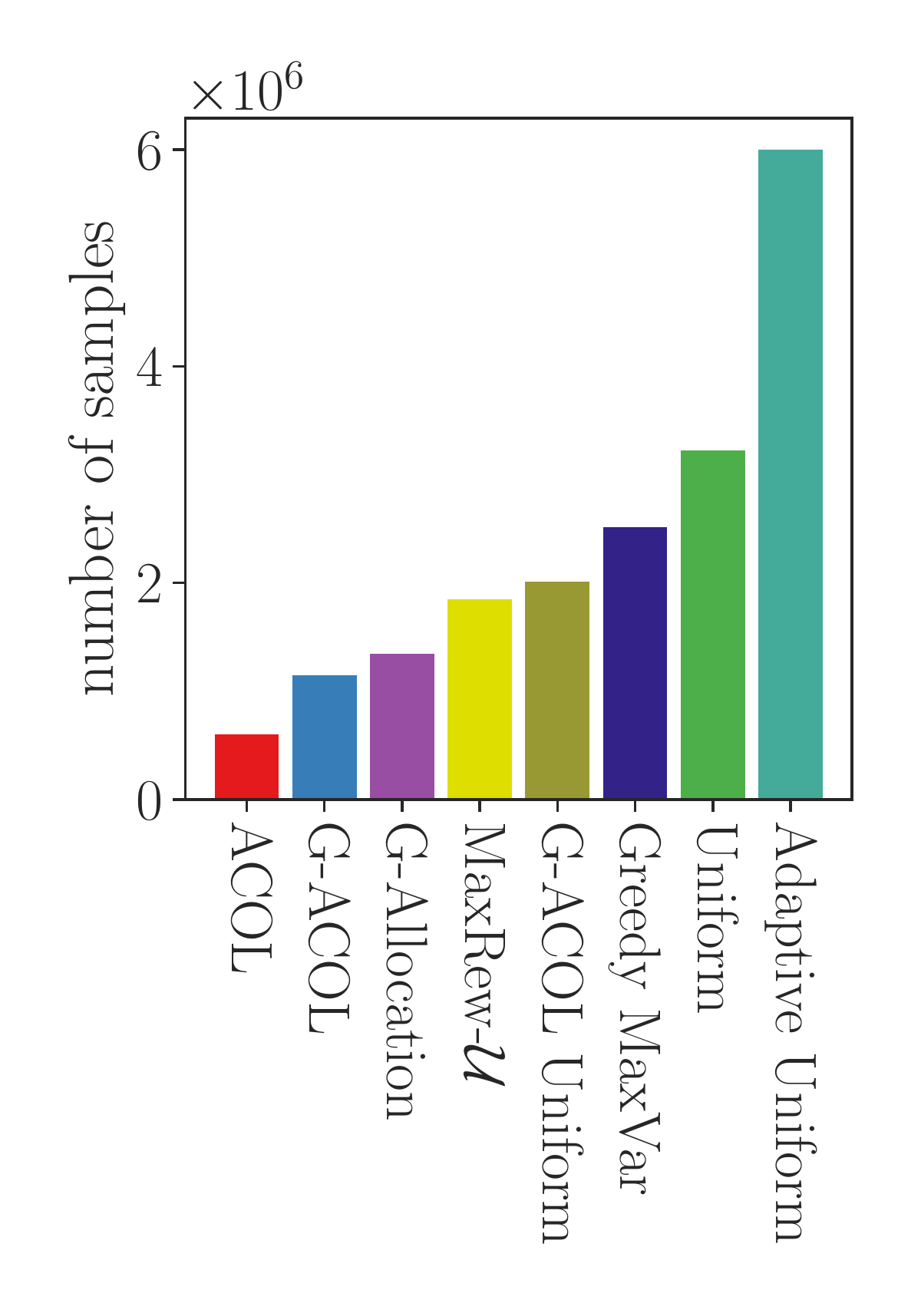}
       \end{minipage}\hspace{0.2cm}
       \begin{minipage}{0.28\linewidth}
        \includegraphics[width=0.95\linewidth]{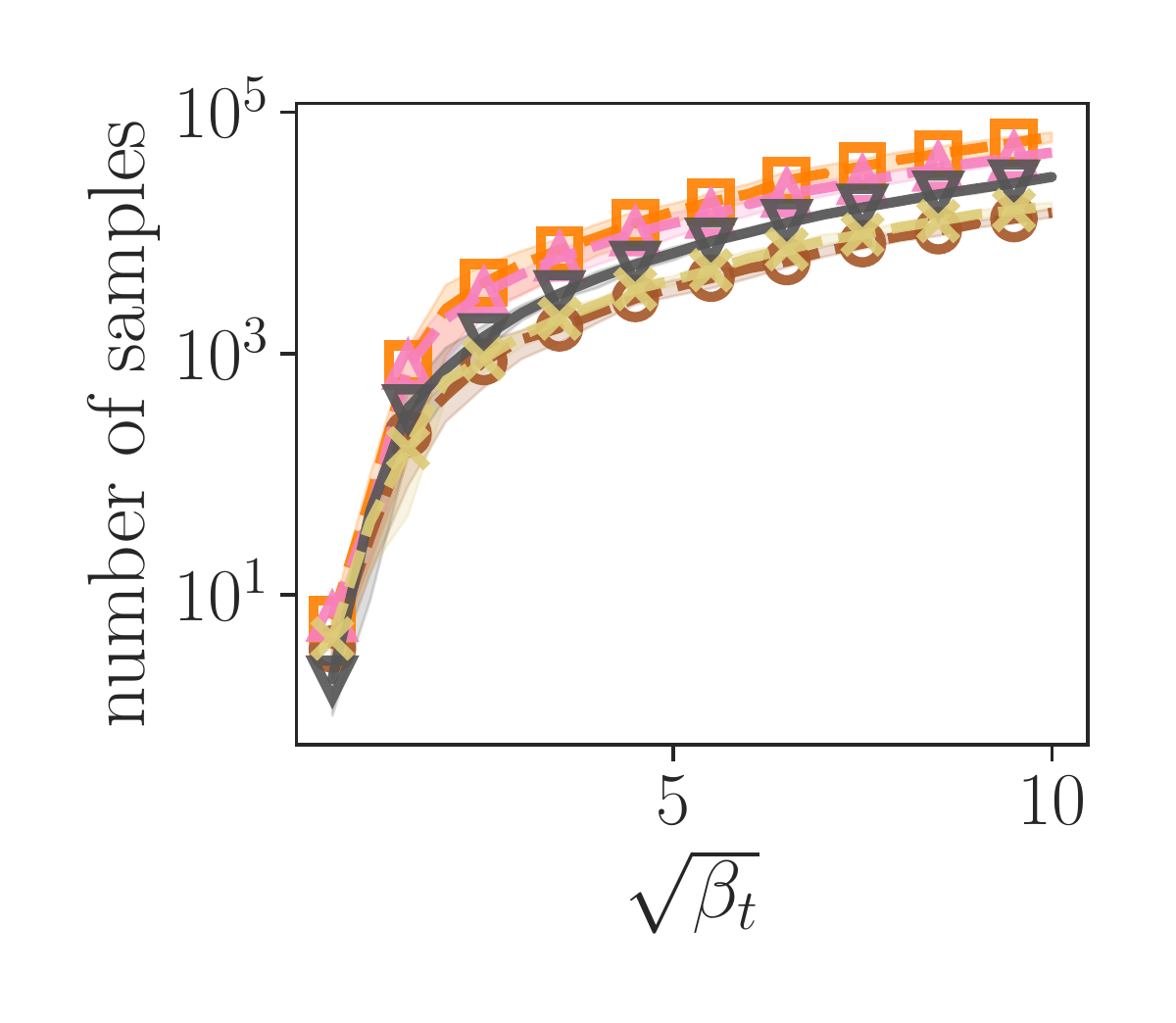}
       \end{minipage}
       \begin{minipage}{0.28\linewidth}
        \includegraphics[width=0.95\linewidth]{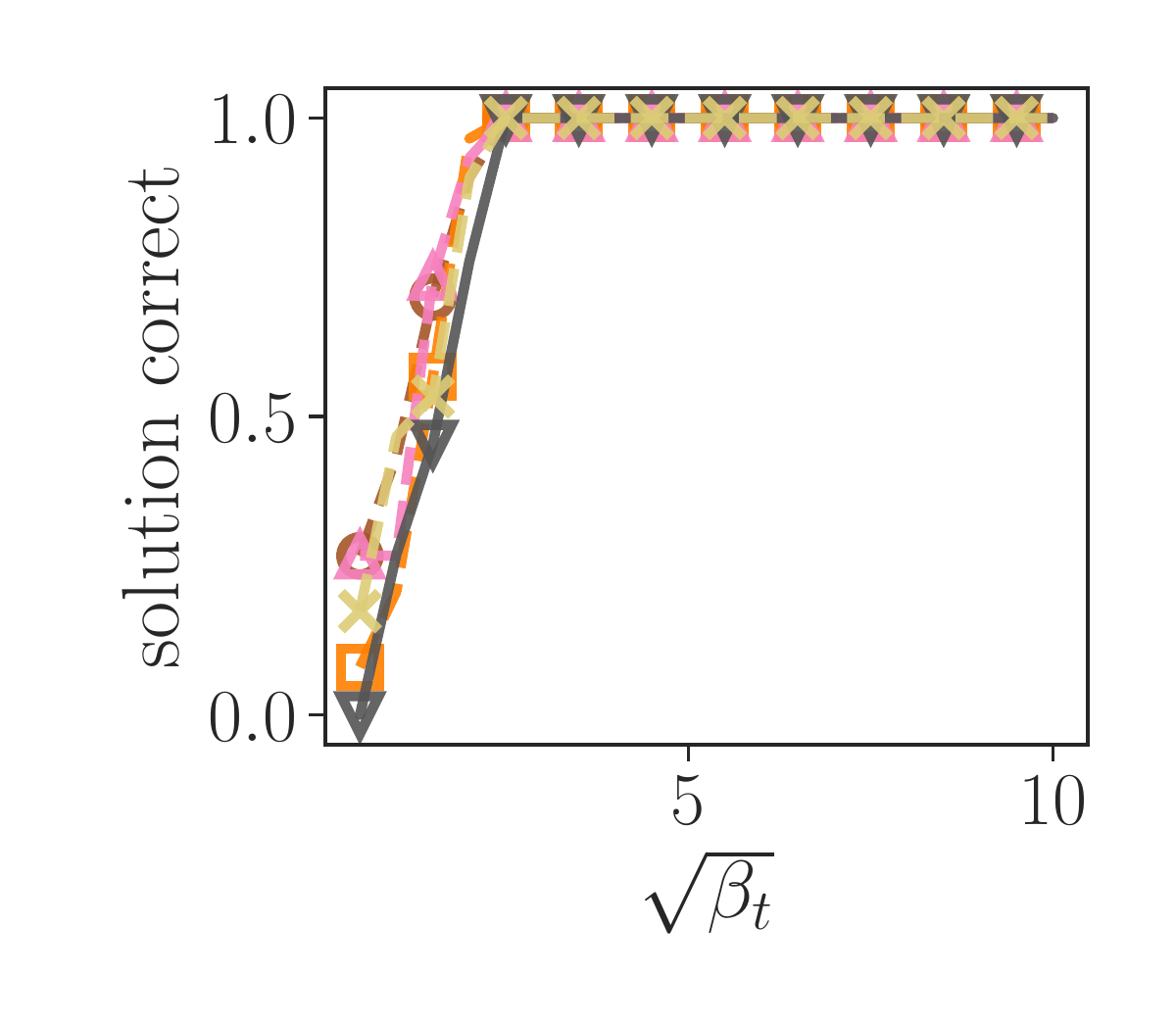}
       \end{minipage}
       \vspace{0.1cm}
    } \\
    \subfigure[``Different reward'']{
       \begin{minipage}{0.2\linewidth}
        \includegraphics[width=\linewidth]{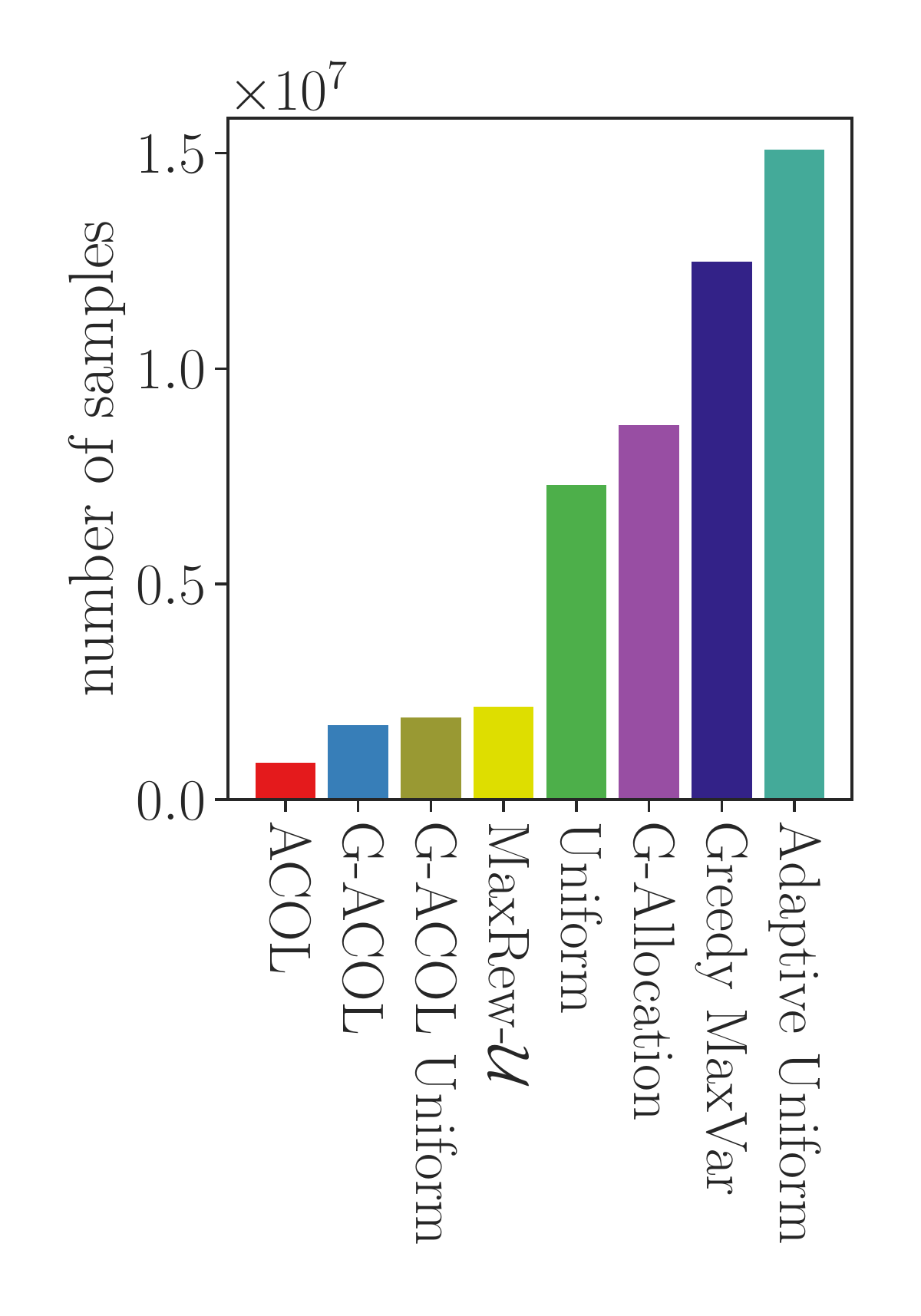}
       \end{minipage}\hspace{0.2cm}
       \begin{minipage}{0.28\linewidth}
        \includegraphics[width=0.95\linewidth]{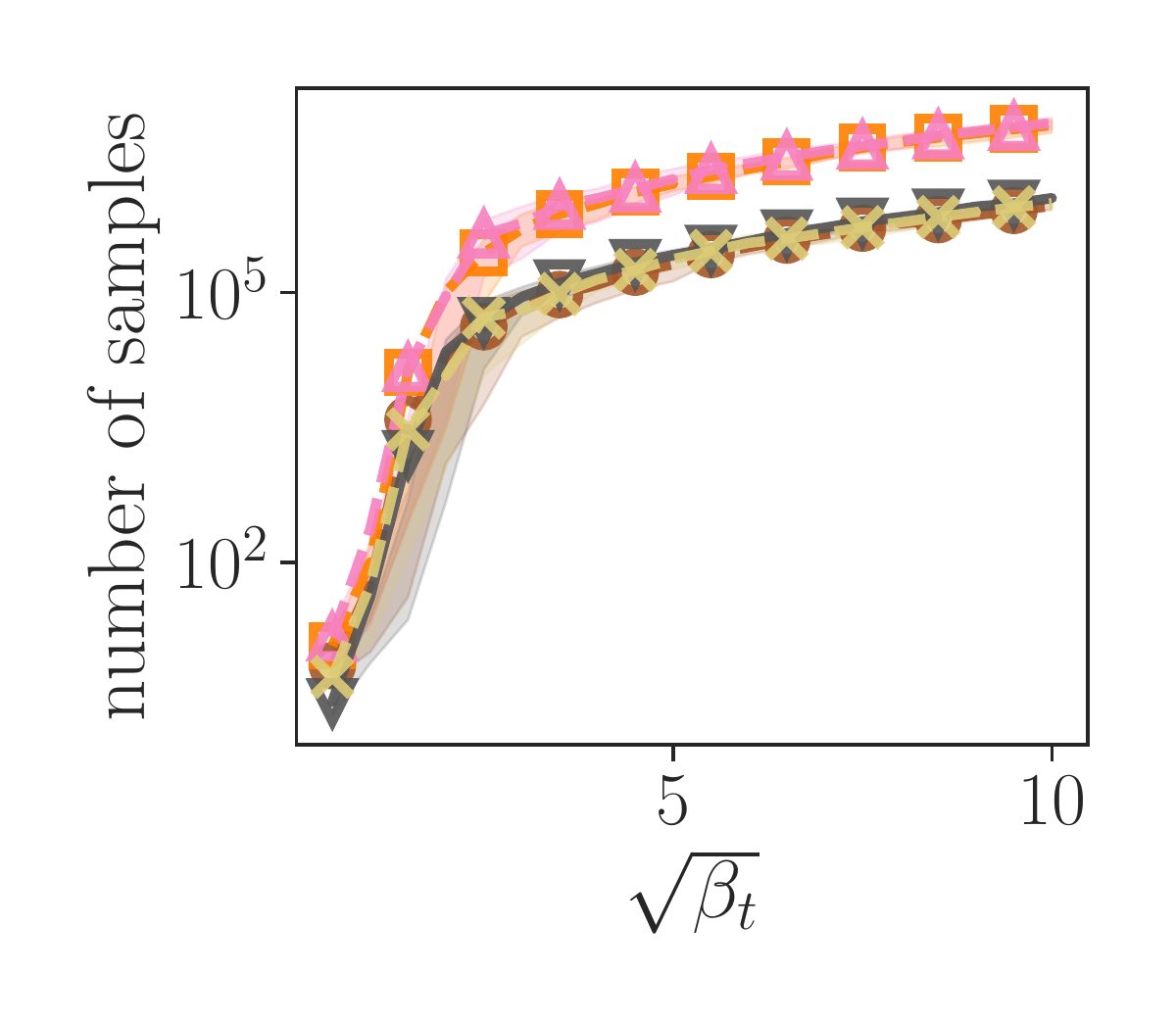}
       \end{minipage}
       \begin{minipage}{0.28\linewidth}
        \includegraphics[width=0.95\linewidth]{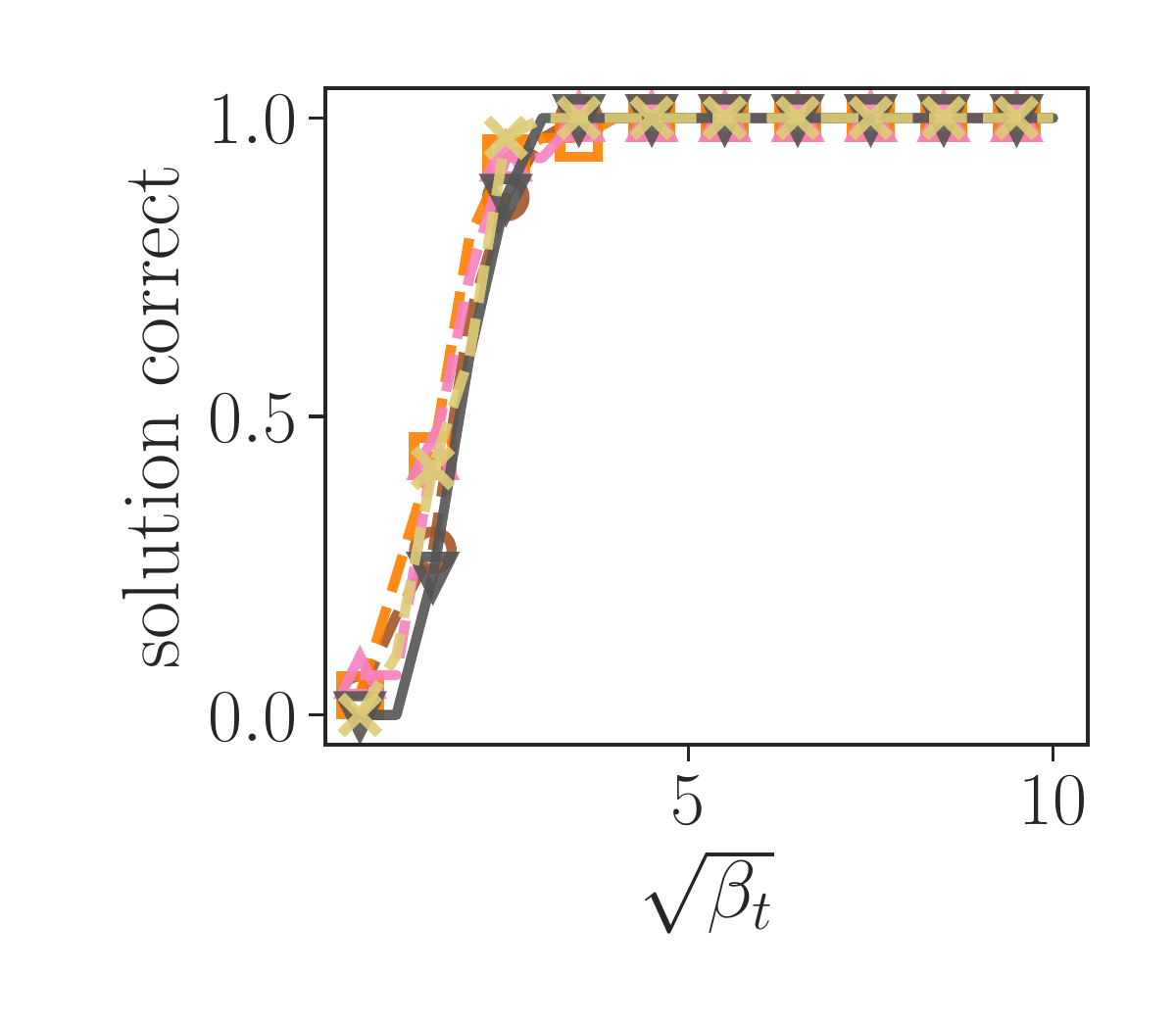}
       \end{minipage}
       \vspace{0.1cm}
    } \\
    \subfigure[``Different environment'']{
       \begin{minipage}{0.2\linewidth}
        \includegraphics[width=\linewidth]{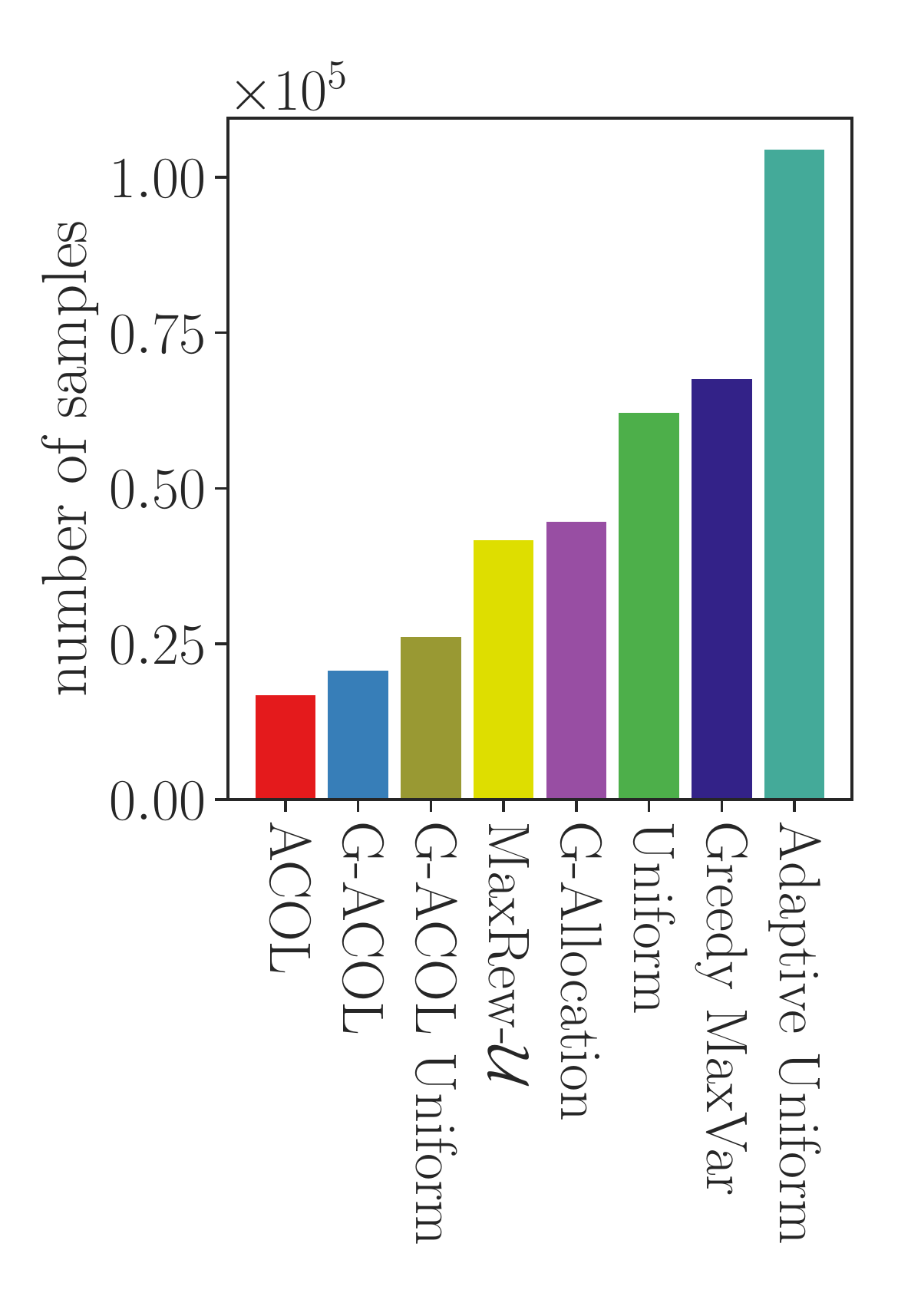}
       \end{minipage}\hspace{0.2cm}
       \begin{minipage}{0.28\linewidth}
        \includegraphics[width=0.95\linewidth]{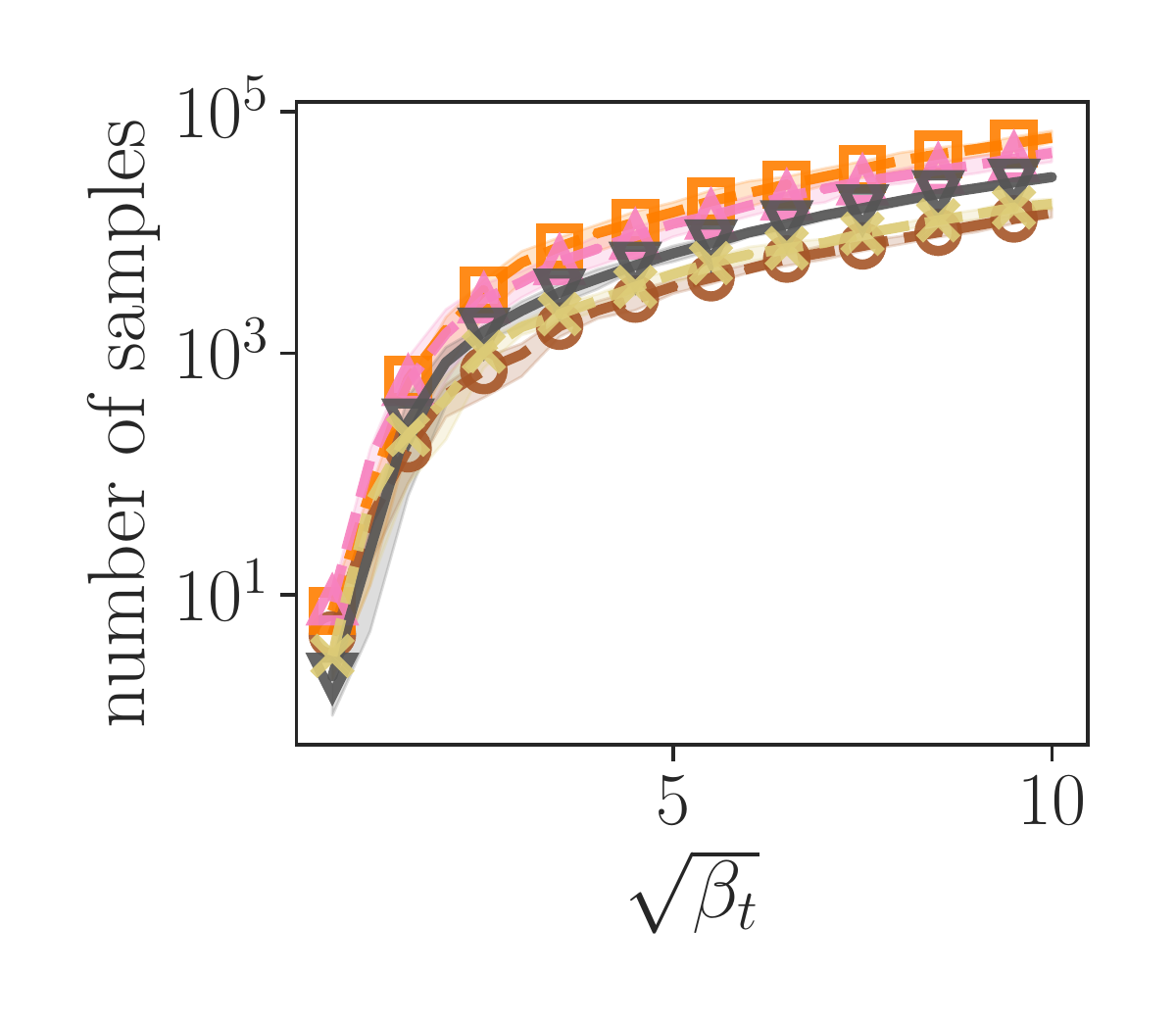}
       \end{minipage}
       \begin{minipage}{0.28\linewidth}
        \includegraphics[width=0.95\linewidth]{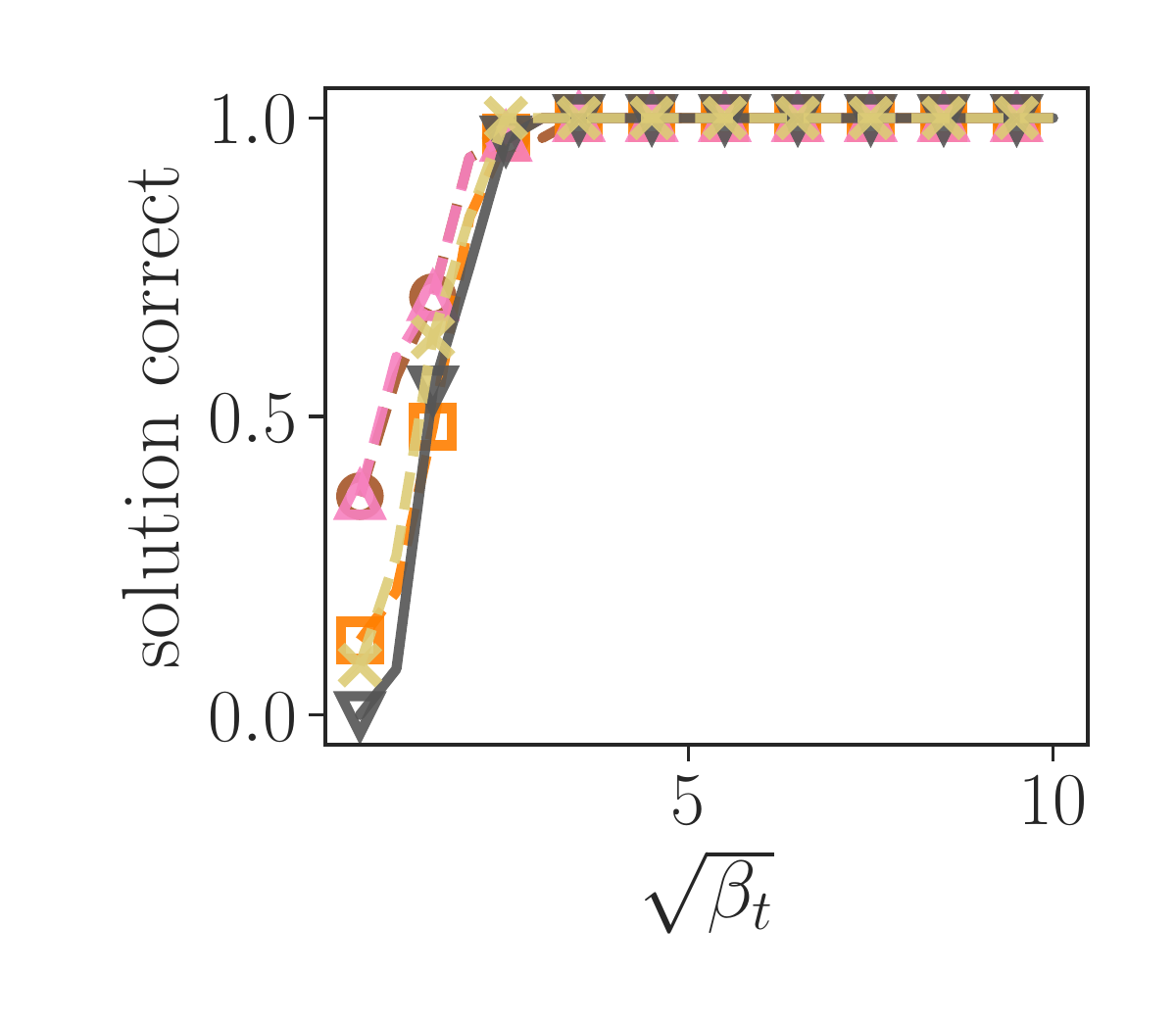}
       \end{minipage}
       \vspace{0.1cm}
    }
   {\small\centering
    \begin{tabular}{llllll}
    \vspace{0.1cm}\\
        \legendUniformAllTuned & Adaptive Uniform &
        \legendMaxvarAllTuned & Greedy MaxVar &
        \legendAdaptiveTuned & \AlgGreedyShort \\
        & & \legendMaxRewUTuned & \MaxRewU & \legendAdaptiveUniformTuned & \AlgGreedyShort Uniform
    \end{tabular}
    } 
    \caption{
        Similar plots as in \Cref{fig:driver_results} for all three driving scenarios from \Cref{fig:driver_example}, showing \AlgGreedyShort Uniform as aan additional baseline.
    }
    \label{fig:driver_more_results}
\end{figure*}

\end{document}